\documentclass[lettersize,journal]{IEEEtran}
\usepackage{amsmath,amsfonts}
\usepackage{algorithmic}
\usepackage{algorithm}
\usepackage{array}
\usepackage{subfig} \usepackage{textcomp}
\usepackage{stfloats}
\usepackage{url}
\usepackage{verbatim}
\usepackage{graphicx}
\usepackage{cite}
\usepackage{subfiles}
\usepackage{xr-hyper}
\usepackage[pagebackref=true,breaklinks=true,letterpaper=true,colorlinks,filecolor=red,bookmarks=false]{hyperref} 

\usepackage{makecell}
\usepackage{boldline}
\usepackage{algorithm, algorithmic}
\usepackage{multirow}   
\usepackage{enumitem}
\usepackage{breqn}
\usepackage{cleveref} \usepackage{wrapfig}  
\usepackage{caption} \usepackage{booktabs}
\usepackage{amsthm} \usepackage{amssymb}

\usepackage{amsmath,amsfonts}
\usepackage{algorithmic}
\usepackage{array}
\usepackage{textcomp}
\usepackage{stfloats}
\usepackage{url}
\usepackage{verbatim}
\usepackage{graphicx}
\usepackage{balance}
\usepackage{graphicx}
\usepackage{mathtools}
\usepackage{pgffor}
\usepackage{algorithm}
\usepackage{algorithmic}
\usepackage{booktabs}
\usepackage{placeins}
\usepackage{hyperref}

\usepackage{xparse}
\usepackage{cleveref}
\usepackage{bm}

\usepackage{chngcntr}
\usepackage[linewidth=3pt]{mdframed}

\newif\ifHeightAcquired
\NewDocumentCommand{\resetHeight}{}{\HeightAcquiredfalse}
\NewDocumentCommand{\myincludegraphics}{O{} m}{\ifHeightAcquired \includegraphics[height=\figheight]{#2}\else \let\tempbox\relax \newsavebox{\tempbox}\sbox{\tempbox}{\includegraphics[#1]{#2}}\global\edef\figheight{\the\ht\tempbox}\global\HeightAcquiredtrue \usebox{\tempbox}\fi }

\def\addvalue#1#2{\expandafter\gdef\csname my@data@#1\endcsname{#2}}
\def\usevalue#1{\csname my@data@#1\endcsname}

\newcommand{\mysubsubsection}[1]{\noindent\subsubsection{#1}}

\title{Towards the Connection between Activation Sparsity and Flat Minima}
\author{
    \IEEEauthorblockN{Ze Peng, Jian Zhang, Lei Qi, Yang Gao, Yinghuan Shi*}
    \thanks{Ze Peng, Jian Zhang, Yang Gao, and Yinghuan Shi are with 
State Key Laboratory for Novel Software Technology, Nanjing University, 
Nanjing 210023, China. 
Ze Peng, Jian Zhang, and Yinghuan Shi are also with the Institute of Brain-Machine Interface, Nanjing University, Nanjing 210023, China.
E-mail: \{pengze\}@smail.nju.edu.cn, \{zhang.jian, gaoy, syh\}@nju.edu.cn.
Lei Qi is with the School of Computer Science and Engineering, 
Southeast University, Nanjing 211189, China. E-mail: qilei@seu.edu.cn.
Corresponding author: Yinghuan Shi (Email: syh@nju.edu.cn).}
}

\newcommand{\set}[1]{\{#1\}}

\newcommand{\size}[1]{{\left|{#1}\right|}}
\newcommand{\abs}[1]{\size{#1}}

\newcommand{\indic}[1]{{\mathrm{I}\left[{#1}\right]}}

\newcommand{\norm}[1]{\left\lVert#1\right\rVert}
\newcommand{\mat}[1]{\mathbf{#1}}

\NewDocumentCommand{\reals}{}{\mathbb{R}}
\NewDocumentCommand{\nats}{}{\mathbb{N}}

\NewDocumentCommand{\opt}{O{}}{\mathrm{OPT}_{\mathrm{#1}}}
\NewDocumentCommand{\sol}{O{}}{\mathrm{SOL}_{\mathrm{#1}}}

\NewDocumentCommand{\ex}{O{} m}{{\mathbb{E}_{#1}\left[#2\right]}}

\NewDocumentCommand{\defeq}{}{\triangleq}
\NewDocumentCommand{\mutualinfo}{O{} O{} m m}{I_{#1}^{#2}(#3; #4)}
\NewDocumentCommand{\entropy}{O{} O{} m m}{H_{#1}^{#2}(#3)}
\NewDocumentCommand{\prob}{O{} O{} m}{\mathrm{P}^{#1}_{#2}\left[#3\right]}
\NewDocumentCommand{\transpose}{}{T}
\NewDocumentCommand{\trace}{m}{\mathrm{tr}\left(#1\right)}
\NewDocumentCommand{\hadamard}{}{\odot}

\NewDocumentCommand{\dataset}{}{\mathcal{D}}
\NewDocumentCommand{\loss}{}{\mathcal{L}}
\NewDocumentCommand{\derivatives}{m m O{}}{\frac{\partial#3 #1}{\partial #2#3}}
\NewDocumentCommand{\hessian}{O{\theta}}{H_{#1}}
\NewDocumentCommand{\mlp}{}{\operatorname{MLP}}

\NewDocumentCommand{\relu}{}{\operatorname{ReLU}}
\NewDocumentCommand{\gelu}{}{\operatorname{GELU}}
\RenewDocumentCommand{\defeq}{}{\coloneq}

\RenewDocumentCommand{\transpose}{}{\top}

\NewDocumentCommand{\dbmlp}{}{\operatorname{DB-MLP}}
\NewDocumentCommand{\jrelu}{}{\operatorname{JSReLU}}
\NewDocumentCommand{\Alpha}{}{A}

\NewDocumentCommand{\magic}{O{}}{\mathrm{#1MagicSynapse}}

\renewcommand{\vec}[1]{{\bm{#1}}}
\renewcommand{\mat}[1]{\bm{#1}}
\newcommand{\sampleloss}{\ell}
\newcommand{\uniform}[1]{U\left\{#1\right\}}

\NewDocumentCommand{\param}{O{I}}{\vec{\theta}_{#1}}
\NewDocumentCommand{\allparam}{}{\param[]}

\NewDocumentCommand{\KVparam}{}{\param[\mathrm{KV}]}
\NewDocumentCommand{\Kparam}{}{\param[\mathrm{K}]}
\NewDocumentCommand{\Vparam}{}{\param[\mathrm{V}]}

\NewDocumentCommand{\tildeparam}{O{I}}{\tilde{\vec{\theta}}_{#1}}
\NewDocumentCommand{\tildeallparam}{}{\tildeparam[]}

\NewDocumentCommand{\tildeKVparam}{}{\tildeparam[\mathrm{KV}]}
\NewDocumentCommand{\tildeKparam}{}{\tildeparam[\mathrm{K}]}
\NewDocumentCommand{\tildeVparam}{}{\tildeparam[\mathrm{V}]}

\NewDocumentCommand{\AS}{O{\KVparam}}{\operatorname{AF}_{#1}}
\NewDocumentCommand{\tildeAS}{O{\tildeKVparam}}{\tilde{\operatorname{AF}}_{#1}}

\NewDocumentCommand{\weird}{}{\operatorname{weird}}
\NewDocumentCommand{\halfwidth}{}{w} \newcommand{\wrt}{w.r.t.}
\newcommand{\ie}{\textit{i.e.,}}
\newcommand{\eg}{\textit{e.g.,}}
\newcommand{\etal}{\textit{et al.}}
\newcommand{\lhs}{\textrm{LHS}}
\newcommand{\rhs}{\textrm{RHS}}

\newtheorem{definition}{Definition}

\newtheorem{assumption}{Assumption}
\newtheorem{lemma}{Lemma}
\newtheorem{theorem}{Theorem}
\newtheorem{corollary}{Corollary}

\crefname{definition}{Definition}{Definitions}
\crefname{assumption}{Assumption}{Assumptions}
\crefname{lemma}{Lemma}{Lemmas}
\crefname{theorem}{Theorem}{Theorems}
\crefname{remark}{Remark}{Remarks}
\crefname{appendix}{Appendix}{Appendices}
\crefname{section}{Section}{Sections}
\crefname{algorithm}{Algorithm}{Algorithms}
\crefname{corollary}{Corollary}{Corollaries}
\crefname{figure}{Fig.}{Figs.}
\crefname{table}{Table}{Tables}
\crefname{equation}{}{}

\newcommand{\whichisbetter}[1]{\scalebox{0.8}[1.0]{$(#1)$}}
\newcommand{\lowerbetter}{\whichisbetter{\downarrow}}
\newcommand{\higherbetter}{\whichisbetter{\uparrow}}

\newcommand{\whatisbold}{ Symbol \higherbetter{} indicates a higher value is better, while symbol \lowerbetter{}  means the opposite. The best performance and sparsity are marked \textbf{bold}.}  
\usepackage[table]{xcolor}

\newcommand{\gray}[1]{\textcolor{gray}{#1}}

\newcommand{\version}{pami}

\usepackage{ifthen}
\newcommand{\pamionly}[1]{\ifthenelse{\equal{\version}{pami}}{#1}{}}

\newcommand{\arxivonly}[1]{\ifthenelse{\equal{\version}{arxiv}}{#1}{}}

\newcommand{\minorrevision}[1]{\color{blue}#1\color{black}}
\renewcommand{\minorrevision}[1]{#1}

\newcommand{\minorrevisionimage}[1]{\fcolorbox{blue}{white}{#1}
}
\renewcommand{\minorrevisionimage}[1]{#1} \pamionly{
\NewDocumentCommand{\coderepo}{}{\url{https://github.com/cdgyp/sparsity}} }

\usepackage{xr-hyper}
\begin{document}

\maketitle

\begin{abstract}The observation that activation sparsity emerges in MLP blocks of standardly trained Transformers offers an opportunity to drastically reduce computation costs without sacrificing performance. 
To theoretically explain this phenomenon, existing works have shown that activation sparsity does not result from the data properties or data fitting but from the implicit bias of the training process. However, these connections are obtained with strong assumptions (\eg{} shallow networks, a small number of training steps, and special training techniques), which cannot be applied to deep models standardly trained with a large number of steps. 
Different from these works, we find that the flatness of loss landscapes is also closely related to the MLP activation sparsity and can serve as a weaker assumption because it naturally emerges in the standard training of deep networks without the above strong assumptions.  
Specifically, we find that 1) the MLP activation sparsity equals a ratio between ``augmented flatness'' (a weighted sum of flatness measures) and the product of the input norm and activation gradient of the MLP. We empirically find that this ratio decreases during training, leading to sparse activations. 
2) We also propose the notion of derivative sparsity, which reduces to activation sparsity under $\relu$, but further enables pruning in the backward propagation and is more stable than activation sparsity.
With the theoretical findings, we can further encourage activation sparsity by decreasing the numerator and increasing the denominator of the ratio: 1) To improve (lower) the flatness, we add different bias vectors to input tokens of MLP blocks to strengthen stochastic gradient noises that drive the model to a flat area.  
2) We restrict the lower bound of affine parameters in LayerNorm to increase the input norm of MLPs.
3) To increase the activation sparsity, we propose an activation function $\jrelu$ to encourage the search of parameters with sparse derivatives and sparse activations.
These plug-and-play modifications can effectively reduce the ratio and produce sparser activations. Experiments on ImageNet-1K and C4 demonstrate relative improvements of at least $36\%$ on inference sparsity and at least $50\%$ on training sparsity over vanilla Transformers, indicating further potential cost reduction in both inference and training.
\end{abstract}

\begin{IEEEkeywords}
    sparsity, flat minima, implicit bias
\end{IEEEkeywords} 
\section{Introduction}\label{sec:intro}

Deep neural networks (DNNs) have achieved impressive success on a wide range of tasks. Nevertheless, training and deploying DNNs require substantial computation, time, and especially energy, unlike their biological inspirations, which are highly energy-efficient. 
One possible reason for this difference is sparsity. In biological neural networks, evidence shows that only a small fraction of neurons spike simultaneously despite the existence of billions of neurons \cite{observation}. This activation sparsity is considered a contributing factor to their energy efficiency. However, it was unclear whether artificial DNNs exhibit similar sparsity during inference or training.

A recent work \cite{observation} discovers that, in two-layer MLP blocks, activation sparsity emerges \emph{without} explicit regularization, \ie{} only a small portion of neurons are activated (with non-zero activation) for a large portion of samples during inference in MLP blocks of feed-forward networks, ResNet, T5 \cite{t5}, ($\relu$) ViT \cite{vit}, MLP-Mixer \cite{mixer}, and other architectures across various tasks. This finding suggests aggressive neuron pruning during inference without sacrificing performance. 
For example, if non-activated neurons are ideally skipped, T5's inference FLOPs due to the second linear layers of MLP blocks can be astonishingly reduced by about $90\%$ \cite{observation} without performance loss.
To \emph{exploit} activation sparsity for faster inference, Mirzadeh \etal{} \cite{strikeback} observed that the large plateau of $\relu$ in $(-\infty, 0)$ allows many zero activations, which suggests the renaissance of ReLU in Transformer-based large language models to improve activation sparsity.
Several recent works \cite{liu2023deja,liu2025trainingfree} have demonstrated the acceleration by exploiting activation sparsity on large language models.
The achievable lossless acceleration of this method is positively correlated with the number of zero activations.
Therefore, understanding and improving activation sparsity becomes critical for further accelerating inference, and possibly training as well. 

\begin{table}
    \centering
    \caption{Comparison between assumptions in existing explanations of activation sparsity and ours.}
    \begin{tabular}{cccc}
        \toprule
                                &   Depth    &   Training Steps    &   Training Techniques\\
        \midrule
        \cite{observation}      &  1       &   1    &       SGD      \\
        \cite{sharpness_aware}          &   2         &   All        &   SAM\\
        \cite{from_noises}    &   Deep        &   All        &   Gaussian Noise\\
        \cite{large_step}     &   2, Diagonal         &   All        &   SGD \\
        Ours                            &    Deep    &    All        &    SGD\\
        \bottomrule
    \end{tabular}\label{table:assumptions}
\end{table}

However, the emergence of sparsity is still not well understood. Several studies \cite{observation,sharpness_aware,large_step,from_noises} have sought to explain it through training dynamics. 
Li \etal{} \cite{observation} explain activation sparsity in the last layer at the first update step by manually computing gradients.
They find that gradients have components that decrease activations in the last MLP block. 
When sharpness-aware (SA) optimization \cite{foret2020sharpness,wen2023how} is used, Andriushchenko \etal{} \cite{sharpness_aware} manually compute gradients for a shallow 2-layer MLP and find that gradients from the SA objective also contain components that reduce pre-activations. This reduction pushes activations toward zero, inducing sparsity in $\relu$ networks. \cite{large_step} considers second-order behaviors of SGD and proves sparsity on diagonal 2-layer MLPs, but the emergence of sparsity in general deep, non-diagonal networks remains conjectural. 
Bricken \etal{} \cite{from_noises} find that adding noise to samples improves sparsity. However, this noise is manually imposed and not part of standard augmentations. 
These studies have improved our understanding of activation sparsity. Specifically, they reveal the roles of training dynamics, especially noise \cite{observation, large_step, from_noises} and sharpness reduction \cite{sharpness_aware}, in its emergence. Nevertheless, these studies rely on strong assumptions, \eg{} \textit{shallow networks} \cite{observation,sharpness_aware,large_step}, \textit{a small number of training steps} \cite{observation}, and \textit{additional augmentations \cite{from_noises} or optimization objectives} \cite{sharpness_aware} that are absent from standard training protocols, as summarized in \cref{table:assumptions}. Therefore, a gap remains between these explanations and experiments \cite{observation}, where all these assumptions are violated, \ie{} \emph{activation sparsity emerges in deep networks optimized by standard SGD or variants for millions of steps under standard augmentations}.
 
To advance understanding, we rebase the emergence of activation sparsity on the inductive bias of flat minima.
Flat minima are minima located in wide basins \cite{flatness_generalization} of loss landscape and are favored by SGD because its stochastic gradient noise can drive parameters to escape sharp minima \cite{alpha_stable,escape,alignment}. Flat minima often lead to better generalization of deep neural networks \cite{flatness_generalization,izmailov2018averaging,jiang2019fantastic,foret2020sharpness,tan2024sharpness}.
Flatness is commonly measured by the nuclear \cite{chaudhari_stochastic,anti_PGD} or spectral \cite{cha2021swad,gam,fad} norm of the Hessian of the empirical loss \wrt{} parameters. The lower the measure, the flatter the minima.
Flatness has several desirable properties that can help fill these gaps:
1) Flatness allows us to analyze deep (instead of shallow) networks because it is measured \wrt{} all parameters. 
2) Under a suitably large learning rate and small batch size, flat minima can be obtained with SGD over many training steps, not only at early updates.
3) Its prevalence under standard training eliminates the need for special training techniques.
Therefore, flatness is a more suitable assumption for activation sparsity analysis.

After careful derivation, we find \emph{equalities} linking activation sparsity to a ratio between an augmented form of flatness (\ie{} a weighted sum of flatness measures of multiple losses) and the product of MLP input norms and gradient magnitudes on MLP activations in $\relu$-activated MLP blocks:
\begin{align}
    \frac{[\text{(Augmented) Flatness}]}{\norm{[\text{MLP inputs}]}^2 \times \left|[\text{Gradients}]\right|^2} = [\text{Activation Sparsity}].
\end{align}
We empirically measure the augmented flatness and the denominators of the ratio and observe that the denominators increase faster than the augmented flatness. As a result, the ratio decreases and activations become sparse. 
\begin{figure}
    \centering
    \includegraphics[width=\linewidth]{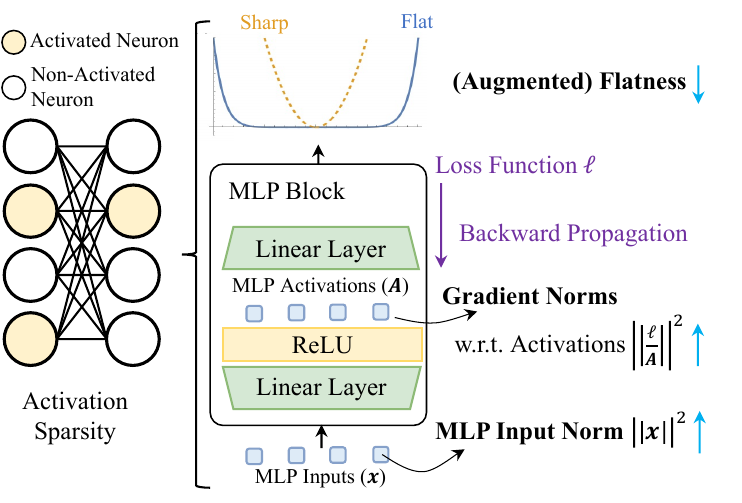}
    \caption{Overview of theoretical results. We connect activation sparsity with the augmented flatness of losses \wrt{} parameters, gradient norms \wrt{} activations, as well as the squared norms of MLP input tokens.}
    \label{fig:theory}
\end{figure}
 
We also propose the notion of derivative sparsity, defined as the number of non-zero derivatives of activation functions evaluated at pre-activations (inputs to activation functions) produced during forward propagation.
The derivative sparsity equals the activation sparsity when $\relu$ is used, because the $\relu$ activation is non-zero if and only if its derivative is non-zero (see \cref{fig:activations}).
Practically, derivative sparsity enables additional pruning during backward propagation and improves training efficiency.
With theoretical and empirical support, we find that derivative sparsity is a more stable notion than activation sparsity under some non-$\relu$ activation functions. Therefore, derivative sparsity is more closely related to implicit biases in training dynamics and is potentially useful for research on activation sparsity.

The equality between activation sparsity and the ratio indicates that activation sparsity can be improved by decreasing the numerator and increasing the denominator.
To this end, we identify factors in vanilla neural networks that hinder reducing this ratio and propose modifications to address them:
1) To decrease the numerator of the ratio, \ie{} to improve flatness in MLP blocks, we propose strengthening the gradient noise that MLP blocks are trained with. 
According to \cite{petzka_relative_2021}, the flatness of MLP blocks can be related to input robustness. Thus, it can be enhanced by noise on MLP input tokens.
We rely on gradient noise generated at parameters in shallower layers and forward-propagated to tokens, due to its larger magnitude \cite{alpha_stable} and alignment with sharp directions \cite{alignment}. 
However, for noise on MLP input tokens, we observe that parameter sharing in shallower layers (\ie{} shallower layers process all tokens with the same set of parameters) causes noise correlation among different input tokens, resulting in low noise diversity.
Besides, we also observe that nonlinearities (\eg{} $\relu$, attention, \textit{etc}) can suppress noise magnitude when saturated. 
Therefore, we propose adding different bias vectors to different tokens right before every MLP block so that gradient noises on these bias vectors are not shared or suppressed by nonlinearities.   
2) To increase the MLP input norms in the denominator, we impose lower-bounds on the affine parameters of LayerNorm layers before MLP blocks. 
3) With empirical evidence, we find that the training process implicitly decreases the norm of activation derivatives, which benefits derivative and activation sparsity.
However, we observe that $\relu$ with piecewise constant derivatives hinders this implicit optimization: when parameters are perturbed slightly, activation derivatives of $\relu$ remain constant, creating spurious minima for activation-derivative norms.
Therefore, we propose JSReLU with monotonically increasing derivatives to eliminate such spurious minima of activation derivative norms.

In experiments on ImageNet-1K \cite{imagenet1k} and C4 \cite{t5}, these plug-and-play modifications bring at least $50\%$ relative improvements in training sparsity and at least $36\%$ in testing sparsity compared to existing naturally emergent sparsity.
The contributions of this work are as follows:
\begin{itemize}
    \item To better understand the emergence of the activation sparsity, we derive equalities between activation sparsity and a ratio between flatness measures and the product of MLP input norms and gradients on MLP activations. We also give empirical results on its decrease;
    \item We propose the notion of derivative sparsity. We argue for its better stability than activation sparsity by showing that activation sparsity can be lost while derivative sparsity persists under some activation functions;
    \item We propose three modifications to architectures, activation functions, and normalization layers to improve activation sparsity.
    \item The proposed modifications achieve at least $36\%$ relative improvements over natural inference sparsity and at least $50\%$ over natural training sparsity in the pretraining of both natural image classification on ImageNet-1K and natural language generalization on C4.
\end{itemize}

The rest of the manuscript is organized as follows: \cref{sec:related} lists related works; \cref{sec:preliminary} introduces notations; \cref{sec:theory} derives theoretical results and presents empirical evidence for the connection between flatness and sparsity; \cref{sec:stability} argues for the value of derivative sparsity; \cref{sec:algorithm} proposes architectural modifications based on theoretical results and empirical observations; \cref{sec:method_experiments} empirically evaluates the modifications followed by ablation studies; \cref{sec:conclusion} summarizes the manuscript, discusses limitations, and outlines potential future work.
Proofs can be found in the Appendix.
\pamionly{The code for experiments can be found on GitHub\footnote{\coderepo{}}.} 
\section{Related Works}\label{sec:related}

In this section, we review related work on the particular roles of neurons in MLP blocks and on understanding the emergence of activation sparsity in these neurons. We also discuss works that further induce and exploit this sparsity for improved computational efficiency. We finally list other notions of emergent sparsity.

\subsection{MLP Blocks and Knowledge Neurons}
   
We trace research on activation sparsity in MLP blocks back to the discovery of the relation between MLP blocks and parametric memory. \cite{mlp_as_database} rewrites Transformer MLPs as an unnormalized attention mechanism, where queries are inputs to the MLP block and keys and values come from the first and second weight matrices rather than inputs. In this way, MLP blocks can be viewed as key-value memories. 
\cite{knowledge_neurons} extends this line by detecting how each key-value pair relates to each question in Q\&A tasks. By exploiting activation magnitudes and their gradients, they provide a method that surgically manipulates answers to individual questions. These works redirect attention to MLP blocks, which were previously overshadowed by self-attention. 

\subsection{Understanding Activation Sparsity}

Recently, comprehensive experiments by \cite{observation} demonstrate that activation sparsity in MLP blocks is prevalent across architectures and across vision and language tasks. 
\cite{observation} also rules out alternative explanations and attributes activation sparsity solely to training dynamics. 
The authors explain sparsity theoretically by exploiting properties of random initialization and manually calculating gradients. They find that gradients \wrt{} activations in the last MLP block have positive components that decrease activations and thus push them toward suppression. However, their explanation is restricted to the last layer (due to manual gradient computation) and to the first training step (because their analysis relies on independence between weights and samples, which no longer holds later in training). 
They also discover that some activation functions, such as $\tanh$, hinder sparsity (see Fig. B.3(c) in \cite{observation}), but do not elaborate further. 
Compared with their explanation, ours applies to all layers and many training steps, and it accounts for the critical role of activation functions in activation sparsity.

Following these empirical discoveries, \cite{sharpness_aware} shows that sharpness-aware (SA) optimization has a stronger bias toward activation sparsity. 
They explain this theoretically by manually calculating gradients and finding that SA optimization introduces gradient components that reduce activation norms. However, their explanation is still restricted to shallow 2-layer pure MLPs and requires SA optimization, which is not part of standard training practice. Nevertheless, this explanation hints at the role of flatness in the emergence of activation sparsity. Inspired by this, we explain \emph{deep} networks trained by standard SGD or its variants by using flat minima in place of SA optimization.

A more recent work \cite{from_noises} points out that sparsity is resistance to noise. However, the noise is manually imposed and not included in standard data augmentations. We instead consider gradient noise from SGD or other stochastic optimizers.
\cite{large_step} proves sparsity in 2-layer diagonal MLPs and conjectures that a similar process occurs in more general networks. Both works hint at a relation between noise (Gaussian noise added to inputs and stochastic gradient noise) and activation sparsity, and thus to implicit flatness bias.

\cite{adversarial_of_moe} studies the adversarial robustness of Mixture-of-Experts (MoE) models induced by architecture-imposed sparsity. 
We are inspired by this work to relate sparsity to adversarial robustness, although we do so in a reverse and implicit manner.

\subsection{Sparsity Inducing and Exploitation}

Although it does not focus on explaining the emergence of activation sparsity in CNNs, \cite{exploit_sparsity_in_CNN} boosts activation sparsity through Hoyer regularization \cite{hoyer} and a new activation function, FATReLU, which uses dynamic thresholds between activation and deactivation. 
They also design algorithms to exploit this sparsity, leading to $\ge 1.75\mathrm{x}$ speedup in CNN inference. Compared with their sparsity-encouraging method, which requires carefully designed threshold-selection procedures, hyperparameters for our theoretically motivated modifications are easier to select. The discontinuity of FATReLU also hinders training from scratch \cite{exploit_sparsity_in_CNN}, whereas we recommend applying our modifications from scratch to obtain better sparsity and lower \emph{training} cost. 
\cite{L1_sparsity} encourages activation sparsity in CNNs via explicit $L_1$ regularization. We instead investigate the emergence of activation sparsity from implicit regularization as demonstrated by \cite{observation}, and thus rely only on implicit regularization enhanced by our modifications. Nevertheless, our methods are architecturally orthogonal, and we believe combining both could further boost activation sparsity.
\cite{strikeback} suggests $\relu$ for Transformer-based LLMs because smoother activation functions like $\gelu$ or $\text{SiLU}$ do not provide a large zero-activation plateau where many activations can reside. They also propose a GPU-implemented pruning method, with which inference latency is greatly reduced as sparsity increases.

\subsection{Other Notions of Emergent Sparsity}

Another notion of sparsity in hidden features is attention sparsity, \ie{} many entries in the attention map are near zero. It is the focus of \cite{attention_sparsity_1,attention_sparsity_2}.
For input-side sparsity notions, \cite{sparse_symbol} uses Shapley values to formulate and prove the existence of sparse ``symbols,'' namely small patch groups that are the main contributors to the output of well-trained and masking-robust AI models.
 
\section{Preliminary}\label{sec:preliminary}

Here we introduce the notation. \cref{sec:notations} presents basic symbols, \cref{sec:arch} defines the architecture of interest, and \cref{sec:def_sparsity} defines flatness measures and sparsity.

\subsection{Notation and Problem Setting}\label{sec:notations}

For $n \in \nats$, let $[n] = \set{1, 2, \dots, n}$.
We use bold lowercase letters like $\vec{x}$ for vectors and bold uppercase letters like $\mat{X}$ for matrices.
When the context is clear, we use $x_j$ for an element of vector $\vec{x}$, $\vec{x}_i \in \reals^{d}$ for the \emph{transpose} of a row in matrix $\mat{X} \in \reals^{k \times d}$, and $x_{i, j}$ for an element of matrix $\mat{X}$. 
Let $\norm{\cdot}_2$ denote the $L_2$ norm for vectors and $\norm{\mat{X}}_F$ denote the Frobenius norm for matrices. For vectors or matrices, the $L_0$ pseudo norm $\norm{\cdot}_0$ counts their non-zero elements.
Let $p$ be the number of parameters.
For a parameterized function $f: \reals^p \times \mathcal{A} \to \mathcal{B}$, we write $f_{\allparam}$ for the partial application of a parameter $\allparam$. We also use $f_{\allparam}: \mathcal{A} \to \mathcal{B}$ to indicate $f: \reals^p \times \mathcal{A} \to \mathcal{B}$.

We assume classification tasks throughout the analysis. 
Let $\dataset = \set{(\vec{x}_i, y_i): i \in [N]} \subseteq \reals^{d_{\textnormal{input}}} \times [C]$ be the \emph{finite} training set. Our focus is solely on activation sparsity within \emph{training} samples, as in \cite{observation,sharpness_aware}. Since the dataset is finite, the analysis involves finite summations rather than integrals.
For classification tasks, $f_{\allparam}$ outputs a distribution over $[C]$, where $f(c \mid \vec{x}, \allparam)$ denotes the estimated probability of class $c$ for sample $\vec{x}$ under parameter $\allparam$. 
We define $\sampleloss(f_{\allparam}, (\vec{x}, y)) \defeq -\log f(y \mid \vec{x}, \allparam)$ as the cross-entropy loss on a sample, and $\hat{\loss}(\allparam) \defeq \ex[(\vec{x}, y) \sim \uniform{\dataset}]{\ell(f_{\allparam}, (\vec{x}, y))}$ as the empirical loss on the training set, where $\uniform{\cdot}$ denotes the uniform distribution on a non-empty finite set.
Our analysis naturally extends to self-supervised objectives in language-model pretraining, since they are essentially token prediction with cross-entropy loss. 

\subsection{Architectures}\label{sec:arch}

We assume that the network of interest $f_{\allparam}$ comprises $L$ MLP blocks interleaved with other structures:
\begin{align}
    f_{\allparam} \defeq \text{softmax} \circ D_{\allparam} \circ \left(\mathop{\bigcirc}_{l=1}^L (\mlp^l_{\allparam} \circ G^l_{\allparam})\right) \circ E,\label{eq:architecture}
\end{align}
where $\mathop{\bigcirc}$ denotes function composition, $E$ is the input embedding layer, $D_{\allparam}$ is the classifier, and $\mlp^l_{\allparam}: \reals^{k \times d} \to \reals^{k \times d}$ takes a matrix $\mat{X}^l$ of $k$ $d$-dimensional tokens and outputs a matrix $\mat{Z}^l$ of the same shape. 
$G^l_{\allparam}$ encompasses non-MLP structures, including normalization layers, self-attention layers, \textit{etc}.
Modules are assumed to use disjoint parameters from $\allparam$.
Specifically, $\mlp^l_{\allparam}$ uses parameters from $\allparam$ to define $(\mat{K}^l, \vec{b}^{l, K}) \in \reals^{n \times d} \times \reals^{n}, (\mat{V}^l, \vec{b}^{l, V}) \in \reals^{d \times n} \times \reals^{d}$, and computes $\mat{Z}^l$ by 
\begin{align}
    \mat{P}^l \defeq& \mat{X}^l \times \left(\mat{K}^l\right)^\transpose + \vec{1}_k \times \left(\vec{b}^{l, K}\right)^\transpose \in \reals^{k \times n},\label{eq:mlp_start}\\
    \mat{\Alpha}^l \defeq& \sigma\left(\mat{P}^l\right) \in \reals^{k \times n},\\
    \mat{Z}^l \defeq& \mat{\Alpha}^l \times \left(\mat{V}^l\right)^\transpose +  \vec{1}_k \times \left(\vec{b}^{l, V}\right)^\transpose \in \reals^{k \times d}, \label{eq:mlp_end}
\end{align}
where $\vec{1}_k \in \reals^{k}$ is the all-$1$ vector, and $\sigma: \reals \to \reals$ is the activation function.
Here, $\mat{P}^l$ is the pre-activation matrix and $\mat{A}^l$ is the activation matrix. 
$\mat{K}^l$ and $\mat{V}^l$ are referred to as key and value matrices, following \cite{mlp_as_database,knowledge_neurons}. Let $\Kparam^l, \Vparam^l$ denote the parameters in $\mat{K}^l, \mat{V}^l$, respectively, and let $\KVparam$ denote the parameters of all MLP matrices.
Since the number of tokens affects the analysis, we distinguish networks with different token counts in
\cref{def:token_difference}. 
\begin{definition}[Single- and Multi-Token Scenarios]\label{def:token_difference}
    $f$ is single-token if $k=1$ and is multi-token if $k > 1$.
\end{definition}
According to this definition, feed-forward networks are single-token networks, while CNNs (with receptive fields as tokens), Transformers, and MLP-Mixers are multi-token networks. For clarity, skip connections (\eg{} residual and dense connections) are omitted in \cref{eq:architecture}. Nevertheless, our results extend to architectures with skip connections, as detailed in \cref{appendix:full}.

\subsection{Gradients, Flatness and Sparsities}\label{sec:def_sparsity}

Given a scalar function $g$ of a matrix or vector $\mat{X}$, we use $\derivatives{g}{\mat{X}} \defeq \begin{bmatrix} \derivatives{g}{x_{i, j}} \end{bmatrix}_{i, j}$ to denote partial derivatives.
Our analysis involves non-smooth activation functions such as $\relu(x) = \max\set{x, 0}$. With \cref{assumption:differentiable} and \cref{def:subdrivatives}, we treat undefined derivatives as $0$ when computing gradients, supported formally by \cref{lemma:differentiable} in \cref{appendix:full}.
\begin{assumption}\label{assumption:differentiable}
    At every $(\vec{x}, y) \in \dataset$, the model $f$ and parameter $\allparam$ satisfy that $f_{\allparam}$ is twice continuously differentiable \wrt{} every hidden feature and parameter.
\end{assumption}
\begin{definition}\label{def:subdrivatives}
    The (zero-assigned) subderivative matrix $\mat{D}^l$ of $\mlp^l_{\allparam}$ on sample $(\vec{x}, y) \in \dataset$ is $\mat{D}^l \in \reals^{k \times n}$, where
    \begin{align}
        d^l_{i, j} \defeq \begin{cases}
            \sigma'(p^l_{i, j}) & \text{$\sigma'(p^l_{i, j})$ exists},\\
            0 & \text{otherwise}.
        \end{cases}
    \end{align}
\end{definition}
The Hessian of a scalar function $g$ \wrt{} $\allparam$ is $\nabla^2_{\allparam} g \defeq \begin{bmatrix} \frac{\partial^2 g}{\partial \theta_i \partial \theta_j} \end{bmatrix}_{i, j}$. Specifically, we denote the Hessian of the empirical loss by $\hessian[\param] \defeq \nabla^2_{\param} \ex[(\vec{x}, y) \sim \uniform{\dataset}]{\sampleloss(f_{\allparam}, (\vec{x}, y))}$, where $\param$ is any subset of $\allparam$.
We measure the flatness with $\trace{\hessian[\param]}$ \cite{anti_PGD}, which is the same as $\norm{\hessian[\param]}_{*}$ used in \cite{chaudhari_stochastic} at minima and lower-bounds the latter in general cases.

Our analysis uses a sum of Hessian traces, as in \cref{def:augmented_sharpness}, due to its connection to gradient norms.
\begin{definition}\label{def:augmented_sharpness}
    The \underline{a}ugmented \underline{f}latness of $f_{\allparam}$ on training set $\dataset$ \wrt{} a subset of parameters $\param$ is
    \begin{align}
        \AS[\param] \defeq \trace{\hessian[\param] + \ex[(\vec{x}, y)\sim \uniform{\dataset}]{\frac{\nabla^2_{\param} f(y \mid \vec{x}, \allparam)}{f(y \mid \vec{x}, \allparam)}}}\label{eq:augmented_flatness}.
    \end{align}
\end{definition}
The augmented flatness combines the trace of the empirical-loss Hessian and the expected Hessian traces of the estimated probabilities of true labels, weighted by those probabilities. The Hessian traces in the second term serve as (weighted) flatness measures for a classification loss that directly uses the estimated probability of the true label as the loss. Therefore, we also consider this term a flatness measure.

Sparsities are defined using hidden features and derivatives.

\newcommand{\definesparsity}[4]{
    \begin{definition}[#1 Sparsity]\label{def:#2_sparsity}
        The #2 sparsity of $\mlp^l$ on sample $\vec{x}$ is $\norm{\mat{#3}^l}_0$.
        The training #2 sparsity of $\mlp^l$ on training set $\dataset$ is $\ex[(\vec{x}, y) \sim \uniform{\dataset}]{\norm{\mat{#3}^l}_0}$.
        #4
    \end{definition}
}

\definesparsity{Activation}{activation}{A}{}
\definesparsity{Derivative}{derivative}{D}{}

Under $\relu$ activation, we have $\norm{\mat{A}^l}_0 = \norm{\mat{D}^l}_0$ because for $x \neq 0$, $\relu'(x) > 0$ if and only if $\relu(x) > 0$. 
\section{Connection between Flatness and Sparsity}\label{sec:theory}

In this section, we present a theoretical analysis that reveals the connection between flatness and activation sparsity in MLP blocks.
First, we provide the main theorems on the equality between sparsity and the ratio.
Next, we support the decrease of the ratio with empirical results.

\subsection{Analysis and Theoretical Results}\label{sec:analyses}

This subsection presents the main theoretical results connecting flatness and activation sparsity.
After noting that gradient norms are linked to augmented flatness and depend on activation derivatives, we first establish the connection between augmented flatness and gradient norms. 
Then, for single-token scenarios, we decompose the gradient norms to obtain derivative sparsity. Since activation and derivative sparsity are equal for $\relu$, we then obtain the connection between augmented flatness and activation sparsity for single-token networks. 
Finally, we extend the single-token results to multi-token scenarios. Details are in \cref{appendix:full}.

We first establish an equality between augmented flatness and expected squared gradient norms under cross-entropy loss. 
\begin{lemma}\label{lemma:af_and_gradient_norm}
    Under \cref{assumption:differentiable}, for any subset of parameters $\param \subseteq \allparam$, we have
    \begin{align}
        \AS[\param]=\ex[\vec{s} \sim \uniform{\dataset}]{\norm{\nabla_{\param} \sampleloss(f_{\allparam}, \vec{s})}_2^2}.\label{eq:af_and_gradient_norm}
    \end{align}
\end{lemma}
 
To compute gradients, back-propagation uses hidden features and derivatives, thereby linking gradient norms (and augmented flatness) with activation and derivative matrices. 
This connection allows us to decompose the expected gradient norms of MLP matrices in \cref{lemma:af_and_gradient_norm} with $\mat{A}^l$ and $\mat{D}^l$, leading to \cref{lemma:main} (details are in \cref{lemma:main_full} of \cref{appendix:full}). 
\begin{lemma}\label{lemma:main} 
    Under \cref{assumption:differentiable}, for $l \in [L]$, we have
    \begin{align}
\AS[\Kparam^l]
        =& \ex[\vec{s} \sim \uniform{\dataset}]{\norm{\left(\mat{X}^l\right)^\transpose \left(\derivatives{\sampleloss(f_{\allparam}, \vec{s})}{\mat{A}^l} \hadamard \mat{D}^l \right)}_F^2}, \label{eq:main_K}\\
        \AS[\Vparam^l]=&\ex[\vec{s} \sim \uniform{\dataset}]{\norm{\left(\mat{A}^l\right)^\transpose \derivatives{\sampleloss(f_{\allparam}, \vec{s})}{\mat{Z}^l}}_F^2}. \label{eq:main_V}
    \end{align}
\end{lemma}

To obtain a clearer relationship between flatness and derivative sparsity $\norm{\mat{D}^l}_0$, the number of tokens should be differentiated in the following sections.

\subsection{Single-Token Scenarios}

In single-token scenarios, all matrices in \cref{lemma:main} reduce to vectors. This simplification allows us to interchange matrix multiplication and the Frobenius norm, resulting in \cref{lemma:single_token} (details are in \cref{lemma:single_token_full}).

\begin{lemma}\label{lemma:single_token}
    Under \cref{assumption:differentiable}, if $f_{\theta}$ is single-token, then for $l \in [L]$,
    \begin{align}
        \AS[\Kparam^l] =& \ex[\vec{s} \sim \uniform{\dataset}]{\norm{\derivatives{\sampleloss(f_{\allparam}, \vec{s})}{\mat{A}^l} \hadamard \mat{D}^l}_F^2 \norm{\mat{X}^l}_F^2}, \label{eq:single_token_K}\\
        \AS[\Vparam^l] =& \ex[\vec{s} \sim \uniform{\dataset}]{\norm{\derivatives{\sampleloss(f_{\allparam}, \vec{s})}{\mat{Z}^l}}_F^2 \norm{\mat{A}^l}_F^2}.\label{eq:single_token_V} 
    \end{align}
\end{lemma}

To extract derivative sparsity $\norm{\mat{D}^l}_0$ from \cref{lemma:single_token}, we exploit properties of $\relu$ and LayerNorm.
First, for $\relu$ networks, when $x \neq 0$, the derivative satisfies $\relu'(x) = 1$ if $\relu(x) > 0$ and $\relu'(x) = 0$ if $\relu(x) = 0$. As a result, the subderivatives in $\mat{D}^l$ act as indicators of whether $\relu$ neurons are activated. Since indicators are 0-1 valued and the Frobenius norm of a 0-1 matrix equals its $L_0$ pseudo norm, one can extract sparsity measured by the $L_0$ pseudo norm from Frobenius norms with the help of \cref{lemma:L0_and_L2} in \cref{appendix:full}. 
Second, to extract $\norm{\mat{X}^l}_F^2$, we recall that in modern architectures such as Transformers and MLP-Mixer, inputs to MLP blocks are normalized by LayerNorm. If the affine parameters in LayerNorm are turned off and $\epsilon_{\textnormal{LayerNorm}}$ (the value added to the denominator in LayerNorm for numerical stability) is set to $0$, then every (and here the only) token in $\mat{X}$ has $L_2$ norm $\sqrt{d}$, allowing us to extract $\norm{\mat{X}^l}_F^2$ and leave only $\norm{\derivatives{\sampleloss(f_{\allparam}, \vec{s})}{\mat{A}^l} \hadamard \mat{D}^l}_F^2$ inside the expectation.
\cref{theorem:single_token} summarizes the contributions of LayerNorm and $\relu$ on activation sparsity. The theorem is proved as an implication of \cref{theorem:single_token_full} in \cref{appendix:full}.

\begin{theorem}\label{theorem:single_token}
    Under \cref{assumption:differentiable}, if $f_{\allparam}$ is used under single-token scenarios, for $l \in [L]$, we have
    \begin{align}
        \ex[\vec{s} \sim \uniform{\dataset}]{\norm{\mat{D}^l}_0}
        =&\frac{\AS[\Kparam^l]}{\ex{\norm{\mat{X}^l}_F^2 \left(\derivatives{\sampleloss(f_{\allparam}, \vec{s})}{a^l_{i, j}} \cdot d^l_{i, j}\right)^2 \mid d^l_{i, j} > 0}}, \label{eq:single_token_expectation}\\
        \AS[\Vparam^l]=& \ex[\vec{s} \sim \uniform{\dataset}]{\norm{\derivatives{\sampleloss(f_{\allparam}, \vec{s})}{\mat{Z}^l}}_F^2 \norm{\mat{A}^l}_F^2},
    \end{align}
    where the conditional expectation in \cref{eq:single_token_expectation} is taken over $\vec{s} \sim \uniform{\dataset}, (i, j) \sim \uniform{[k] \times [n]}$.
    If further MLP blocks use $\relu$ as activation functions, then we have 
    \begin{align}
        \ex[\vec{s} \sim \uniform{\dataset}]{\norm{\mat{A}^l}_0}
        =&\ex[\vec{s} \sim \uniform{\dataset}]{\norm{\mat{D}^l}_0}\\
        =&\frac{\AS[\Kparam^l]}{\ex{\norm{\mat{X}^l}_F^2 \left(\derivatives{\sampleloss(f_{\allparam}, \vec{s})}{a^l_{i, j}}\right)^2 \mid a^l_{i, j} > 0}}. 
    \end{align}
    If further inputs to MLP blocks are LayerNorm-ed with LayerNorms' affine parameters and $\epsilon_{\textnormal{LayerNorm}}$ turned off, we have
    \begin{align}
        \ex[\vec{s} \sim \uniform{\dataset}]{\norm{\mat{A}^l}_0}
        =&\frac{\AS[\Kparam^l]}{d \cdot \ex{\left(\derivatives{\sampleloss(f_{\allparam}, \vec{s})}{a^l_{i, j}}\right)^2 \mid a^l_{i, j} > 0}}. \label{eq:main}
    \end{align}
\end{theorem}

In \cref{theorem:single_token}, derivative sparsity $\norm{\mat{D}^l}_0$ and activation norm $\norm{\mat{A}^l}_F^2$ are connected to augmented flatness \wrt{} key matrices and value matrices, respectively. Activation sparsity $\norm{\mat{A}^l}_0$ is then plugged in via properties of $\relu$. As a result, activation sparsity equals a ratio between augmented flatness \wrt{} key matrices and a product related to MLP input norms and magnitudes of gradients \wrt{} activations.
 
\subsection{Multi-Token Scenarios}
When multiple tokens are used, as in CNNs, Transformers, and MLP-Mixers, the derivation is non-trivial because $\mat{K}^l, \mat{V}^l$ are shared among tokens and gradients from different tokens may cancel out, preventing results similar to \cref{lemma:single_token}.
Fortunately, if we copy the shared $\mat{K}^l, \mat{V}^l$ $k$ times ($k$ is the token number) and assign each copy to process exactly one token as in \cref{def:unwrapped}, we obtain a single-token network to which we can directly apply the single-token results.

\begin{definition}\label{def:unwrapped}
    Let $f_{\allparam}$ be a neural network defined by \cref{eq:architecture}, then its copied architecture $\tilde{f}$ is defined by 
    \begin{align}
        \tilde{f}_{\tildeallparam} \defeq
            \operatorname{softmax} \circ D_{\tildeallparam} \circ \left(\mathop{\bigcirc}_{l=1}^L (\operatorname{kMLP}^l_{\tildeallparam} \circ G^l_{\tildeallparam})\right) \circ E,\label{eq:unwrapped}
    \end{align}
    where $\operatorname{kMLP}^l_{\tildeallparam}: \reals^{k \times d} \to \reals^{k \times d}$ is defined by
    \begin{align}
        \operatorname{kMLP}^l_{\tildeallparam}(\tilde{\mat{X}}^{l})
        =&  \begin{bmatrix}
            \mlp^{l, i}_{\tildeallparam}(\tilde{\vec{x}}_i^\transpose)
        \end{bmatrix}_{i \in [k]},
    \end{align}
    and $\mlp^{l, i}_{\tildeallparam}$ has parameters $(\mat{K}^{l, i}, \vec{b}^{l, i, K}) \in \reals^{n \times d} \times \reals^{n}, (\mat{V}^{l, i}, \vec{b}^{l, i, V}) \in \reals^{d \times n} \times \reals^{d}$.
    The copied parameter $\tildeallparam$ of $\allparam$ is constructed by copying parameters of $D$ and $G^l$, and by copying parameters of $\mlp^l$ in $f_{\allparam}$ $k$ times to $\mlp^{l, i}$ of $f_{\tildeallparam}$.
    Denote the input and output of $\text{k}\mlp^l_{\tildeallparam}$ by $\tilde{\mat{X}}^l$ and $\tilde{\mat{Z}}^l$, respectively.
    Let $\tildeKparam^l$ and $\tildeVparam^l$ be the parameters in all $\mat{K}^{l, i}$ and $\mat{V}^{l, i}$ of layer $l$, respectively.
    Let $\tildeAS[\tildeparam]$ be the augmented flatness in $\tilde{f}_{\tildeallparam}$ \wrt{} any subset of parameters $\tildeparam$ of $\tildeallparam$.
\end{definition}

Similar connections between augmented flatness, activation sparsity, and activation norms, as in \cref{theorem:single_token}, can be derived on $\tilde{f}_{\tildeallparam}$.
Additionally, one can easily verify inductively that $(\tilde{\mat{X}}^l, \tilde{\mat{Z}}^l) = (\mat{X}^l, \mat{Z}^l)$ on every sample and gradients \wrt{} them are also the same.
We then transfer the results on $\tilde{f}_{\tildeallparam}$ to $f_{\allparam}$ through this equivalence, obtaining \cref{theorem:multi_token}, which is an implication of \cref{theorem:multi_token_full} in \cref{appendix:full}.

\begin{theorem}\label{theorem:multi_token}
    Assume that $f_{\theta}$ satisfies \cref{assumption:differentiable} on $\dataset$.
    Then we have
    \begin{align}
        \ex{\norm{\mat{D}^l}_0}
        =&\frac{\tildeAS[\tildeKparam^l]}{\ex{\norm{\vec{x}^l_i}_2^2 \left(\derivatives{\sampleloss}{a^l_{i, j}}\right)^2 \left(d^{l}_{i, j}\right)^2 \mid d^l_{i, j} > 0}},\label{eq:multi_token_derivative_sparsity}\\
        \tildeAS[\tildeVparam^l]=& \ex[\vec{s} \sim \uniform{\dataset}]{\sum_{i \in [k]} \norm{\derivatives{\sampleloss}{\vec{z}^l_i}}_F^2 \norm{\vec{a}^l_i}_2^2}\label{eq:multi_token_norm}.
    \end{align}
    where the expectation in \cref{eq:multi_token_derivative_sparsity} is taken over $\vec{s} \sim \uniform{\dataset}$ and $(i, j) \in \uniform{[k] \times [n]}$.
    If all MLP blocks use $\relu$ activation functions, we further have
    \begin{align}
        \ex{\norm{\mat{A}^l}_0}
        =&\ex{\norm{\mat{D}^l}_0}\\
        =&\frac{\tildeAS[\tildeKparam^l]}{\ex{\norm{\vec{x}^l_i}_2^2 \left(\derivatives{\sampleloss}{a^l_{i, j}}\right)^2 \mid a^l_{i, j} > 0}}.\label{eq:multi_token_sparsity}
    \end{align}
    If further all inputs to MLP blocks are LayerNorm-ed with affine parameters and $\epsilon_{\textnormal{LayerNorm}}$ turned off, we further have
    \begin{align}
        \ex{\norm{\mat{A}^l}_0}
        =&\frac{\tildeAS[\tildeKparam^l]}{d \cdot \ex{\left(\derivatives{\sampleloss}{a^l_{i, j}}\right)^2 \mid a^l_{i, j} > 0}}. \label{eq:multi_token_sparsity_layernormed}
    \end{align}
\end{theorem}

In \cref{theorem:multi_token}, we obtain results similar to \cref{theorem:single_token}, except that augmented flatness is defined on the copied version $\tilde{f}_{\tildeallparam}$ of $f_{\allparam}$.
We will measure this augmented flatness, the denominator of \cref{eq:multi_token_sparsity}, and the expected average coefficients before activation norms in \cref{eq:multi_token_norm}.
It turns out that as the training proceeds, the ratio in \cref{eq:multi_token_sparsity} continues to decrease, leading to activation sparsity.  
The coefficients before $\norm{\vec{a}_i^l}_2^2$ also increase faster than $\tildeAS[\tildeVparam]$, so activation norms generally decrease.
The augmented flatness in the copied network may seem somewhat distant from the flatness of the original network. To give it a more concrete interpretation, we show that it can be connected to robustness against hidden-feature noise. See \cref{theorem:multi_token_full} in \cref{appendix:full} for details. 
\subsection{Empirical Results}\label{sec:coefficients}

\begin{figure}
    \centering
    \includegraphics[width=0.8\linewidth]{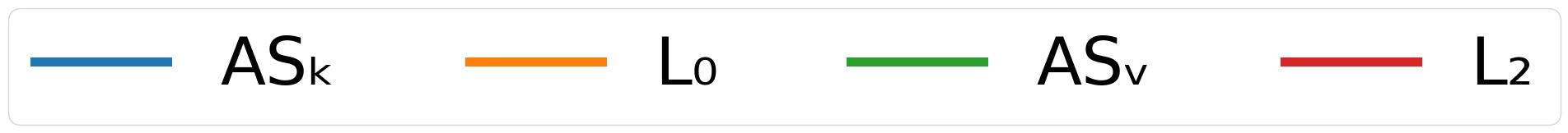}\\
    \subfloat[CIFAR-10\label{fig:full_run_cifar10}]{
        \includegraphics[width=0.48\linewidth]{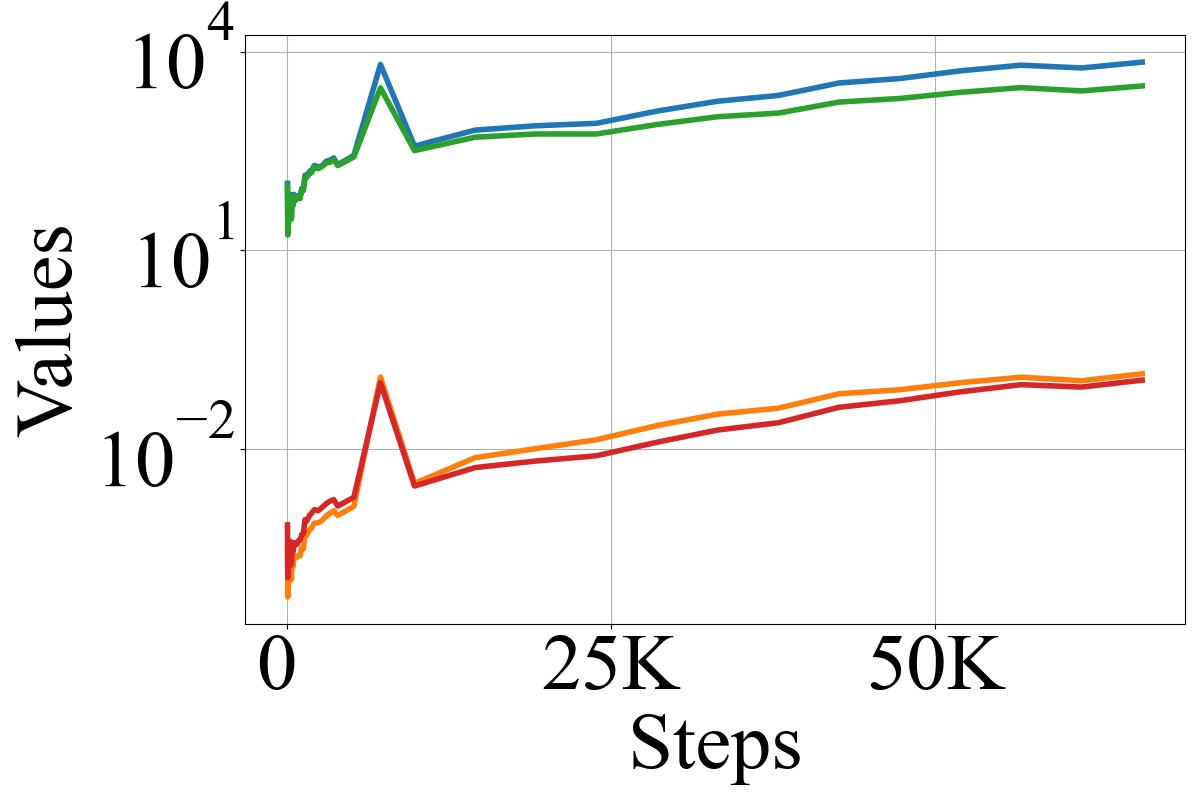}
    }\subfloat[ImageNet-1K\label{fig:full_run_imagenet1k}]{
        \includegraphics[width=0.48\linewidth]{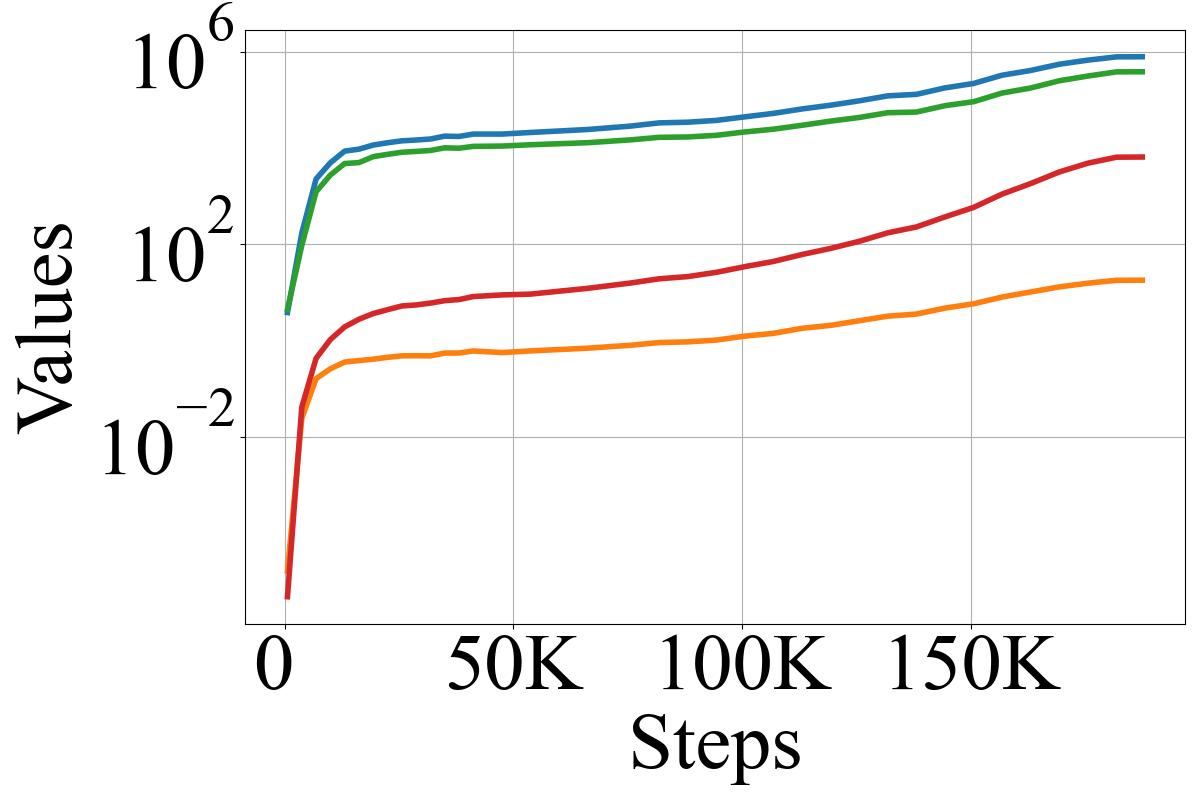}
    }\\
    \subfloat[Sparsity on CIFAR-10\label{fig:cifar10_sparsity}]{
        \includegraphics[width=0.48\linewidth]{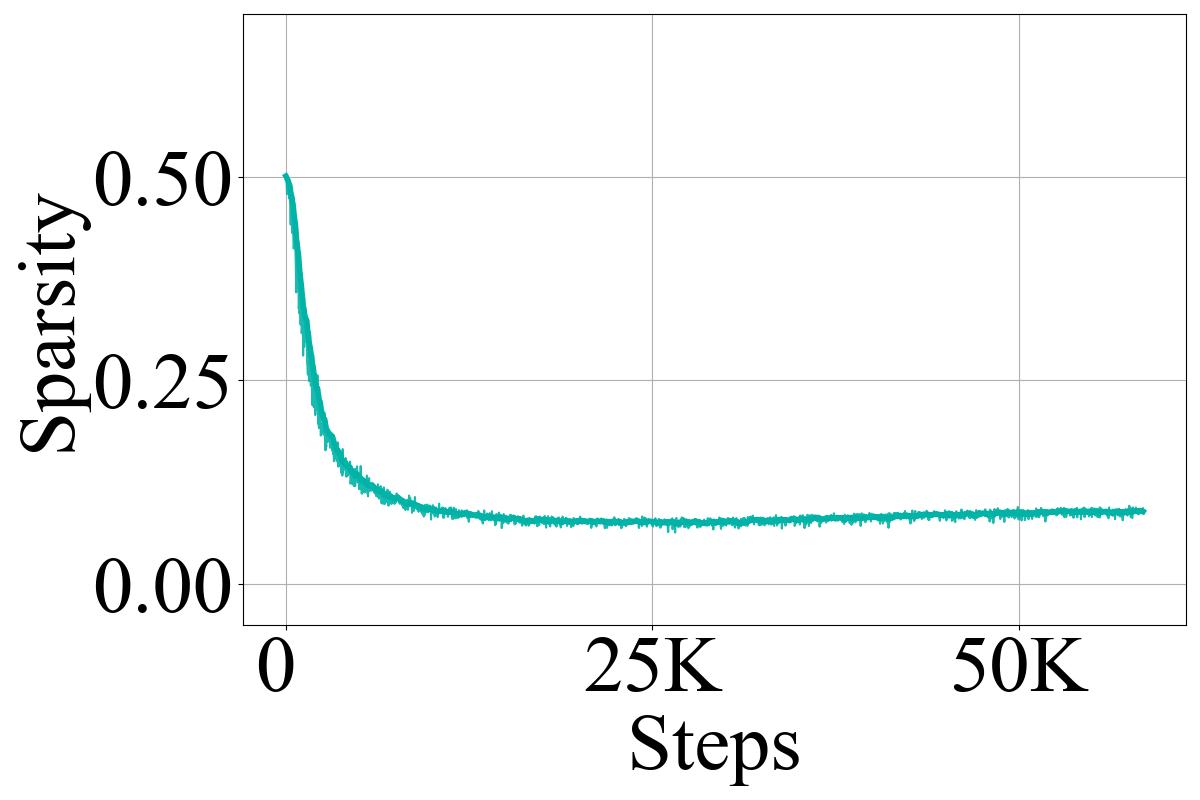}
    }
    \caption{Evolution of $\tildeAS[\tildeKparam]^l, c_{L_0}^l, \tildeAS[\tildeVparam]^l, c_{L_2}^l$ during the entire training on CIFAR-10 and ImageNet-1K as well as the training sparsity on CIFAR-10.}\label{fig:full_run}
\end{figure}

We focus on the multi-token case and measure how quantities in \cref{theorem:multi_token} evolve during Transformer training.
Specifically, we estimate the augmented flatness $\tildeAS[\tildeKparam^l]$ and its denominator in \cref{eq:multi_token_sparsity}
\begin{align}
    c_{L_0}^l \defeq \ex{\norm{\vec{x}^l_i}_2^2 \left(\derivatives{\sampleloss}{a^l_{i, j}}\right)^2 \mid a^l_{i, j} > 0}
\end{align}
as well as the augmented flatness $\tildeAS[\tildeVparam^l]$ and the \emph{expected average} coefficients in \cref{eq:multi_token_norm}
\begin{align}
    c_{L_2}^l \defeq \ex[\vec{s} \sim \uniform{\dataset}]{\frac{1}{k} \sum_{i=1}^k \norm{\derivatives{\sampleloss}{\vec{z}^l_i}}_2^2}
\end{align}
for every layer $l \in [L]$ at every checkpoint during training.
We train vanilla ViT-Base/16 \cite{vit} (with $\relu$ activation as in \cite{observation}) on CIFAR-10 \cite{cifar10} and ImageNet-1K \cite{imagenet1k} for natural image classification. Experimental details are in \cref{appendix:details}.
The estimates over the full training process are shown in \cref{fig:full_run}, while those from the initial epochs are shown in \cref{fig:initial_epochs}. To save space, results are averaged across layers; full layer-wise results are provided in \cref{appendix:more_experiments}. 
The overall training sparsity on CIFAR-10 is shown in \cref{fig:cifar10_sparsity}, while that on ImageNet-1K is shown in \cref{figure:productive_vit_average_training}.

Throughout training, especially at the beginning, $\tildeAS[\tildeKparam]^l$ increases but more slowly than the denominator. Consequently, the ratio in \cref{eq:multi_token_sparsity} decreases (as reflected by their difference in the log-scaled plot), leading to the emergence of activation sparsity.
Similarly, the activation norms decrease in general, suppressing large activations.

In the later stages of training on ImageNet-1K, activation sparsity deteriorates. 
We hypothesize that this happens because the denominators $c^l_{L_0}$ increase more slowly. 
By comparing \cref{eq:multi_token_sparsity} and \cref{eq:multi_token_sparsity_layernormed}, we hypothesize that the affine parameters in LayerNorm layers decrease during training, reducing MLP input norms and the denominator $c^l_{L_0}$. To counteract this, we propose lower-bounding the affine parameters in LayerNorm layers, as discussed in \cref{sec:algorithm}.

\begin{figure}[t]
    \centering \includegraphics[width=0.8\linewidth]{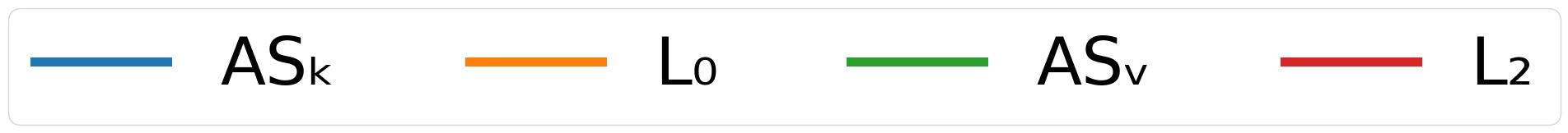}\\
    \subfloat[CIFAR-10\label{fig:initial_epochs_cifar10}]{
        \includegraphics[width=0.48\linewidth]{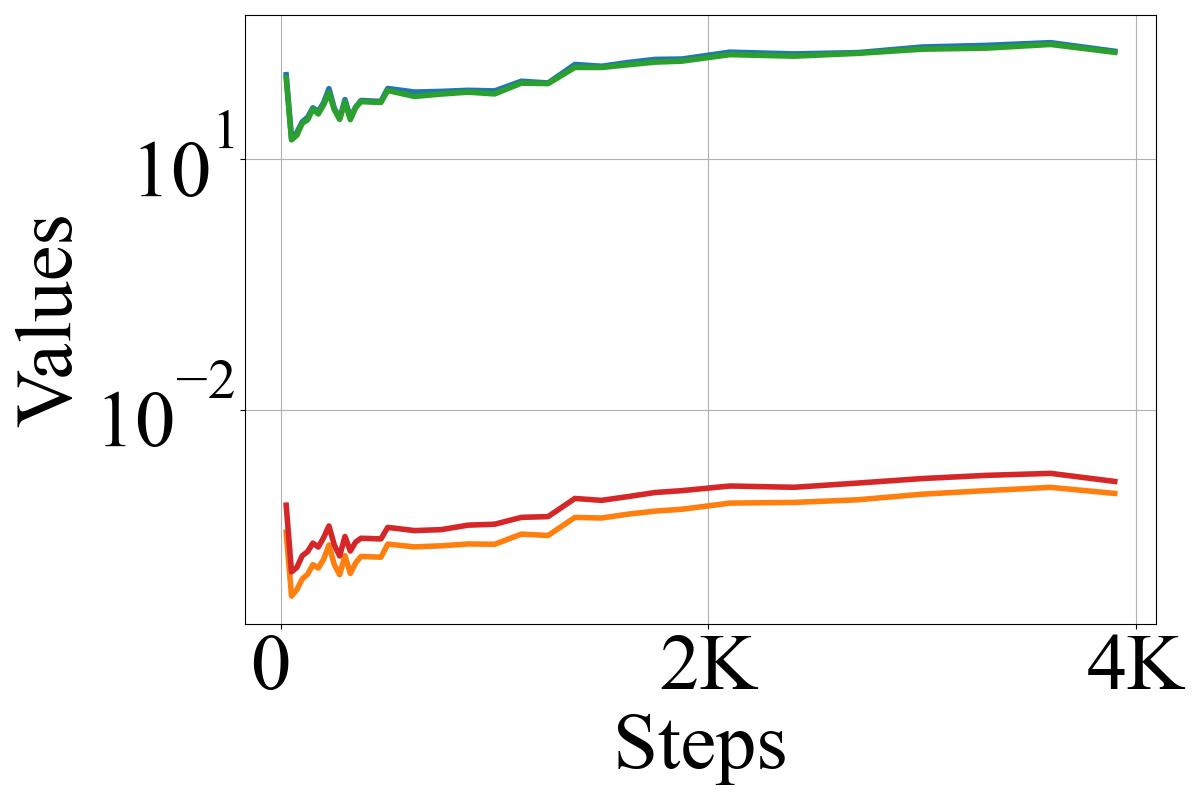}
    }
    \subfloat[ImageNet-1K\label{fig:initial_epochs_imagenet1k}]{
        \includegraphics[width=0.48\linewidth]{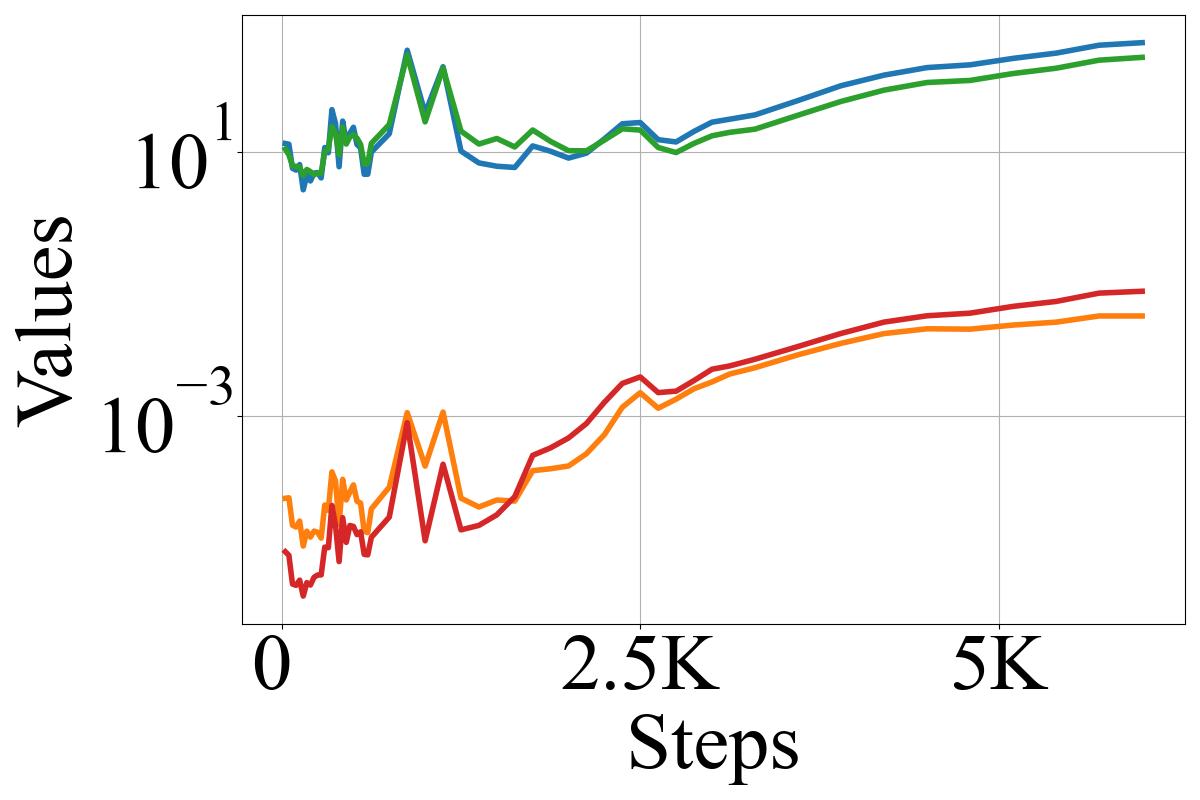}
    }
    \caption{Evolution of $\tildeAS[\tildeKparam]^l, c_{L_0}^l, \tildeAS[\tildeVparam]^l, c_{L_2}^l$ in the first 10 epochs.}\label{fig:initial_epochs}
\end{figure}
  \section{Discussions on Derivative Sparsity}\label{sec:stability}

\begin{figure}
    \centering \foreach \i in {,-}{
    \foreach \j in {-,}{
        \subfloat[$(\i1.6, \j1.6)$ \label{fig:shifted_weird\i0\j}]{
            \includegraphics[width=0.46\linewidth]{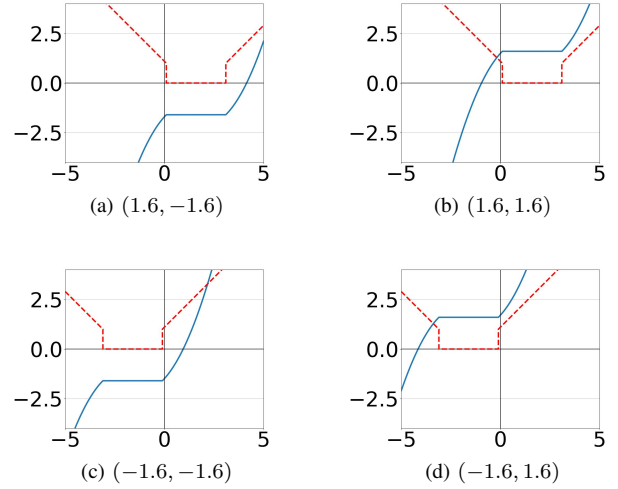}
        }
    }}
    \caption{The visualization of $\weird$ activation functions with different $(\Delta x, \Delta y)$ and their derivatives, indicated by blue lines and red dashed lines, respectively.}\label{fig:weird} 
\end{figure}

A byproduct of our analysis is derivative sparsity, which bridges augmented flatness and activation sparsity in \cref{theorem:single_token,theorem:multi_token} under $\relu$ activations.
Here, we argue that attention should be paid to the notion of derivative sparsity.

First, derivative sparsity has practical value.
In addition to pruning during forward propagation, neurons with zero derivatives backpropagate zero and can be safely pruned during backpropagation. This dual pruning reduces training cost, making derivative sparsity crucial for large-scale training.

Second, the connection between derivative sparsity and flatness is more stable. In \cref{theorem:single_token} and \cref{theorem:multi_token}, the ratio directly and stably equals derivative sparsity. In contrast, the connection to activation sparsity relies on the activation-derivative correlation specific to $\relu$, which may not hold for other activation functions. Therefore, derivative sparsity is more closely connected to (augmented) flatness and training dynamics than activation sparsity.

To verify this stability, we empirically demonstrate that derivative sparsity is a more stable notion than activation sparsity under some non-$\relu$ activation functions. 
To differentiate activation sparsity and derivative sparsity, we design special activation functions where they are no longer equal:

\begin{definition}\label{def:weird}
    The base $\weird$ function is defined by
    \minorrevision{
    \begin{align}
        \weird(x) =& \begin{cases}
            \frac{\operatorname{sgn}(x)}{2} ((|x| - \halfwidth + 1)^2 - 1)  & |x| \ge \halfwidth,\\
            0 & |x| < \halfwidth.
\end{cases}
    \end{align}
    }
    For $\Delta x, \Delta y \in \reals$, let $\weird_{\Delta x, \Delta y}(x) = \weird(x - \Delta x) + \Delta y$.
\end{definition}

\begin{figure}
    \centering \resetHeight{}\newcommand{\si}{}\newcommand{\sj}{}\newcommand{\validationalwidth}{0.45}
    
    \includegraphics[width=0.9\linewidth]{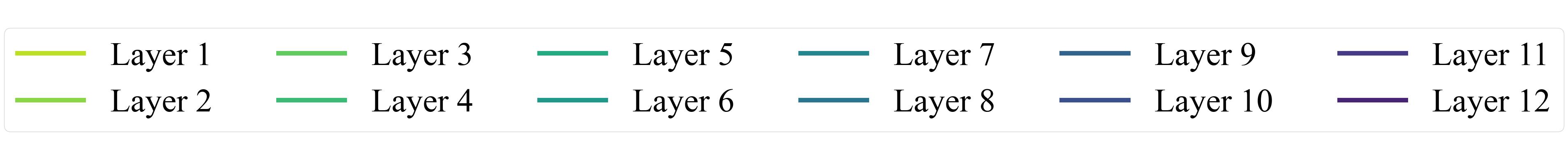}\foreach \i in {+,-}{
    \foreach \j in {-,+}{
        \subfloat[$(\j1.6, \i1.6)$ \label{fig:validational_full\i0\j}]{
            \centering
            \myincludegraphics[width=\validationalwidth\linewidth]{pic/results/dumps/validations/dx=\j1.6,dy=\i1.6/sparsity_training.jpg}
        }
    }\\}
    \caption{Derivative sparsity when training ViT-Base/16 with $\dbmlp$ and $\weird$ shifted by $\set{\pm 1.6}^2$ on CIFAR-10.}\label{fig:exp_stability}
\end{figure}

As illustrated in \cref{fig:weird}, this family of activation functions has plateaus that are shifted around. 
On these plateaus, derivatives are zero while activations are nonzero, allowing us to distinguish the two sparsities and compare their stability.
Moreover, existing work often interprets activation sparsity through decreasing magnitudes of activations or pre-activations \cite{observation,sharpness_aware}. Motivated by the theoretical stability of derivative sparsity, we ask whether decreases in derivative magnitudes are more stable. To compare these two decreases, we construct $\weird_{1.6, 1.6}$ so that stability of the former indicates pre-activation concentration around $x \approx -1.05$, where $\weird_{1.6, 1.6}(x) = 0$, while stability of the latter implies concentration on the plateau. One can then compare their stability by where pre-activations concentrate.
Similarly, vertical and horizontal shifts allow us to compare stability of decreases in derivative magnitudes against decreases or increases in activations, pre-activations, or their magnitudes.
These activation functions have discontinuous derivatives at plateau edges and increasing derivatives beyond the plateaus, designed to amplify these effects in the same spirit as $\jrelu$ defined later in \cref{sec:algo_jrelu}.

We conduct experiments to observe which form of sparsity is preserved. 
We replace GELU with $\weird_{\Delta x, \Delta y}$, where $(\Delta x, \Delta y) \in \set{-1.6, +1.6}^2$, in ViT-Base/16 and train the model on CIFAR-10 \cite{cifar10}. $\dbmlp$, defined later in \cref{sec:algorithm}, is also incorporated to strengthen the effects. The percentage of non-zero derivatives of MLP-block activation functions in training batches is shown in \cref{fig:exp_stability}.
In most MLP blocks, at least $70\%$ of activation derivatives become zero in deep MLP blocks during training, \ie{} derivative sparsity is achieved. At the same time, by construction of the activation function, in most layers at least $70\%$ of activations are non-zero, \ie{} activation sparsity is lost.
Moreover, as the plateau shifts, derivative sparsity still holds, demonstrating its stability.
Similar trends exist for derivative norms in \cref{fig:exp_stability_derivative_norm}, where squared derivatives decrease in most layers.  
These results indicate that some previous intermediate characterizations, such as ``decreasing activation'' \cite{observation,sharpness_aware}, are not accurate for general activation functions, because training prefers plateau occupancy and derivative sparsity over smaller activations. Other potential characterizations, such as decreasing or increasing activations, pre-activations, or their norms or absolute values, can also be invalidated by comparing the results of activation functions with differently shifted plateaus. Emphasizing derivative sparsity and derivative norms leads to a better understanding of the mechanisms behind activation sparsity.

\begin{figure}
    \centering \resetHeight{}\newcommand{\si}{}\newcommand{\sj}{}\newcommand{\validationalwidth}{0.45}\includegraphics[width=0.9\linewidth]{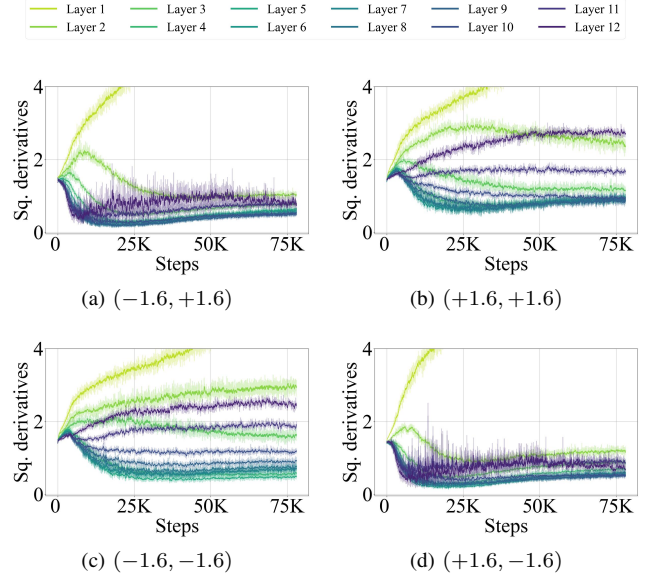}\foreach \i in {+,-}{
    \foreach \j in {-,+}{
        \subfloat[$(\j1.6, \i1.6)$ \label{fig:validational_full\i0\j_derivative_norm}]{
            \centering
            \myincludegraphics[width=\validationalwidth\linewidth]{pic/results/dumps/validations/dx=\j1.6,dy=\i1.6/norm_training.jpg}
        }
    }\\}
    \caption{Squared derivatives during ViT training with $\weird$ shifted by different $(\Delta x, \Delta y)$ and $\dbmlp$ on CIFAR-10.}\label{fig:exp_stability_derivative_norm}
\end{figure}

\section{Methods}\label{sec:algorithm}

\begin{figure*}
    \centering
    \includegraphics[width=\linewidth]{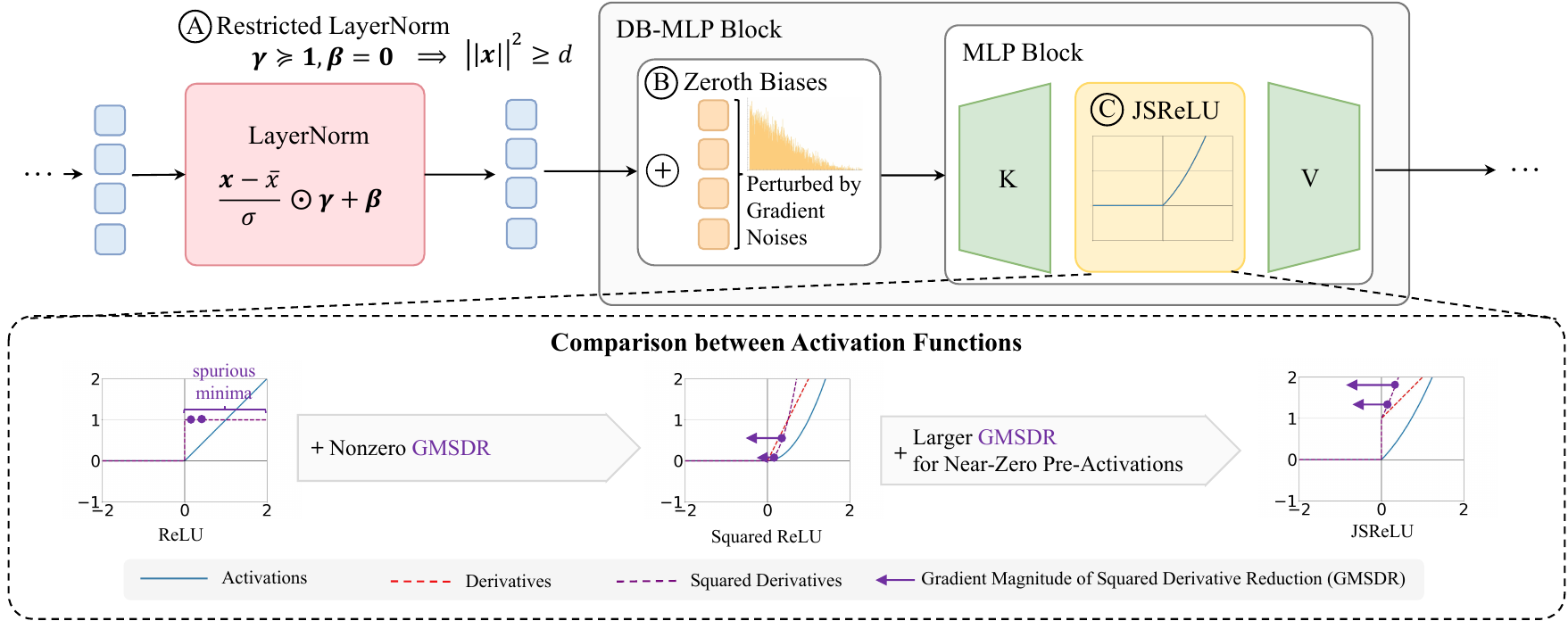}
    \caption{Overview of methods. In \cref{sec:restricted_layernorm}, we restrict affine parameters in LayerNorm layers so that MLP input tokens have large norms. In \cref{sec:zeroth_biases}, we add unshared biases, perturbed by gradient noises, to MLP input tokens, to magnify the gradient noises in MLP input tokens and improve the flatness of MLP blocks. In \cref{sec:algo_jrelu}, we propose a new activation function $\jrelu$ that helps the implicit optimization of reducing the mean squared activation derivatives.}
    \label{fig:methods} \label{fig:activations}
\end{figure*}

The theoretical and empirical results reveal a connection between activation sparsity and augmented flatness. In addition, properties of normalization layers and activation functions play critical roles in this connection.
Therefore, we propose three methods to improve sparsity via flatness, activation functions, and normalization layers, for both pretraining and finetuning.

\subsection{Restricted LayerNorm for Larger Denominator}\label{sec:restricted_layernorm}

\cref{theorem:single_token,theorem:multi_token} indicate that activation sparsity is related to MLP input norms $\norm{\vec{x}_i}_2^2$ in the ratio denominator.
Since MLP inputs are LayerNorm-ed in Transformers, the affine parameters in the LayerNorm layers affect the MLP input norms and the activation sparsity.
If these parameters remain active, they can reduce the MLP input norms, thereby decreasing the denominator and increasing the ratio, and harming activation sparsity.
To prevent this MLP-input norm reduction, we lower-bound the affine factors so that entries are at least $1$. This approach also allows training to adjust and potentially increase them, leading to larger denominators, smaller ratios, and improved sparsity. During pretraining from scratch, we execute \cref{algo:layernorm} after every parameter update. 

During finetuning, LayerNorm affine factors may not initially be lower-bounded by $1$.
To address this discrepancy, we freeze LayerNorm biases and gradually increase the lower bound on absolute affine factors during early finetuning until they reach $1$, as detailed in \cref{algo:layernorm_finetuning}.

\NewDocumentCommand{\listofLayerNorms}{}{\mathcal{L}}
\NewDocumentCommand{\listofZerothBiases}{}{\mathcal{B}}
\begin{algorithm}
    \caption{Restricted LayerNorm (RLN) for Pretraining}\label{algo:layernorm}
    \begin{algorithmic}[1]
        \STATE {\textsc{RLN-Pretraining}}$(\listofLayerNorms)$
        \STATE \hspace{0.5cm} \textbf{For} $L \in \listofLayerNorms$ \textbf{ do}
        \STATE \hspace{1.0cm} $L.\text{weight} \gets \text{clamp}(L.\text{weight},\text{min}=1.0)$
        \STATE \hspace{0.5cm} \textbf{End For}
    \end{algorithmic}
\end{algorithm}

\begin{algorithm}
    \caption{Restricted LayerNorm (RLN) for Finetuning}\label{algo:layernorm_finetuning}
    \begin{algorithmic}[1]
        \STATE {\textsc{RLN-Finetuning}}$(\listofLayerNorms, T_{\textnormal{warmup}}, t)$
        \STATE \hspace{0.5cm} \textbf{For} $L \in \listofLayerNorms$ \textbf{ do}
        \STATE \hspace{1.0cm} $\vec{\gamma} \gets L.\text{weight}$
        \STATE \hspace{1.0cm} $\vec{s} \gets \text{sign}(\text{sign}(\vec{\gamma}) + 0.1)$
        \STATE \hspace{1.0cm} $p \gets \min(t / T_{\textnormal{warmup}}, 1.0)$
        \STATE \hspace{1.0cm} $\vec{u} \gets \text{clamp}(\text{abs}(\vec{\gamma}),\text{min}= p)$
        \STATE \hspace{1.0cm} $L.\text{weight} \gets \vec{s} \hadamard \vec{u}$
        \STATE \hspace{0.5cm} \textbf{End For}
    \end{algorithmic}
\end{algorithm}

\subsection{Zeroth Biases for Larger Stochastic Gradient Noises}\label{sec:zeroth_biases}

It is widely believed that the inductive bias toward flatness stems from gradient noise that makes training escape sharp minima \cite{escape,alpha_stable}.
Therefore, we attempt to improve flatness by increasing these noises.
For gradient noise generated at MLP-block parameters themselves, we currently find limited room for further improvement without injecting additional noise. Nevertheless, \cite{petzka_relative_2021} indicates that MLP-block robustness to inputs is also related to flatness, suggesting potential gains from augmenting noise on MLP input tokens. Given the advantages of gradient noise in magnitude \cite{alpha_stable} and alignment with sharp local directions \cite{alignment}, we choose to strengthen gradient noise generated at shallower-layer parameters and forward-propagated to MLP input tokens, instead of manually injecting noise.
To this end, we identify two insufficiencies in the gradient noises due to the following factors: 1) parameter sharing and 2) nonlinearities.
In modern neural networks with extensive parameter sharing, MLP input tokens are produced by the same group of parameters, with minor exceptions such as positional embeddings. As a result, gradient noise on shared parameters introduces correlation between noise on different input tokens. This correlation harms noise diversity on MLP input tokens.
Moreover, when forward-propagated through nonlinearities (\eg{} $\relu$, self-attention), noise is suppressed by saturation, \eg{} activation plateaus and softmax saturation in attention. Perturbations to inputs of saturated nonlinearities induce little output response, diminishing noise. Therefore, saturation weakens noise magnitudes on MLP input tokens.
Based on these observations, we introduce unshared immediate parameters before MLP blocks to reduce correlation and amplify gradient noise, thereby encouraging training to explore flatter minima \wrt{} $\KVparam$.
Specifically, before feeding an input $\mat{X}^l$ to an MLP block, we add to it a trainable bias \emph{matrix} $\mat{B}^l$ of the \emph{same shape} as $\mat{X}^l$, called Zeroth Bias, forming a Doubly-Biased MLP ($\dbmlp$):
\begin{align}
    \dbmlp^l_{\allparam}(\mat{X}^l) \defeq \mlp^l_{\allparam}(\mat{X}^l + \mat{B}^l).
\end{align}
Since $\vec{b}^l_i$ is added only to $\mat{x}^l_i$ (without sharing) and directly (without nonlinearity), we expect gradient noise on $\vec{b}^l_i$ to be more diverse and stronger, without correlation and saturation effects.

\begin{algorithm}
    \caption{Restricted Zeroth Biases}\label{algo:zeroth_bias}
    \begin{algorithmic}[1]
        \STATE {\textsc{RestrictedZerothBiases}}$(\listofLayerNorms, \listofZerothBiases, c)$
        \STATE \hspace{0.5cm} \textbf{For} $l = 0 \dots \size{\listofLayerNorms} - 1$ \textbf{ do}
        \STATE \hspace{1.0cm} $L \gets \listofLayerNorms[l]$ \quad \gray{// $l$-th LayerNorm}
        \STATE \hspace{1.0cm} $\mat{B} \gets \listofZerothBiases[l]$ \quad \gray{// $l$-th Zeroth Bias right after $L$}
        \STATE \hspace{1.0cm} $\vec{\gamma} \gets \text{abs}(L.\text{weight})$
        \STATE \hspace{1.0cm} $\mat{S} \gets c \cdot \begin{bmatrix}
            \vec{\gamma} & \vec{\gamma} & \cdots & \vec{\gamma}
        \end{bmatrix}$
        \STATE \hspace{1.0cm} $\mat{B} \gets \text{clamp}(\mat{B}, \text{min}=-\mat{S}, \text{max}=\mat{S})$
        \STATE \hspace{0.5cm} \textbf{EndFor}
    \end{algorithmic}
\end{algorithm}

For both pretraining and finetuning, $\mat{B}^l$ is initialized to zero to seamlessly adapt pretrained weights.
However, since $\mat{B}^l$ is added after LayerNorm, it may interfere with restricted LayerNorm by canceling large entries of LayerNorm-ed $\mat{X}^l$ and reducing $\norm{\vec{x}_i^l}_2^2$. To alleviate this conflict, we restrict the magnitude of entries in $\mat{B}^l$ to allow gradient noise while preventing cancellation. Specifically, if $\vec{\gamma}^l$ is the affine weight of the LayerNorm right before $\mat{B}^l$, then we restrict $\mat{B}^l$ so that
$
|b^l_{i, j}| \le c \cdot \abs{\gamma_j^l}
$
with $c \in (0, 1)$. In this way, we always maintain $\norm{\vec{x}^l_i + \vec{b}^l_i}_2^2 \ge (1-c)^2 \norm{\vec{x}^l_i}_2^2 \ge (1-c)^2 d$. Since we only lower-bound $\abs{\gamma_j}$ and allow it to exceed $1$, this strategy also permits larger gradient-noise magnitudes when $\abs{\gamma_j}$ grows. To implement this strategy, we execute \cref{algo:zeroth_bias} after every parameter update.

\subsection{$\jrelu$ Activation for Better Derivative Sparsity Search}\label{sec:algo_jrelu}

The coincidence of activation and derivative sparsity under $\relu$ is beneficial, and we aim to preserve this property.
However, $\relu$'s derivative remains constant after being activated. 
Its piecewise constant derivative hinders the implicit decrease of derivative norms empirically shown in \cref{fig:exp_stability_derivative_norm} of \cref{sec:stability}, which benefits derivative sparsity.
This is because small parameter perturbations do not change activation derivatives or their norms, creating spurious minima in implicit flatness optimization. 
To provide local guidance and remove these spurious minima, we make the derivative monotonically increasing so that as activations move toward zero, activation derivatives decrease, immediately benefiting derivative norms.
The simplest and most direct modification may be $\relu^2(x)$ \cite{primer}, with monotonically increasing derivative $\max\set{2x, 0}$. However, our tentative experiments show that although large activations become rarer, many small but \emph{non-zero} activations remain, and strictly zero activations become fewer.
We hypothesize this is because when the pre-activation $x > 0$ is close to $0$, the derivative $2x$ is already very small, providing little drive for further reduction.
To address this, we reintroduce the derivative discontinuity of $\relu$ at $0$ using $\jrelu$, as in \cref{def:jrelu}, so that derivatives remain large near $0$ and training is pushed to reduce pre-activations to $0$ for augmented flatness. $\relu$, $\relu^2$, and $\jrelu$ are visually compared in \cref{fig:activations}.
\begin{definition}\label{def:jrelu}
    $\jrelu: \reals \to \reals$ is defined by
    \begin{align}
        \jrelu(x) \defeq \begin{cases}
            \frac{1}{2} \left( (x + 1)^2 - 1 \right) & x \ge 0\\
            0   &   x < 0.
        \end{cases} \label{eq:jsrelu}
    \end{align}
\end{definition}

We use $\jrelu$ to replace $\relu$ from the start of pretraining. To adapt $\relu$-pretrained models during finetuning, we linearly mix $\relu$ with $\jrelu$ and linearly increase the weight of $\jrelu$ during the first few thousand finetuning steps. After this warmup, $\relu$ is fully removed.

Similar theoretical results can be derived for $\jrelu$:
\begin{align}
    \ex{\norm{\mat{A}^l}_0} =& \frac{\tildeAS[\tildeKparam^l]}{d \cdot \ex{\left(\derivatives{\ell}{a_{i, j}} \cdot d_{i, j}\right)^2 \mid a_{i, j} > 0}} \label{eq:jsrelu_larger_denominator}. 
\end{align}
The detailed formal results can be found in \cref{appendix:jrelu_results}.
These results highlight an additional benefit of $\jrelu$: multiplying gradient magnitudes by $d_{i, j} \ge 1$, which further increases the denominator and decreases the ratio. 
\section{Experiments}\label{sec:method_experiments}

\subsection{Evaluation}

Methods proposed in \cref{sec:algorithm} are evaluated on natural image classification and natural language generation tasks.
Following \cite{observation}, we assess and compare sparsity using the percentage of non-zero activations.
To better reflect both training and inference costs, we distinguish training sparsity from testing sparsity during pretraining. Since training cost is much lower during finetuning, we focus comparisons primarily on testing sparsity of finetuning.
For each layer's MLP block, training sparsity is estimated from training batches every few steps, while testing sparsity is computed over the entire test set at each evaluation.
The overall training sparsity is obtained by averaging over steps 
and layers, and the overall testing sparsity by averaging over layers at the last evaluation.
All models are pretrained only once due to their high computational costs.
As baselines, we use vanilla ($\relu$) vision or language Transformers.
Throughout this section, we use warm colors for modified Transformers and cold colors for baselines in plots. 

\subsection{Sparsity-Aware Pretraining}\label{sec:p_experiments}

\minorrevision{
For vision tasks, we train ViT \cite{vit} or SwinTransformer \cite{swin} of various scales on ImageNet-1K \cite{imagenet1k}, Places365 \cite{places365}, and a LAION-400M \cite{laion400m} subset. 
Among them, LAION-400M provides language captions for images instead of explicit labels. Thus, we train the ViT using CLIP loss \cite{clip} with text embeddings from the pretrained CLIP-Base text encoder. We sampled a 1/10 subset of LAION-400M, containing approximately 43M images. See \cref{appendix:details} for full details.
For language tasks, we train T5-Base \cite{t5} on the C4 dataset \cite{t5}.
}

To construct the sparsified model, we replace every activation function in vision and language Transformers with $\jrelu$. Biases in all LayerNorm layers are fixed at $0$, while affine factors are lower-bounded by $1$ by running \cref{algo:layernorm} after every parameter update. Zeroth biases are inserted before every MLP block and restricted by running \cref{algo:zeroth_bias} with $c=0.1$, ensuring that the input to MLP blocks has a large norm $\norm{\vec{x}^l_i + \vec{b}^l_i}_2^2 
\ge 0.81 d$.
As in \cite{observation}, T5 is only pretrained for $10^5$ steps to save computation.

\begin{figure*}
    \centering \includegraphics[width=0.24\linewidth]{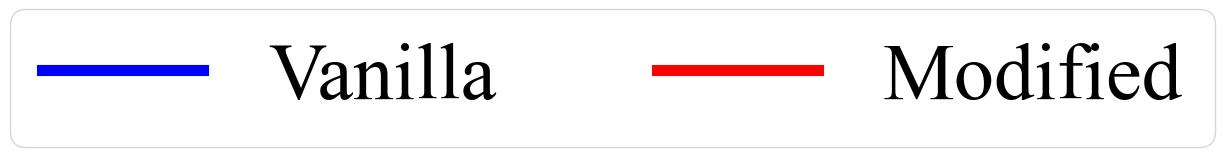}\\
    \subfloat[Training, stepwise. \label{figure:productive_vit_average_training}]{
        \includegraphics[width=0.24\linewidth]{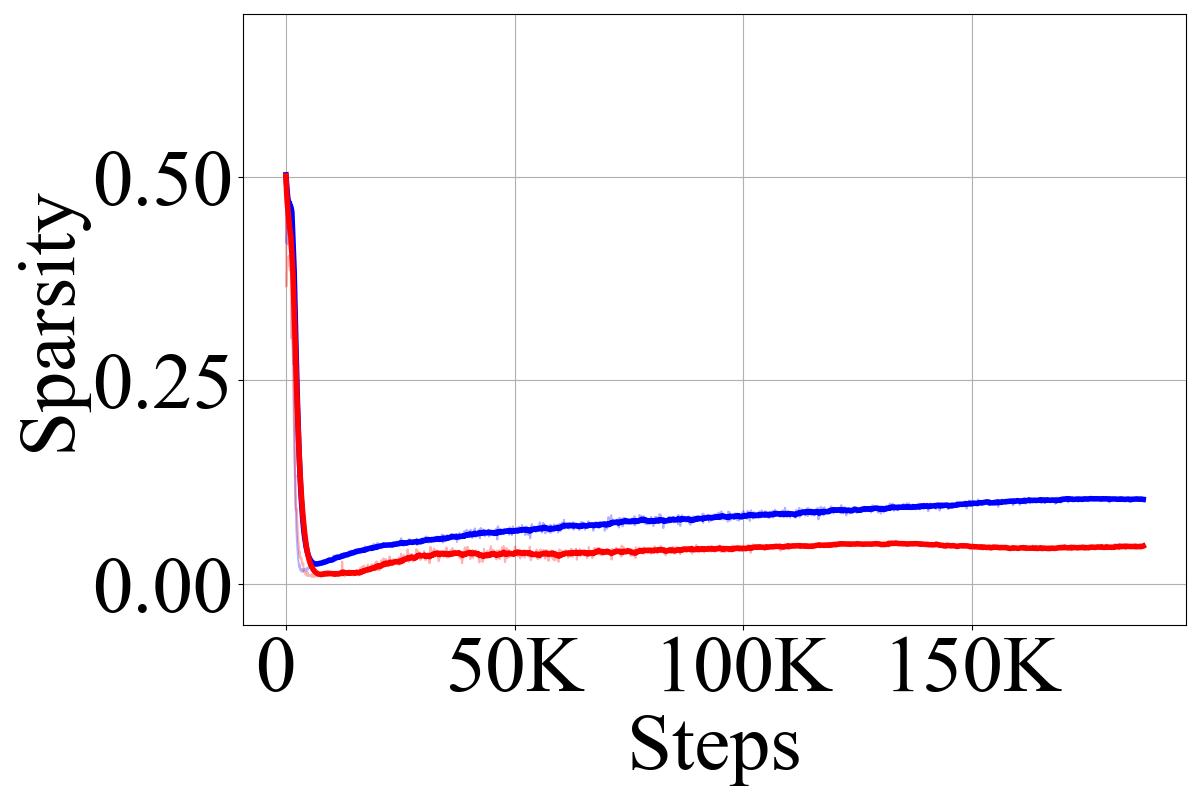}
}
    \subfloat[Training, layerwise. \label{figure:productive_vit_end_training}]{
        \includegraphics[width=0.24\linewidth]{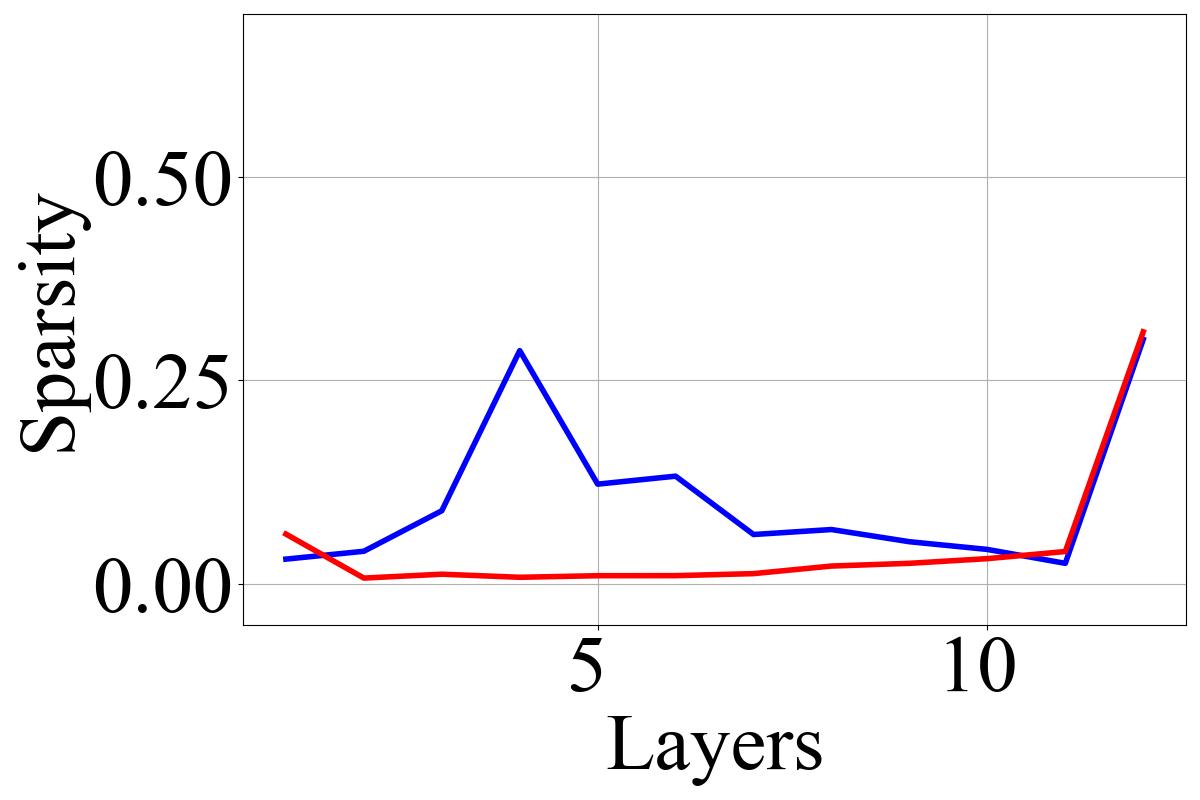}
}
    \subfloat[Testing, stepwise.\label{figure:productive_vit_average_testing}]{
        \includegraphics[width=0.24\linewidth]{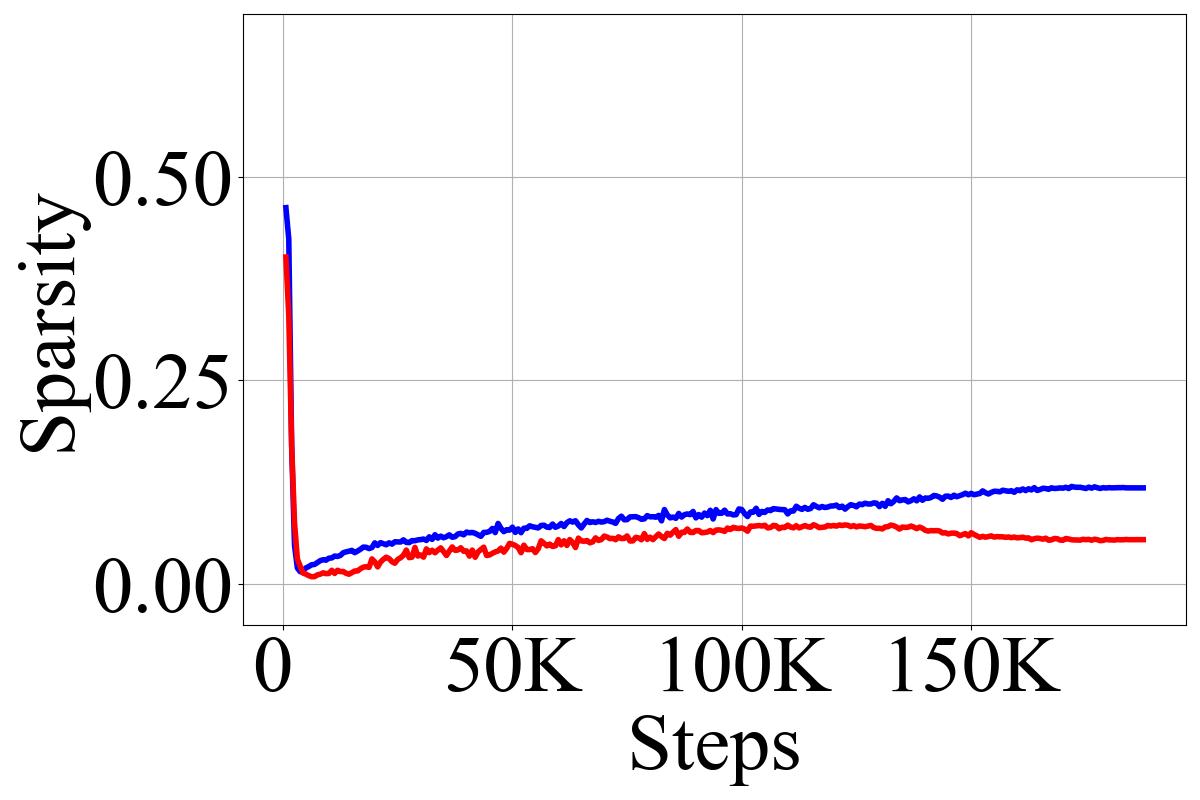}
}
    \subfloat[Testing, layerwise. \label{figure:productive_vit_end_testing}]{
        \includegraphics[width=0.24\linewidth]{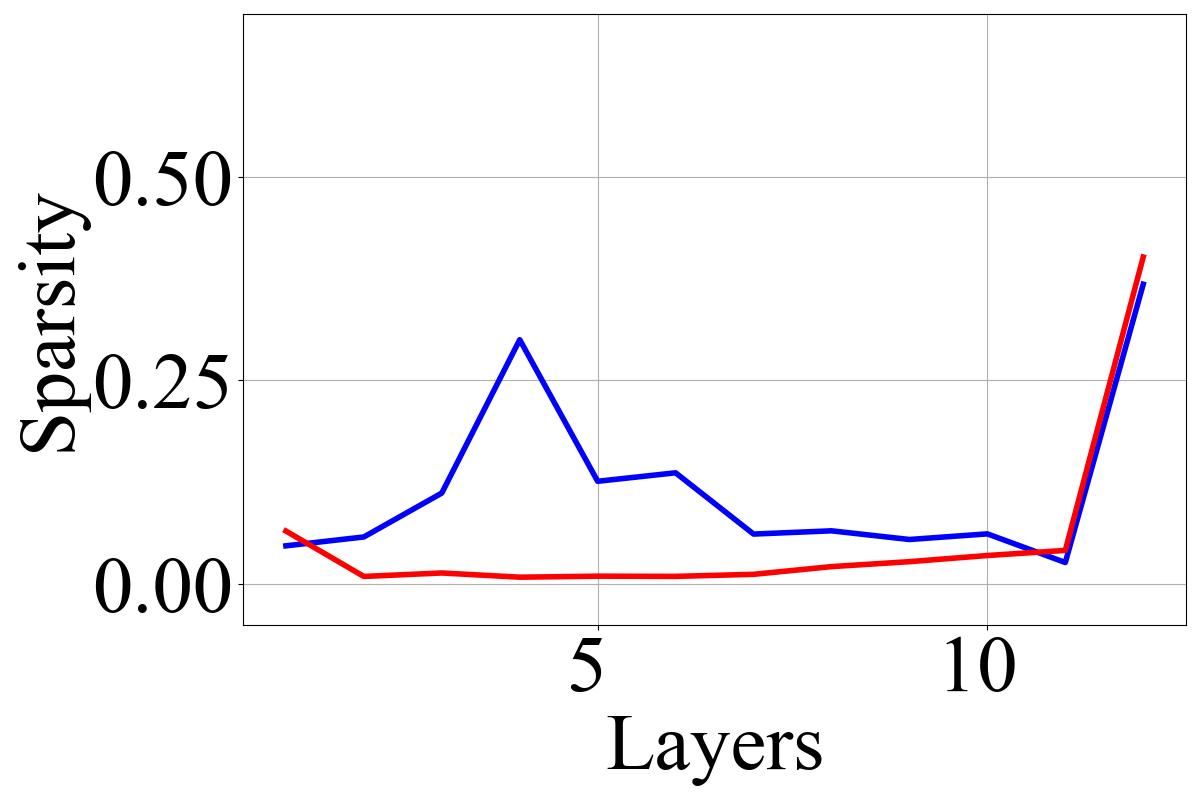}
}
    \caption{Training and testing sparsity during training of ViT-Base on ImageNet-1K.
        \cref{figure:productive_vit_average_training,figure:productive_vit_end_training} show \emph{training} sparsity while \cref{figure:productive_vit_average_testing,figure:productive_vit_end_testing} show \emph{testing} sparsity. 
        \cref{figure:productive_vit_average_training,figure:productive_vit_average_testing} show the evolution of sparsity along the training steps, while \cref{figure:productive_vit_end_training,figure:productive_vit_end_testing} compare sparsity in a layerwise manner.
        More detailed results can be found in \cref{figure:productive_vit_full} of \cref{appendix:more_experiments}.
    }\label{figure:productive_vit}
\end{figure*}

\begin{figure}
    \centering \subfloat[\minorrevision{Averaged testing sparsity of SwinTransformer-Base on ImageNet-1K.}\label{figure:productive_swin_average_testing}]{
        \minorrevisionimage{\includegraphics[width=0.48\linewidth]{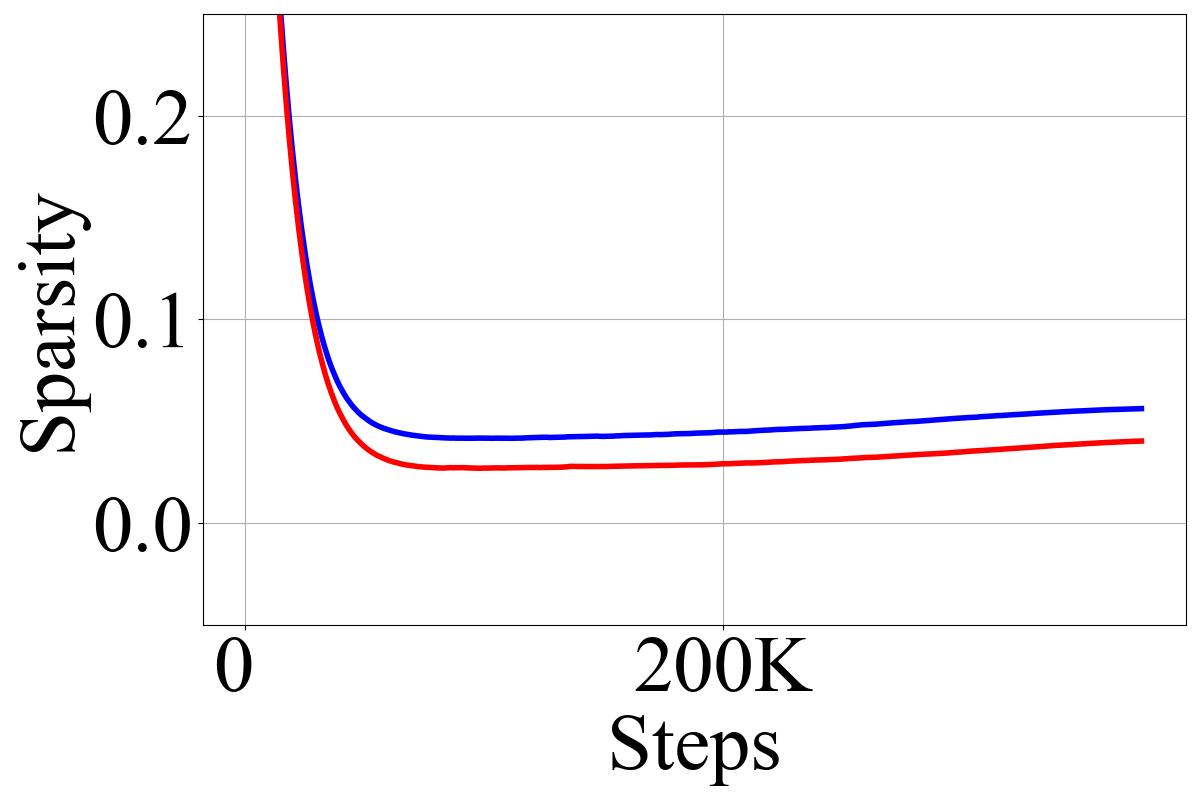}}
        \minorrevisionimage{\includegraphics[width=0.48\linewidth]{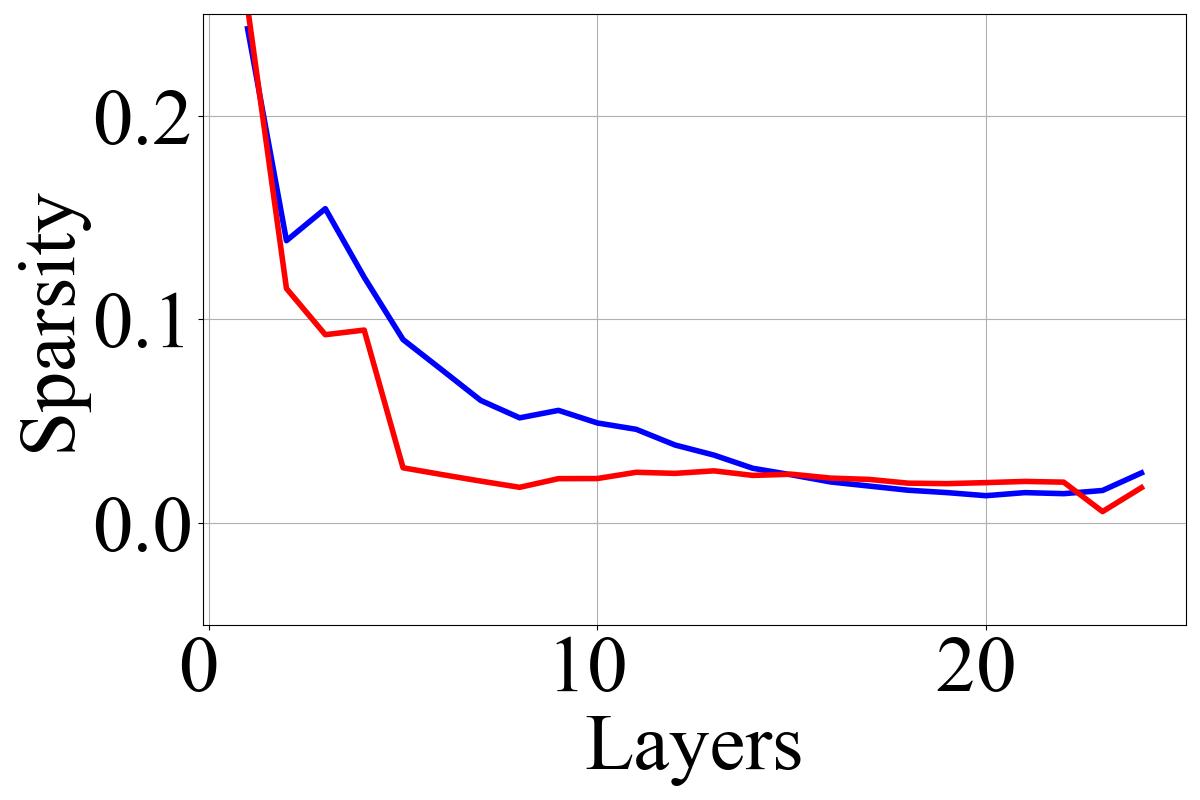}}
    }\\
    \subfloat[\minorrevision{Averaged testing sparsity of ViT-Large on Places365.}\label{figure:productive_vit_large_average_testing}]{
        \minorrevisionimage{\includegraphics[width=0.48\linewidth]{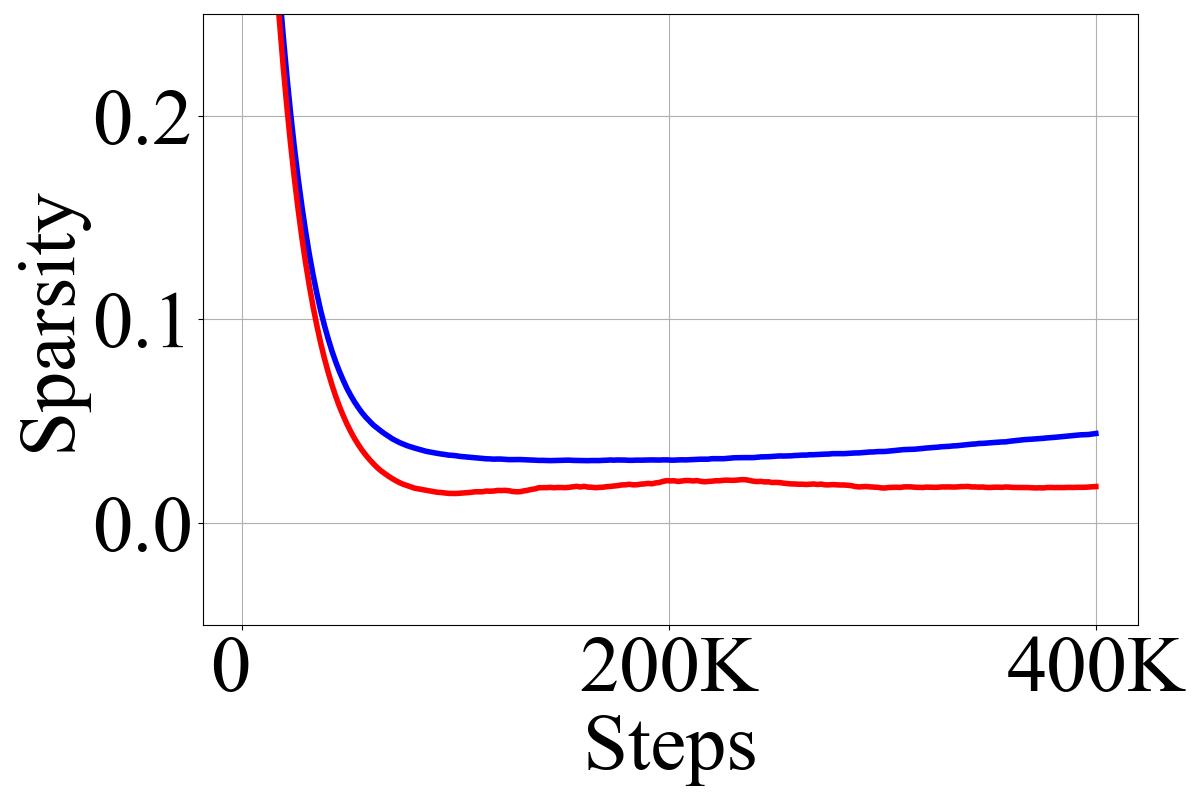}}
        \minorrevisionimage{\includegraphics[width=0.48\linewidth]{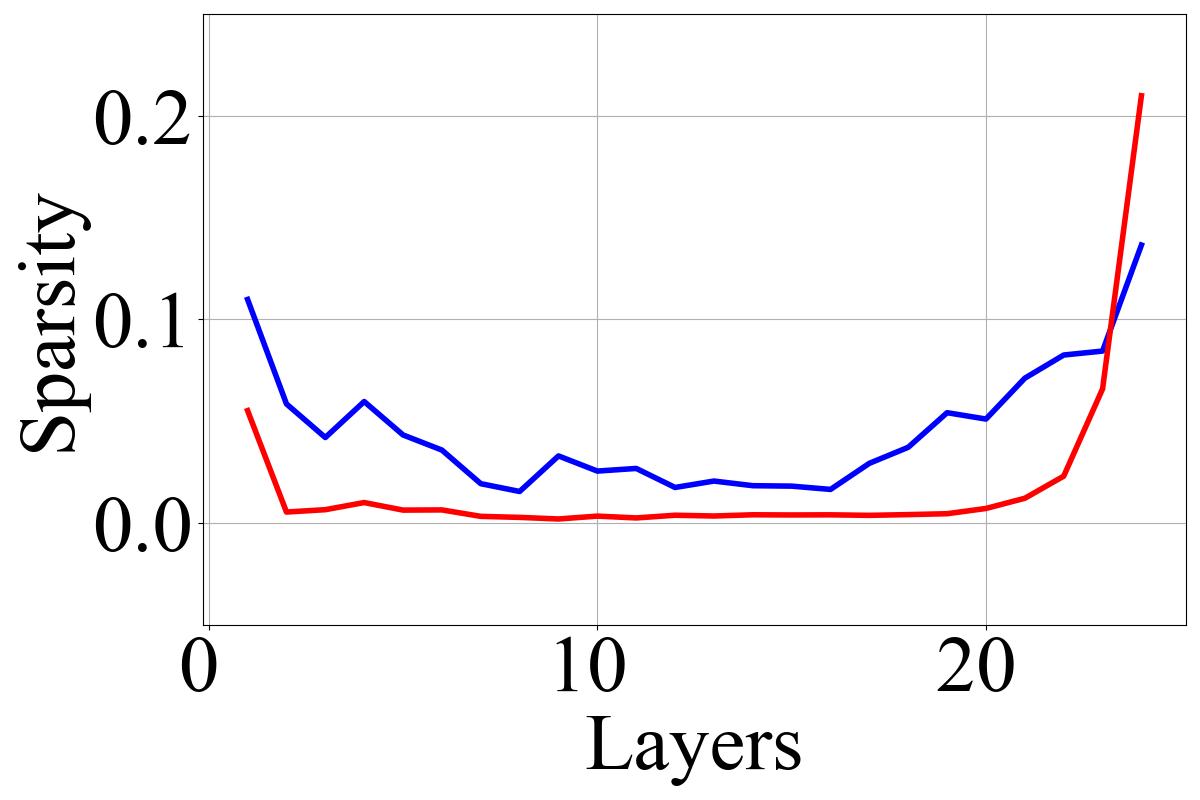}}
    }\\
    \subfloat[\minorrevision{Averaged testing sparsity of ViT-Base on a 1/10 subset of LAION-400M.}\label{figure:productive_laion_average_testing}]{
        \minorrevisionimage{\includegraphics[width=0.48\linewidth]{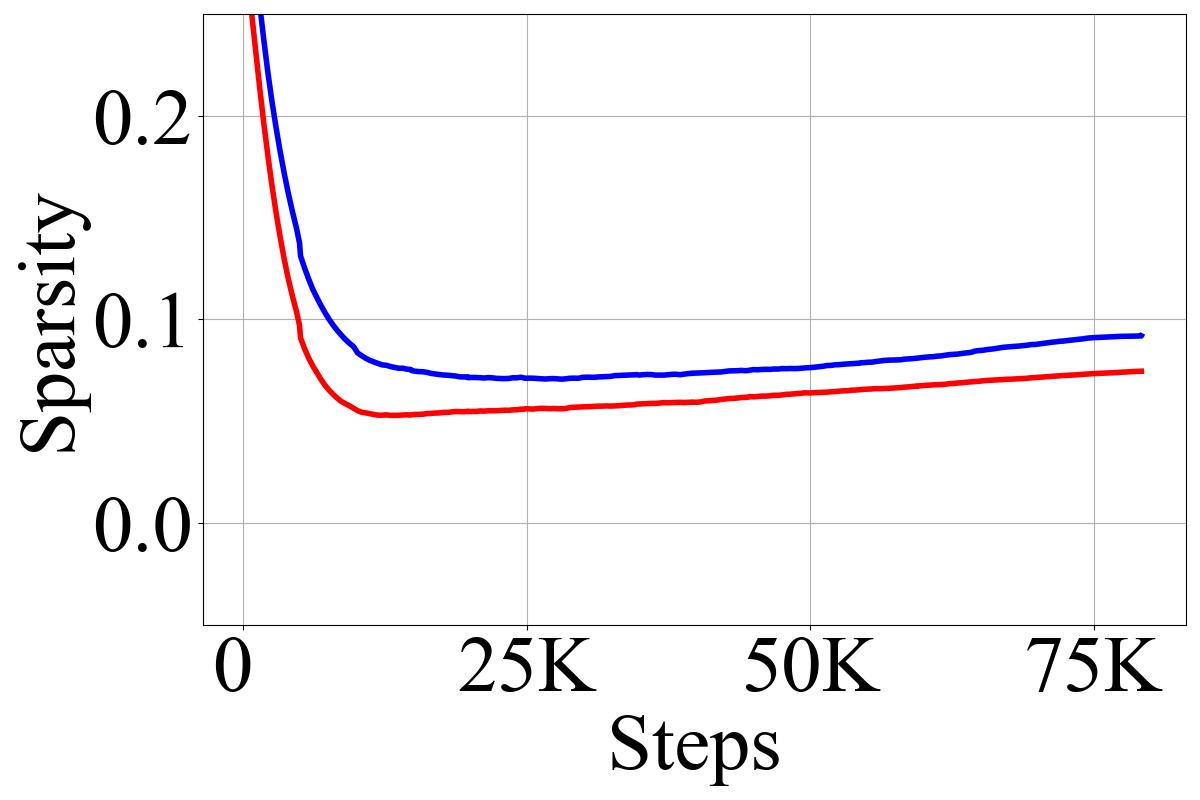}}
        \minorrevisionimage{\includegraphics[width=0.48\linewidth]{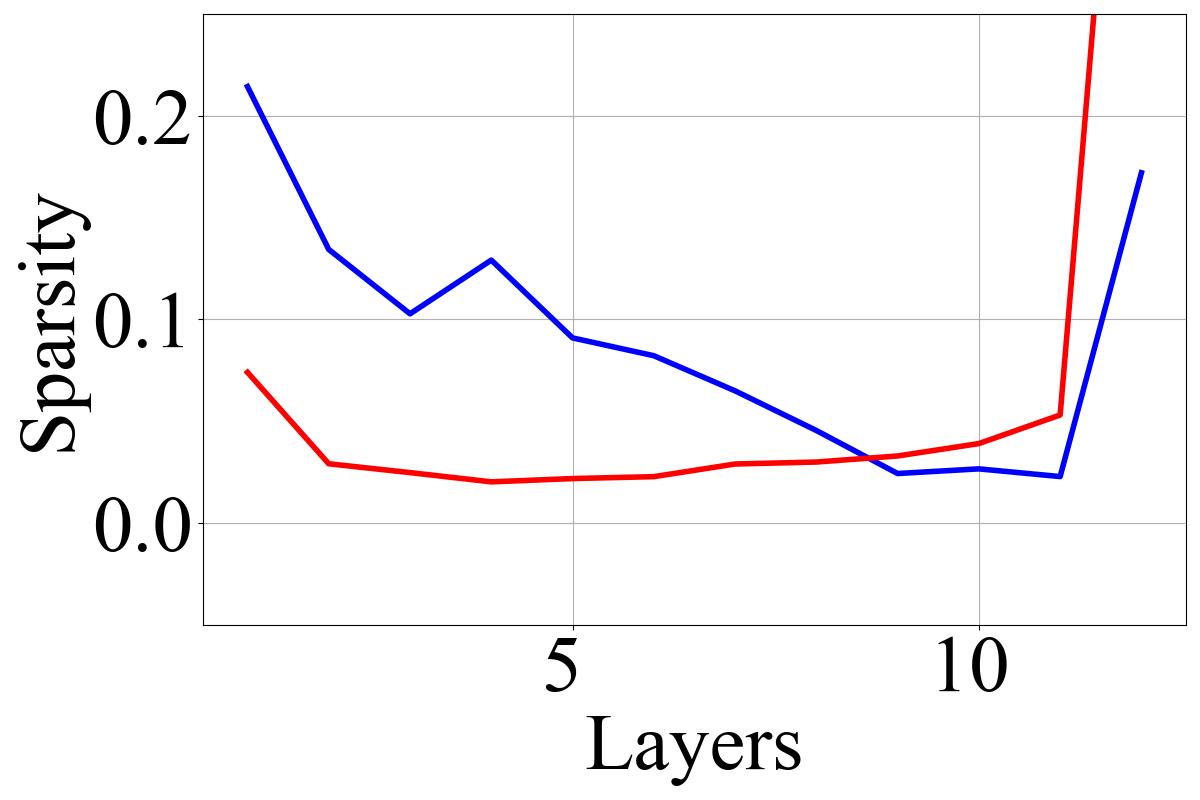}}
    }
    \caption{\minorrevision{Testing sparsity of SwinTransformer-Base on ImageNet-1K, ViT-Large on Places365, and ViT-Base on a 1/10 subset of LAION-400M. Testing sparsity is averaged across layers or steps in the same manner as in \cref{figure:productive_vit}.}
    }\label{figure:other_vision_tasks}
\end{figure}

\begin{table}
    \centering
    \caption{\minorrevision{The overall training and testing sparsities of models pretrained on vision tasks. \whatisbold{}}}\label{table:productive_vit}
    \minorrevision{
    \begin{tabular}{c|lrrrrrr}
        \toprule
                &                &    \multicolumn{2}{c}{Sparsity} & \multicolumn{2}{c}{Test acc. \higherbetter{} or loss \lowerbetter{}}\\
        \cmidrule(lr){3-4}
        \cmidrule(lr){5-6}
         Setting        &   Method  & Train \lowerbetter{}   &   Test \lowerbetter{}          &   Acc@1        &   Acc@5\\
        \midrule    
        ViT-B & Vanilla                 &   $0.104$                 &   $0.087$                         &   $\mathbf{77.35\%}$          &   $\mathbf{93.50\%}$\\
        I.Net & Modified                &   $\mathbf{0.046}$        &   $\mathbf{0.055}$                &   $76.77\%$                   &   $93.30\%$\\
        \midrule    
        SwT-B & Vanilla                 &   $0.056$                 &   $0.057$                         &   $79.75\%$          &   $94.73\%$\\
        I.Net & Modified                &   $\mathbf{0.040}$        &   $\mathbf{0.041}$                &   $\mathbf{79.82\%}$                   &   $\mathbf{94.90\%}$\\
        \midrule    
        ViT-L & Vanilla                 &   $0.049$                 &   $0.039$                         &   $\mathbf{54.81\%}$          &   $\mathbf{84.97\%}$\\
        P.365 & Modified                &   $\mathbf{0.020}$        &   $\mathbf{0.018}$                &   $54.60\%$                   &   $84.87\%$\\
        \midrule    
        ViT-B & Vanilla                 &   $0.090$                 &   $0.092$                         &   \multicolumn{2}{c}{$\mathbf{0.984}$}\\
        L.40M & Modified                &   $\mathbf{0.073}$        &   $\mathbf{0.075}$                &   \multicolumn{2}{c}{1.067}\\
        \bottomrule
    \end{tabular}
    }
\end{table}

\begin{table}
    \centering
    \caption{The overall training and testing sparsities of T5-Base pretrained on a subset of C4. \whatisbold{}}\label{table:productive_t5}
    \begin{tabular}{lrrrrrr}
        \toprule
                                Method &   Train sparsity \lowerbetter{}     &   Test sparsity \lowerbetter{}      &   Test loss \lowerbetter{}\\
        \midrule    
        Vanilla                 &   $0.302$                     &   $0.299$                     &   $4.88$\\
        Modified                &   $\mathbf{0.153}$            &   $\mathbf{0.180}$            &   $\mathbf{4.78}$\\
        \bottomrule
    \end{tabular}
\end{table}

\begin{figure*}
    \centering \includegraphics[width=0.24\linewidth]{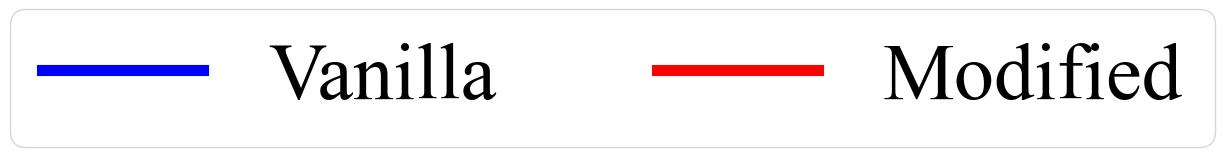}\\
    \resetHeight{}\subfloat[Training, stepwise. \label{figure:productive_t5_average_training}]{
        \includegraphics[width=0.24\linewidth]{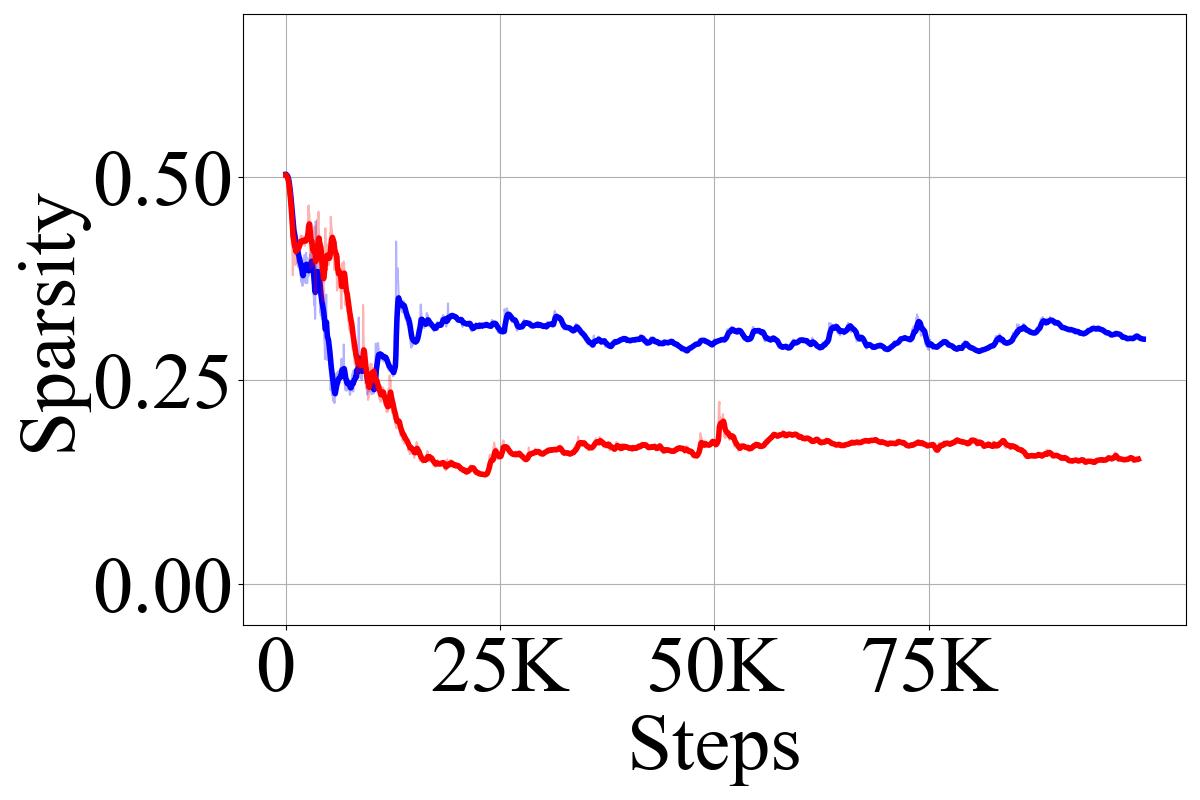}
}
    \subfloat[Training, layerwise. \label{figure:productive_t5_end_training}]{
        \includegraphics[width=0.24\linewidth]{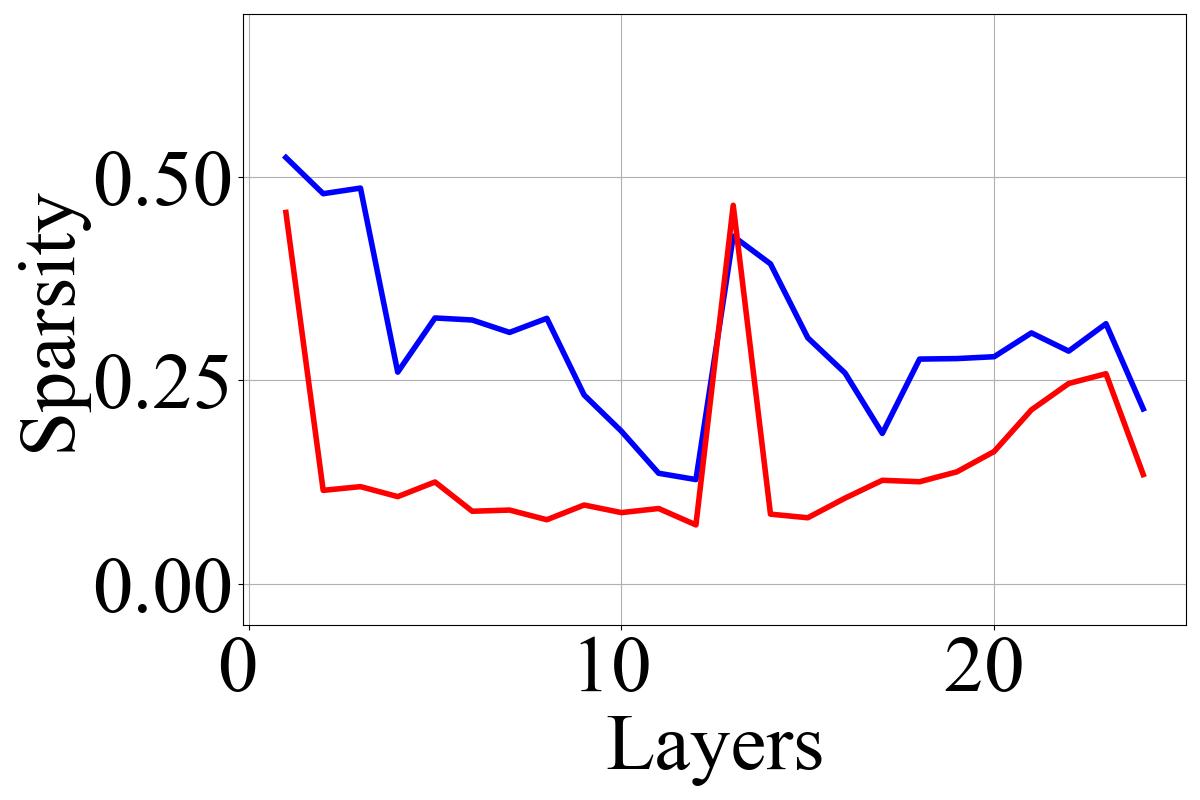}
}
    \subfloat[Testing, stepwise.\label{figure:productive_t5_average_testing}]{
        \includegraphics[width=0.24\linewidth]{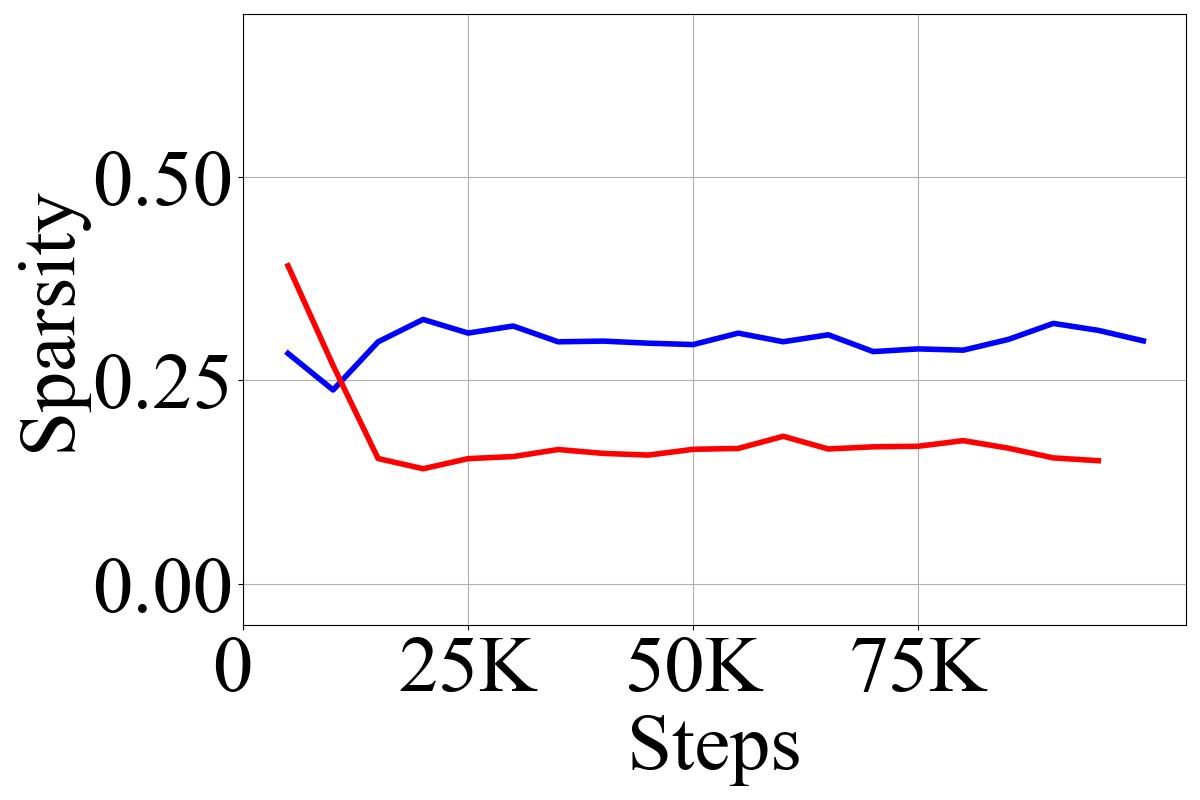}
}
    \subfloat[Testing, layerwise. \label{figure:productive_t5_end_testing}]{
        \includegraphics[width=0.24\linewidth]{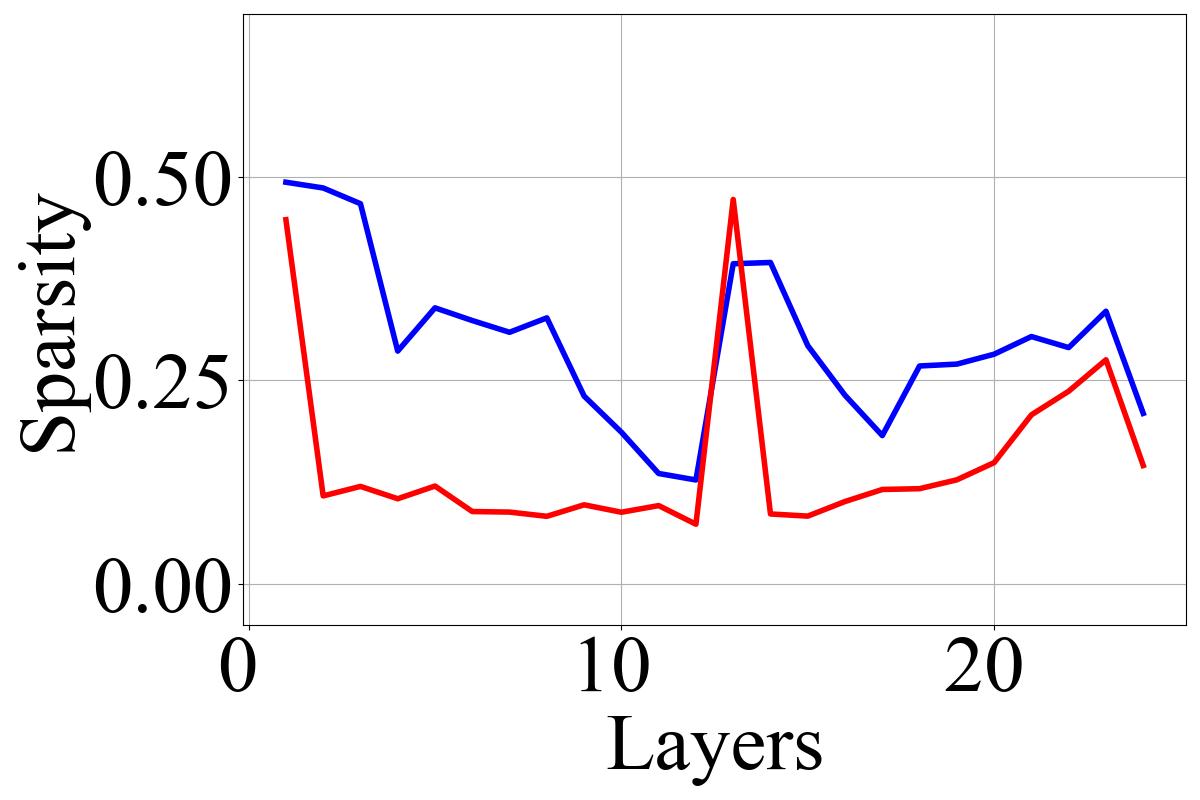}
}
    \caption{Training and testing sparsity during training of T5-Base on C4. 
        \cref{figure:productive_t5_average_training,figure:productive_t5_end_training} show \emph{training} sparsity while \cref{figure:productive_t5_average_testing,figure:productive_t5_end_testing} show \emph{testing} sparsity. 
        \cref{figure:productive_t5_average_training,figure:productive_t5_average_testing} show the evolution of sparsity along the training steps, while \cref{figure:productive_t5_end_training,figure:productive_t5_end_testing} compare sparsity improvements in a layerwise manner.
        Layers in the encoder are indexed by 1-12 while those in the decoder are indexed by 13-25.
        More detailed results can be found in \cref{figure:productive_t5_full} of \cref{appendix:more_experiments}.
    }\label{figure:productive_t5}
\end{figure*} 
Training and testing sparsity in vision and language tasks are summarized in \cref{table:productive_vit} and \cref{table:productive_t5}. \minorrevision{The evolution of sparsity during training is shown in \cref{figure:productive_vit,figure:other_vision_tasks,figure:productive_t5}. More detailed results can be found in \cref{appendix:more_experiments}}.
\minorrevision{Overall, on the two tasks across all datasets, model architectures, and model scales, our methods achieve $20\% \sim 50\%$ relative improvements on training and testing sparsity compared to the baseline. They achieve the most impressive improvement of approximately $50\%$ on experiments with ViT-B on ImageNet-1K, ViT-Large on Places365, and T5-Base on C4.} 
Despite the drastic improvement in sparsity, the drop in accuracy is minor on vision tasks. In \cref{table:productive_t5} for the NLP task, the testing loss even improves, likely due to the enhanced expressivity from the zeroth biases and the nonlinear $\jrelu$.

\minorrevision{
In vision tasks, architecture, model scale, and data distribution strongly impact vanilla and improved sparsity. Deeper networks (ViT-B with 12 MLP blocks vs SwT-B with 24 MLP blocks) tend to have stronger vanilla sparsity, possibly due to stronger gradient noise and stronger flatness optimization. The improved sparsity is maintained at the same level, however, leaving the improvement less significant. ViT-L, which is both deeper and wider, leads to equal enhancement on vanilla and improved sparsity, likely because wider networks have more redundancy. Therefore, our methods are the most effective if width and depth are scaled up together, which is often the case in modern practice. 
Surprisingly, a more diverse and challenging dataset does not severely damage vanilla sparsity, leaving it mainly determined by architecture and model scale. However, improving sparsity becomes difficult, likely because more active neurons are required to fit the data. This contrast is intermediate evidence that vanilla sparsity may contain a lot of redundancy intrinsic to architecture, while our methods can effectively remove it so that fitting difficulty controls sparsity.
}

\minorrevision{
From \cref{figure:productive_vit_average_training,figure:productive_vit_average_testing,figure:productive_swin_average_testing,figure:productive_vit_large_average_testing,figure:productive_laion_average_testing}, we observe that modified training improves sparsity through stronger initial sparsification and slower re-densification in the middle and late stages of training. From \cref{figure:productive_vit_end_training,figure:productive_vit_end_testing,figure:productive_swin_average_testing,figure:productive_vit_large_average_testing,figure:productive_laion_average_testing}, one can see that sparsity improvements mainly occur in middle layers, with the exception that the deeper half of SwT-Base layers are naturally sparse and improvements are concentrated in middle layers.}
Similar observations can be made for T5-Base, except that spatially the sparsity improvements concentrate in the middle layers of both the encoder and the decoder.

Unfortunately, the modified models exhibit sparsity overfitting, while vanilla models' sparsity generalizes well. Nevertheless, the sparsity improvement on the training set is so large that the improvement overwhelms the overfitting.

\subsection{Sparsity-Aware Finetuning}\label{sec:f_experiment}

We use the last checkpoints of vanilla Transformers from \cref{sec:p_experiments} as pretrained weights and finetune them on ImageNet-1K and C4 for better sparsity, demonstrating a cheaper way to improve sparsity on existing pretrained models.

To enable finetuning, we freeze pretrained parameters and deploy LoRA \cite{hu2022lora} at weight matrices in MLP blocks and query, key, and value matrices in attention blocks.
ViT-Base is tuned for 15 epochs on ImageNet-1K, while T5-Base is tuned for $10,000$ steps on the same subset of C4. More details on finetuning can be found in \cref{appendix:details}. 

For sparsified models, we add the finetuning version of the proposed modifications with the same configuration as \cref{sec:p_experiments}. The warmup of LayerNorm restriction and activation substitution lasts for $3,000$ steps.
The sparsity before and after finetuning is compared in \cref{table:finetuning}.

\begin{table}
    \centering
    \caption{Testing sparsity after finetuning \minorrevision{for ViT-Base on ImageNet-1K and T5-Base on C4} for better sparsity. \whatisbold{} \minorrevision{The second-best result is \underline{underlined}.}}\label{table:finetuning}
    \begin{tabular}{lrrrrrrrrrr}
        \toprule
                        &   \multicolumn{2}{c}{ViT-Base}                                                                          &   \multicolumn{2}{c}{T5-Base}\\
                        \cmidrule(lr){2-3}
        \cmidrule(lr){4-5}
                        &   \multicolumn{1}{c}{Sparsity \lowerbetter{}}                &   \multicolumn{1}{c}{Acc@1 \higherbetter{}}          &   \multicolumn{1}{c}{Sparsity \lowerbetter{}}                &   \multicolumn{1}{c}{Test Loss \lowerbetter{}}\\ 
        \midrule    
        Before F.T.   &   $\mathbf{0.087}$                                    &   $\mathbf{77.35\%}$             &   $0.299$                 &   $4.88$\\
        Vanilla         &   $0.122$                                             &   $\underline{77.17\%}$                          &   $0.304$                 &   $\mathbf{4.59}$\\
        Modified        &   $\underline{0.102}$                                             &   $74.46\%$                          &   $\mathbf{0.199}$            &   $\underline{4.61}$\\
        \bottomrule
    \end{tabular}
\end{table} 
For ViT, finetuning severely harms both testing sparsity and accuracy. Nevertheless, the modified model still has better sparsity than the baseline after finetuning.
However, for T5, the testing sparsity improves by approximately $1/3$ with a minor increase in testing loss.

\subsection{Analysis}

\mysubsubsection{Effects of Lower-bounded LayerNorm}

\begin{figure}
    \centering
    \subfloat[\label{fig:layernorm}]{
        \includegraphics[width=0.48\linewidth]{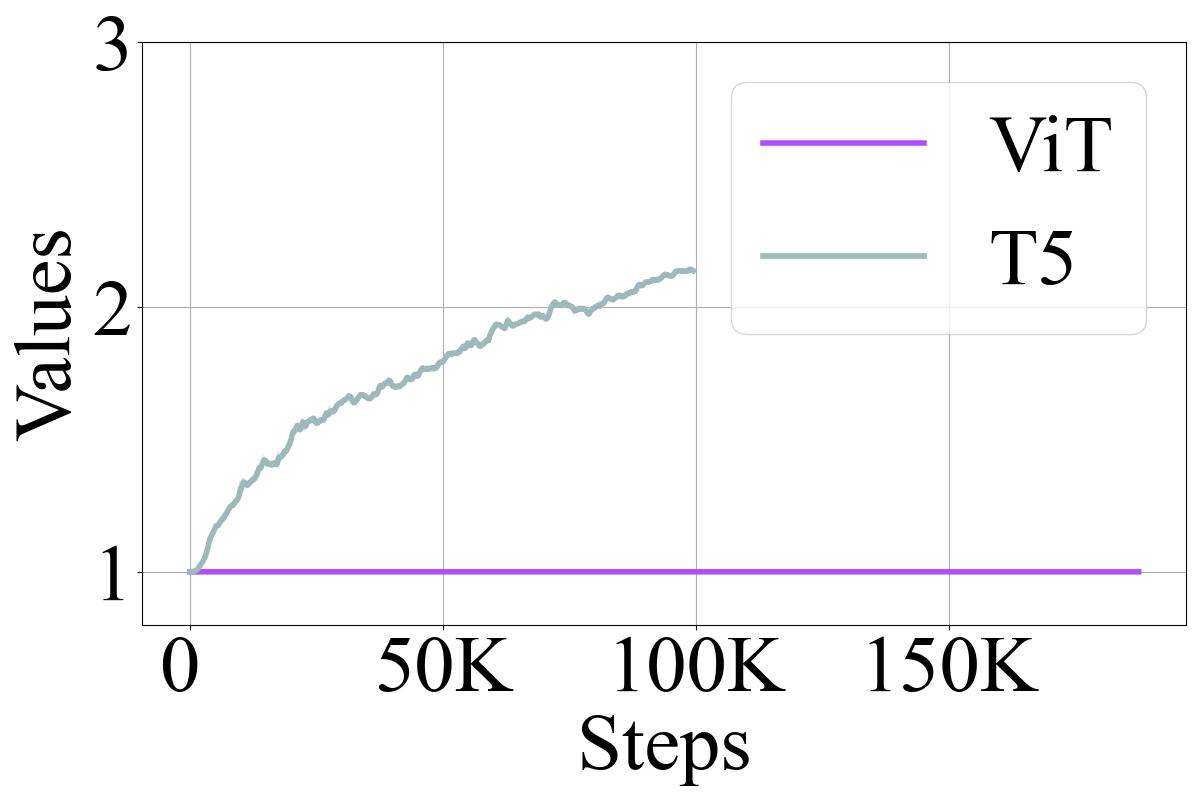}
    }
    \subfloat[\label{fig:c_with_lowerbound}]{
        \includegraphics[width=0.48\linewidth]{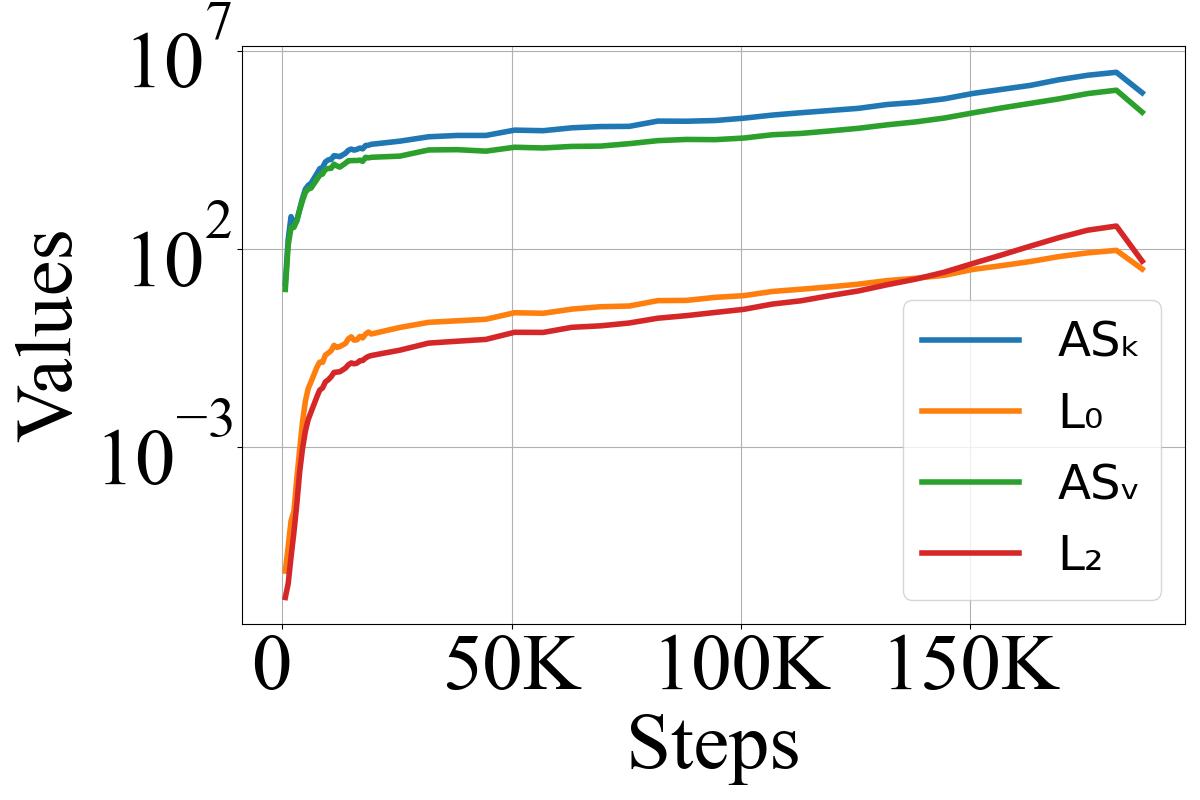}    
    }
    \caption{The role of lower-bounding LayerNorm layers. \cref{fig:layernorm} displays the averaged absolute values of affine parameters in LayerNorm layers, during the pretraining of the modified models. \cref{fig:c_with_lowerbound} displays how the $c^l_{L_0}, c^l_{L_2}$ evolve during the pretraining of the modified ViT. Full layerwise results can be found in \cref{appendix:more_experiments}.}
\end{figure}

The evolution of averaged absolute values of LayerNorm affine factors during sparsity-aware pretraining is shown in \cref{fig:layernorm}. It shows that ViT's factors closely adhere to $1$, suggesting that without the lower bound they would be much smaller, compromising the denominators $c_{L_0}^l$. For T5, \cref{fig:layernorm} shows little difference with or without the lower bound. The improvement for T5 is due to $\jrelu$ and Zeroth Biases.

To assess how lower-bounding LayerNorm layers mitigate the compromise in denominators $c_{L_0}^l$, we compute the two coefficients of the modified ViT in the same way as \cref{fig:full_run}, whose results are shown in \cref{fig:c_with_lowerbound}. Note that \cref{fig:c_with_lowerbound} does not include the denominator increase of $\jrelu$ (see \cref{eq:jsrelu_larger_denominator}) and only displays the increase due to lower-bounded LayerNorms in the sparsified ViTs.
It can be seen that $c^l_{L_0}$ becomes larger compared to \cref{fig:full_run_imagenet1k}. As a result, the ratio in \cref{theorem:multi_token} becomes smaller, and better activation sparsity is gained, given that the augmented flatness remains similar.

\mysubsubsection{Ablation on Lower-bounded LayerNorm}\label{sec:ablation}

To assess the contribution of the lower bound among the proposed methods for ViT, we present results from tentative experiments with Zeroth Biases and $\jrelu$ on ImageNet-1K, where LayerNorm affine factors and Zeroth Biases are unbounded. Besides removing restrictions, we also used a larger gradient clipping norm of up to $10$, hoping for better flatness. However, this leads to unstable training, likely due to the larger derivatives of $\jrelu$. Other setting differences from \cref{sec:p_experiments} are listed in \cref{appendix:details}.
The results are shown in \cref{fig:ablation}.

The results show that while the training sparsity improves, the degree of improvement is much smaller.
The preference for augmented flatness on $L_2$ terms also appears much earlier, leading to inferior training sparsity measured by the $L_0$ pseudo norm in later training. 
This indicates that lower-bounded LayerNorm is critical for improving sparsity in ViT, as are $\dbmlp$ and $\jrelu$.

\section{Conclusion, Limitation, and Future}\label{sec:conclusion}

In this paper, we focus on the emergence and enhancement of activation sparsity in MLP blocks.
Regarding emergence, we derive equalities linking activation sparsity to a ratio between an augmented form of flatness and a denominator related to MLP input norms and gradient magnitudes of MLP activations.
By empirically measuring the augmented flatness, the denominator, and the ratio, we observe that the denominator increases faster than the augmented flatness, causing the ratio to decrease as training proceeds and leading to activation sparsity.
We propose derivative sparsity, which allows pruning during backward propagation and is more stable than activation sparsity.
Based on the theoretical roles of flatness, activation functions, and normalization layers, we propose three simple plug-and-play methods to improve sparsity. Experimental results demonstrate that sparsity-aware pretraining substantially improves both training and testing sparsity compared with baseline Transformers on natural image classification and natural language generation tasks. This improvement offers strong potential cost reductions during both pretraining and inference.

\begin{figure}[t]
    \centering \subfloat[Modified, overall: 0.115.\label{fig:ablation_training}]{
        \includegraphics[width=0.48\linewidth]{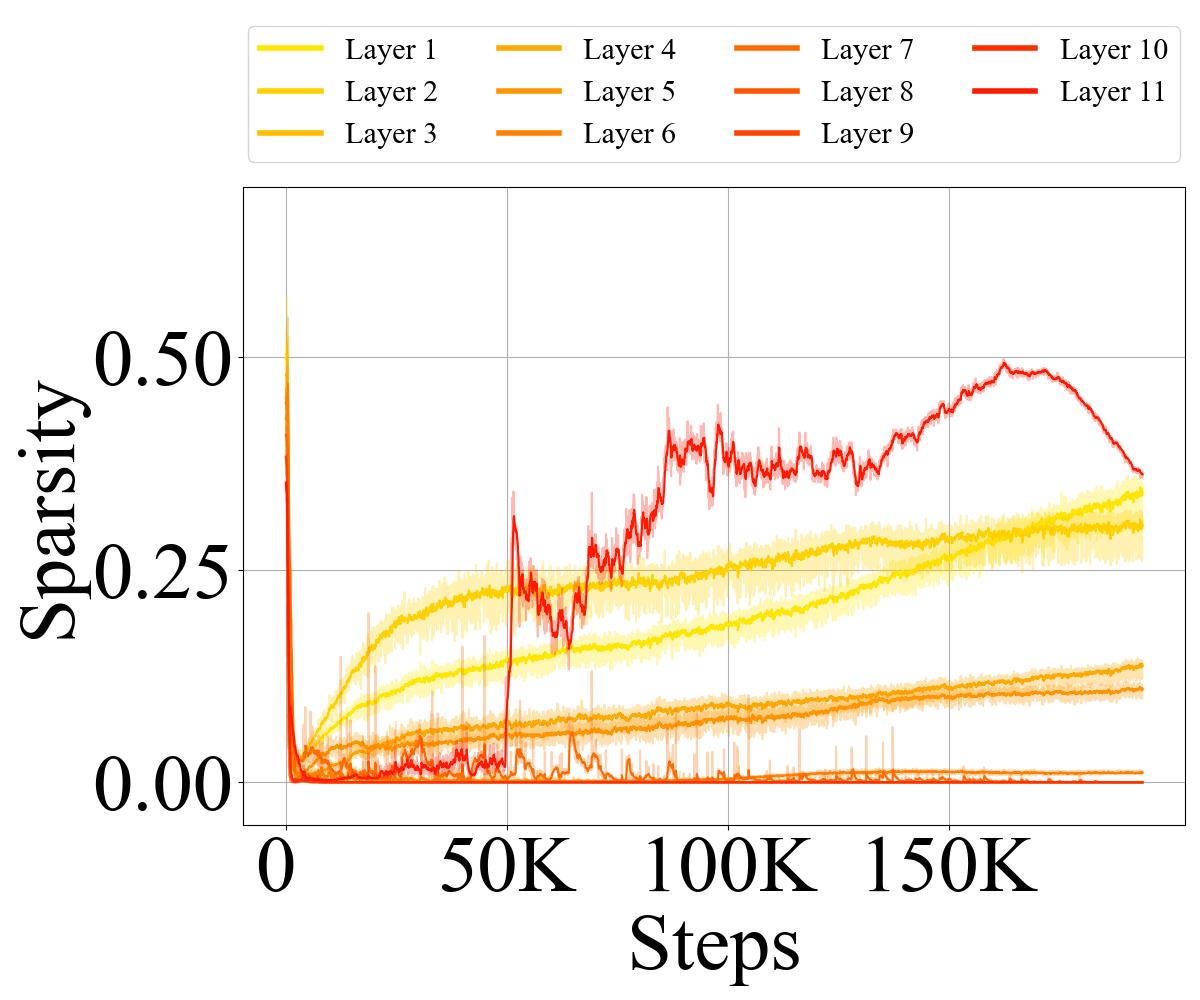}
    }
    \subfloat[Vanilla, overall: 0.130.\label{fig:ablation_testing}]{
        \includegraphics[width=0.48\linewidth]{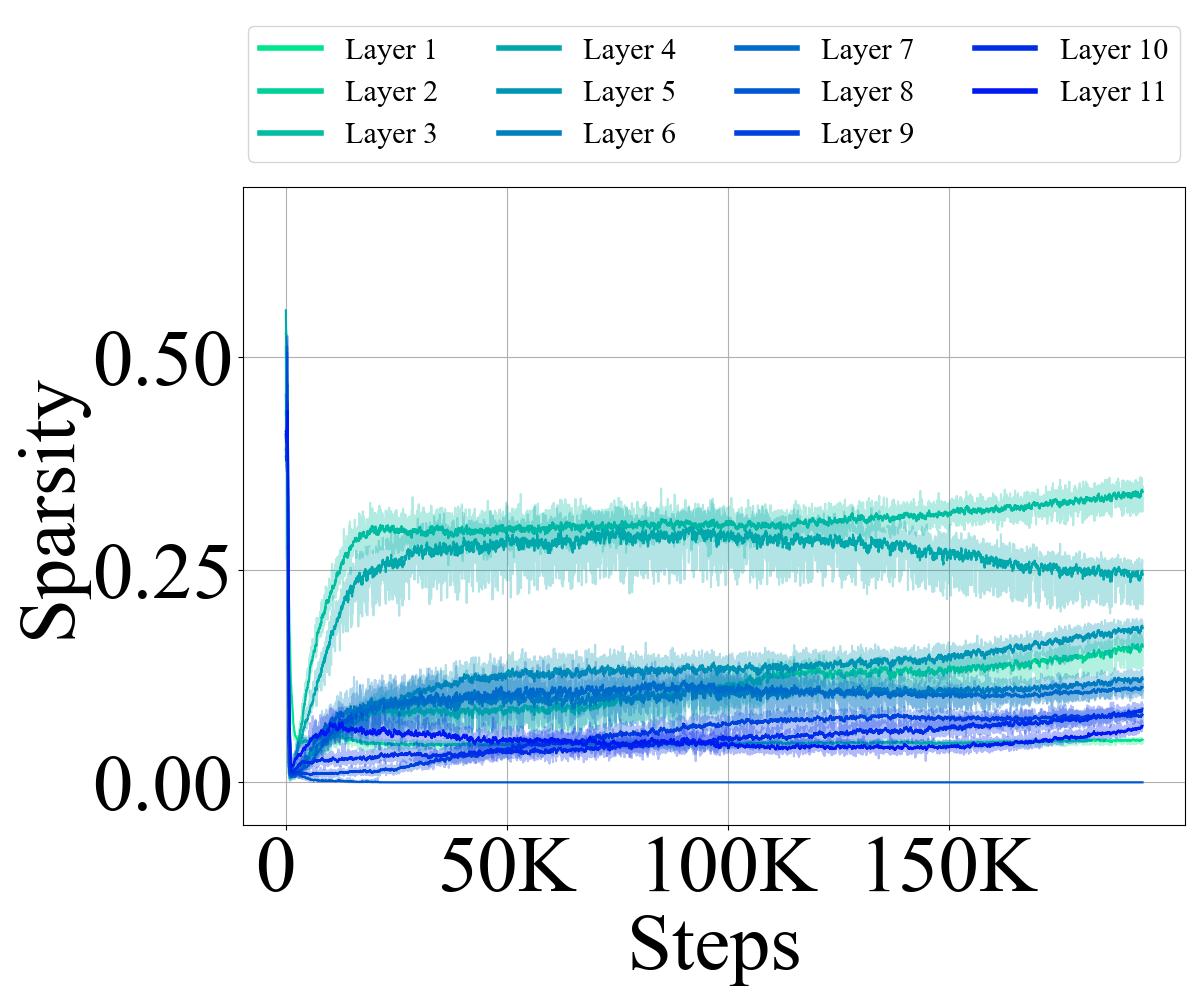}
    }
    \caption{The training sparsity of ViT trained with only (unrestricted) $\dbmlp$ and $\jrelu$, without restrictions on LayerNorm layers and $\dbmlp$. ``Overall'' indicates their overall training sparsity across layers and steps.}\label{fig:ablation}
\end{figure}

Activation sparsity is connected to augmented flatness, which may attract more research attention towards understanding the emergence of flatness.
However, our conclusion on sparsity emergence is supplemented by empirical results. Theoretical explanations for why the measured quantities behave as observed in our findings remain unknown.
Lastly, the enhanced sparsity exhibits overfitting to the training set. Developing methods to alleviate this overfitting could further improve inference efficiency.
 
\minorrevision{
We discuss conjectures and potential future studies on the above limitations.
For a more complete picture on sparsity emergence, we may need to rewrite \cref{eq:main} into $\AS[\Kparam^l] = \ex{\norm{\mat{A}^l}_0} \cdot d \cdot \ex{\left(\derivatives{\sampleloss(f_{\allparam}, \vec{s})}{a^l_{i, j}}\right)^2 \mid a^l_{i, j} > 0}$. This indicates that better sparsity can lead to better flatness, and the training \emph{may} implicitly optimize sparsity to gain flatness. 
However, as observed in the above equality, flatness is not solely determined by sparsity, and training may optimize flatness by decreasing the other factor and totally ignoring sparsity. 
We argue that sparsity optimization is not ignorable, if not dominant, in the flatness optimization. At least in pure MLPs, the other factor is only determined by the sparsity of other layers and weights (especially the spectral properties of the weights). Studies on training dynamics show that the ubiquitous small initialization and (S)GD encourage the weights to be spectrally auto-aligned \cite{autoalignment1,autoalignment2,autoalignment3,autoalignment4}, which are exactly the flattest solution \cite{mulayoff2020unique}. Therefore, the other factor has been fully optimized by (S)GD, and the only factor that can be further optimized is sparsity. This explains why sparsity must be enhanced when SGD seeks flatness. 
Additionally, flatness is mainly optimized through escape from sharp minima driven by stochastic gradient noise \cite{mori2022power,alignment}. It is unlikely that stochastic escaping solely favors the other factor. Minima with sparsity will be visited and can trap training due to better flatness.
Nevertheless, this is only a sketchy conjecture, and we leave the rigorous proof for future studies. 
For sparsity overfitting, if we treat sparsity as a special loss, information-theoretic or PAC-Bayesian generalization theories that impose few assumptions on loss can help. 
Since sparsity percentage $\norm{\mat{D}^l}_0 / k n \in [0, 1]$ is bounded and thus satisfies the sub-Gaussian condition in these bounds \cite{xu_information-theoretic_2017,alquier_user-friendly_2023}, the gap between training and testing sparsity percentages can be bounded by $\sqrt{I(\dataset; \allparam) / 2 |\dataset|}$, where $I(\dataset; \allparam)$ is the mutual information between the randomly drawn training dataset and the learned parameters. See \cref{theorem:sparsity_overfitting} in \cref{appendix:sparsity_overfitting} for formal results.
More advanced bounds \cite{steinke_CMI,asadi_chaining_2018,hafez-kolahi_conditioning_2020,wang_tighter_2023,peng2025leveraging} can also be applied.
Theoretically, this places sparsity overfitting in a unified framework with mainstream studies of loss or accuracy overfitting and allows progress to transfer across both fields. 
For algorithms, regularization techniques based on information-theoretic or PAC-Bayesian theories may help mitigate sparsity overfitting, which we leave for future studies. }
 \section*{Acknowledgments}

This work was supported by NSFC Project (62506162, 62506005, 62192783), Jiangsu Science and Technology Project (BK20251241), Fundamental and Interdisciplinary Disciplines Breakthrough Plan of the Ministry of Education of China (No. JYB2025XDXM118), and the "111 Center" (No. B26023).

\bibliographystyle{IEEEtran}
\bibliography{ref}

\onecolumn
\appendices

\crefalias{section}{appendix}
\newcommand{\appendixcounter}[1]{\counterwithin{#1}{section}}
\appendixcounter{figure}
\appendixcounter{lemma}
\appendixcounter{theorem}
\appendixcounter{table}
\appendixcounter{equation}
\appendixcounter{definition}
\appendixcounter{corollary}

\section{Results for Architectures with Skip Connections}\label{appendix:full}

In the main part of the paper, we list theoretical results for networks defined by \cref{eq:architecture}, where networks can have MLP blocks in $\mlp^l$ and non-MLP structures in $G^l$, allowing self-attention layers, normalization layers or other constructs.
However, skip connections, a imperative components in modern architectures, are excluded in \cref{eq:architecture} because MLP blocks cut the input and the output of the network.
Nevertheless, the results can be extended to architectures with skip connections.
In this section, we define a more general family of networks that embrace skip connections and derive theoretical results on these networks.
The results in the main part are then implications of these results.

We begin the definition of the networks with skip connections. 
\begin{definition}\label{def:arch}
    A parameterized function $F_{\allparam}: \reals^{d_{\textnormal{input}}} \to \reals^C$ of token number $k$ is a neural network of our interest if there exists a depth $L \in \nats^+$, and also for each layer index $l \in [L]$, 
        an MLP-input-extractor $T^l_{\allparam}: \mathcal{S}_{l-1} \to \reals^{d \times k}$, an MLP-output-reducer $R^l_{\allparam}: (\reals^{d \times k} \times \mathcal{S}_{l-1}) \to \mathcal{S}_{l}$, and finally the input embedding $E: \reals^{d_{\textnormal{input}}} \to \mathcal{S}_0$ and the classifier $D_{\allparam}: \mathcal{S}_L \to \reals^{C}$, such that $F$ can be recursively defined by
    \begin{align}
        \mat{S}^0 \defeq& E(\vec{x}), \label{eq:arch_start}\\
        \mat{X}^l \defeq& T^l_{\allparam}(\mat{S}^{l-1}),\\
        \mat{Z}^l   \defeq& \mlp^l_{\allparam}(\mat{X}^{l}) \in \reals^{d \times k},\\
        \mat{S}^l \defeq& R^l_{\allparam}(\mat{Z}^l, \mat{S}^{l-1}),\\
        F(c \mid \vec{x}, \allparam) \defeq& (\operatorname{softmax}(D_{\param[D]}(\mat{S}^L)))_c,
    \end{align}
    where $l \in [L]$, $\operatorname{softmax}(\vec{z})_j = \exp(z_j) / \sum_{j'} \exp(z_{j'})$, and for $l \in \set{0} \cup [L]$, $\mathcal{S}^l$ is a non-empty set, possibly with complicated structures such as nested Cartesian products.
\end{definition}

In \cref{def:arch}, $\mat{X}^l$ and $\mat{Z}^l$ are the input and output, respectively, to the $l$-th MLP block. $E, T^l, G^l$, and $D$ contain all none-MLP structures such as self-attention and normalization layers. $\mat{S}^l$ contains the features that skip the $l$-th MLP, covering skip connections. 
Therefore, one can suit architectures listed in \cref{sec:intro} into $F$. For example, to suit Transformers, $G^l$ absorbs the addition for the residual connection, self-attention layers, and normalization layers.
To suit \cref{eq:architecture}, we denote the input embedding in \cref{eq:architecture} by $E'$ and the classifier by $D'$, then let $E = E', T^l = G^l, R^l_{\allparam}(\mat{Z}^l, \mat{S}^{l-1}) = \mat{Z}^l$ and finally $D = D'$. Therefore, results in this section can be transferred to models defined by \cref{eq:architecture}, giving results in the main part.

We then derive results based on \cref{def:arch} in the same order as the main part.
First, we show that sub-derivatives can be used to compute gradients in $F_{\allparam}$ when derivatives are undefined.
\begin{lemma}\label{lemma:differentiable}
    Let $F_{\allparam}$ be the neural network defined in \cref{def:arch} with parameter $\allparam$, whose activation function $\sigma$ is continuous. Under \cref{assumption:differentiable}, for any training sample $(\vec{x}, y) \in \dataset$, both $\derivatives{\sampleloss(F_{\allparam}, (\vec{x}, y))}{\mat{P}^l}$ and $\derivatives{\sampleloss(F_{\allparam}, (\vec{x}, y))}{\mat{A}^l}$ exist, and one has 
    \begin{align}
        \derivatives{\sampleloss(F_{\allparam}, (\vec{x}, y))}{\mat{P}^l} = \derivatives{\sampleloss(F_{\allparam}, (\vec{x}, y))}{\mat{A}^l} \hadamard \mat{D}^l.
    \end{align}
\end{lemma}
The proof of \cref{lemma:differentiable} can be found in \cref{proof:differentiable}.

The proof of \cref{lemma:af_and_gradient_norm} is agnostic to network structure so it is not repeated.

We then decompose the expected gradient norms in \cref{lemma:af_and_gradient_norm} to bring activations and (sub-)derivatives in:
\begin{lemma}\label{lemma:main_full}
    Let $F$ be a neural network defined by \cref{def:arch}. Assume \cref{assumption:differentiable} for $F_{\allparam}$ on $\dataset$. Then for $l \in [L]$, we have 
    \begin{align}
        \AS[\Kparam^l]
        =& \ex[\vec{s} \sim \uniform{\dataset}]{\norm{\left(\mat{X}^l\right)^\transpose \left(\derivatives{\sampleloss(F_{\allparam}, \vec{s})}{\mat{A}^l} \hadamard \mat{D}^l \right)}_F^2},
        \AS[\Vparam^l]=\ex[\vec{s} \sim \uniform{\dataset}]{\norm{\left(\mat{A}^l\right)^\transpose \derivatives{\sampleloss(F_{\allparam}, \vec{s})}{\mat{Z}^l}}_F^2}.
    \end{align}
\end{lemma}
The lemma, proved in \cref{proof:main_full}, implies \cref{lemma:main}.

Under single-token scenarios, the matrices in \cref{lemma:main_full} are essentially vectors, allowing the swap between the Frobenius norm and the matrix multiplication.
\begin{lemma}\label{lemma:single_token_full}
    Let $F$ be a neural network defined by \cref{def:arch}. Assume \cref{assumption:differentiable} for $F_{\allparam}$ on $\dataset$.
    If $F_{\theta}$ is single-token, then for $l \in [L]$,
    \begin{align}
        \AS[\Kparam^l] =& \ex[\vec{s} \sim \uniform{\dataset}]{\norm{\derivatives{\sampleloss(F_{\allparam}, \vec{s})}{\mat{A}^l} \hadamard \mat{D}^l}_F^2 \norm{\mat{X}^l}_F^2},\\
        \AS[\Vparam^l] =& \ex[\vec{s} \sim \uniform{\dataset}]{\norm{\derivatives{\sampleloss(F_{\allparam}, \vec{s})}{\mat{Z}^l}}_F^2 \norm{\mat{A}^l}_F^2}. 
\end{align}
\end{lemma}
The lemma, proved in \cref{proof:single_token_full}, implies \cref{lemma:single_token}.

We then exploit the property that $\mat{D}^l$ is 0-1 valued under $\relu$. 
To this end, one needs \cref{lemma:L0_and_L2}, which is proved in \cref{proof:L0_and_L2}.
\begin{lemma}\label{lemma:L0_and_L2}
    Let $S, h, w \in \nats^+$. Let $\set{\mat{U}^s}_{s \in [S]} \subseteq \set{0, 1}^{h \times w}$ be $S$ $0$-$1$ matrices and $\set{\mat{V}^s}_{s \in [S]} \subseteq \reals^{h \times w}$.
    Let $D$ be any distribution over $[S]$, then 
    \begin{align}
        \ex{\norm{\mat{V}^s \hadamard \mat{U}^s}_F^2} = \ex{\left(v^s_{i, j}\right)^2 \mid u^s_{i, j} = 1} \cdot \ex{\norm{\mat{U}^s}_0}
    \end{align}
    with expectations taken over $s \sim D, (i, j) \sim \uniform{[h] \times [w]}$.
\end{lemma}
This lemma helps extract the $L_0$ pseudo norm of $\mat{D}^l$, leading to the \cref{theorem:single_token_full} for single-token scenarios. The theorem is proved in \cref{proof:single_token_theorem_full} and implies \cref{theorem:single_token}.
\begin{theorem}\label{theorem:single_token_full}
    Let $F$ be a neural network defined by \cref{def:arch}. Assume \cref{assumption:differentiable} for $F_{\allparam}$ on $\dataset$.
    If $F_{\allparam}$ is used under single-token scenarios, then for $l \in [L]$,
    \begin{align}
        \ex[\vec{s} \sim \uniform{\dataset}]{\norm{\mat{D}^l}_0}
        =&\frac{\AS[\Kparam^l]}{\ex{\norm{\mat{X}^l}_F^2 \left(\derivatives{\sampleloss(F_{\allparam}, \vec{s})}{a^l_{i, j}}\right)^2 \left(d^l_{i, j}\right)^2 \mid d^l_{i, j} > 0}},\\
        \AS[\Vparam^l]=& \ex[\vec{s} \sim \uniform{\dataset}]{\norm{\derivatives{\sampleloss(F_{\allparam}, \vec{s})}{\mat{Z}^l}}_F^2 \norm{\mat{A}^l}_F^2}.
    \end{align}
    If further MLP blocks use $\relu$ as activation functions, then we have 
    \begin{align}
        \ex[\vec{s} \sim \uniform{\dataset}]{\norm{\mat{A}^l}_0}
        =&\ex[\vec{s} \sim \uniform{\dataset}]{\norm{\mat{D}^l}_0}\\
        =&\frac{\AS[\Kparam^l]}{\ex{\norm{\mat{X}^l}_F^2 \left(\derivatives{\sampleloss(F_{\allparam}, \vec{s})}{a^l_{i, j}}\right)^2 \mid a^l_{i, j} > 0}}.
    \end{align}
    If further inputs to MLP blocks are LayerNorm-ed with LayerNorms' affine parameters and $\epsilon_{\textnormal{LayerNorm}}$ turned off, we have
    \begin{align}
        \ex[\vec{s} \sim \uniform{\dataset}]{\norm{\mat{A}^l}_0}
        =&\frac{\AS[\Kparam^l]}{d \cdot \ex{\left(\derivatives{\sampleloss(F_{\allparam}, \vec{s})}{a^l_{i, j}}\right)^2 \mid a^l_{i, j} > 0}}.
    \end{align}
\end{theorem}

\newcommand{\decode}{\textrm{dec}}
\newcommand{\encode}{\textrm{enc}}
\newcommand{\codedSkip}[1]{(\tilde{\mat{Z}}^{#1}, \tilde{\mat{S}}^{#1})}
\newcommand{\codedTuple}[1]{(\tilde{\vec{u}}^{#1}, \codedSkip{#1})}
\newcommand{\fetchToken}[1]{\textrm{fetch}^{#1}}
\newcommand{\storeToken}[1]{\textrm{put}^{#1}}
For multi-token scenarios, we also copy the shared parameters in \cref{def:arch} as mentioned in the main part.
\begin{definition}\label{def:unwrapped_full}
    Let $F$ be a neural network by \cref{def:arch} with token number $k$. The copied architecture $\tilde{F}$ of $F$ has $p + (k-1) \cdot (2 n d + n + d)$ parameters and is recursively defined by
    \begin{align}
        \codedSkip{0} \defeq& (\vec{0}, E(\vec{x})),\\
        \codedSkip{1, 0} \defeq&    \codedSkip{0},\\
        \tilde{\vec{x}}^{l, i} \defeq& \fetchToken{i}(T^{l, i}_{\tildeallparam}(\decode(\codedSkip{l, i-1}))),\label{eq:layer_start}\\
        \tilde{\vec{z}}^{l, i} \defeq& \mlp^{l, i}_{\tildeallparam}(\tilde{\vec{x}}^{l, i}) \in \reals^d,\label{step:mlp_one_token}\\
        \codedSkip{l, i} \defeq& \storeToken{i}(\tilde{\vec{z}}^{l, i}, \codedSkip{l, i-1})\label{step:before_tilde_G},\\
        \codedSkip{l} \defeq& (\vec{0}, R^l_{\tildeallparam}\codedSkip{l, k}),\label{step:tilde_G}\\
        \tilde{F}(c \mid \vec{x}, \tildeallparam) \defeq& D_{\param[D]}(\decode(\tilde{\mat{Z}}^L, \mat{S}^L))_c,
    \end{align}
    where $l \in [L], i \in [k]$, and $T^{l, i} = T^l$, $\mlp^{l, i} = \mlp^l$, and
    \begin{align}
        \decode(\codedSkip{}) \defeq& \tilde{\mat{S}},
        \encode(\tilde{\mat{S}}) \defeq (\vec{0},\tilde{\mat{S}}),\\
        \fetchToken{i}(\tilde{\mat{X}}) \defeq& \tilde{\vec{x}}_{i},\\
        \storeToken{i}(\tilde{\vec{z}}, \codedSkip{}) \defeq& \begin{cases}
            (\tilde{\vec{z}}, \tilde{\mat{S}})  &   i =1, \\
            \left(\begin{bmatrix}
                \tilde{\mat{Z}} & \tilde{\mat{z}}
            \end{bmatrix}, \tilde{\mat{S}}\right) &   i > 1.
        \end{cases}
    \end{align}

    The copied parameter $\tildeallparam$ of $\allparam$ is defined so that $R^l$ uses the same parameters in $\tilde{F}_{\tildeallparam}$ and $F_{\allparam}$, $T^{l, i}$ in $\tilde{F}_{\tildeallparam}$ uses the same parameters as $T^l$ in $F_{\allparam}$, and $\mlp^{l, i}$ in $\tilde{F}_{\tildeallparam}$ uses the same parameters as $\mlp^l$ in $F_{\allparam}$, \ie{} $T^l$ and $\mlp^l$ are copied $k$ times.
    Let $\mat{K}^{l, i}$ and $\mat{V}^{l, i}$ be the weight matrices in $\mlp^{l, i}_{\allparam}$.
    Let $\tildeKparam^l$ and $\tildeVparam^l$ be the collection of parameters in $\mat{K}^{l, i}$ and $\mat{V}^{l, i}$, respectively.
    For a subset $\tildeparam$ of copied parameter $\tildeallparam$, define $\tildeAS[\tildeparam]$ be the augmented flatness \wrt{} $\tildeparam$ in $\tilde{F}_{\tildeallparam}$.
\end{definition}

\cref{def:unwrapped_full} transforms one $T^l$-$\mlp^l$-$R^l$ structure in \cref{def:arch} into \cref{eq:layer_start}-\cref{step:tilde_G}. In these steps, \cref{eq:layer_start}-\cref{step:before_tilde_G} are iteratively executed $k$ times to simulate multi-token MLPs with copied single-token MLPs. In each iteration, $\tilde{F}$ process one token by extracting one input token from the hidden features in \cref{eq:layer_start}, computing the output token using the single-token MLP in \cref{step:mlp_one_token}, and storing the output token in the hidden feature passed to later iterations in \cref{step:before_tilde_G}.
Mapping the superscripts into natural numbers and merging \cref{step:before_tilde_G} and \cref{step:tilde_G} when $i=k$, we can suit the copied network into the definition in \cref{sec:arch}, although we will not do it for better presentation.
Note that since in $\tilde{F}_{\tildeallparam}$, each MLP block $\mlp^{l, i}$ only takes one token, $\tilde{F}_{\tildeallparam}$ is always a single-token network. Therefore, we can reuse the above lemmas and theorems for single-token networks:
\begin{lemma}\label{lemma:multi_to_single}
    For any $F_{\allparam}$ defined in \cref{def:arch}, the corresponding $\tilde{F}_{\tildeallparam}$ is single-token.
\end{lemma}
After applying the above lemmas and theorems, we will use the correspondence between $\tilde{F}_{\tildeallparam}$ and $F_{\allparam}$to chain the equalities about the sparsities in $\tilde{F}_{\tildeallparam}$ with the sparsities in $F_{\tildeallparam}$. 
To give practical meanings to the augmented flatness of $\tilde{F}_{\tildeallparam}$, we also connect the augmented flatness of $\tilde{F}_{\tildeallparam}$ with the robustness of $F_{\allparam}$ \wrt{} parameter noises.
The results are given in \cref{theorem:multi_token_full}, which is proved in \cref{proof:multi_token_full} and implies \cref{theorem:multi_token}.
\begin{theorem}\label{theorem:multi_token_full}
    Let $F$ be a neural network defined by \cref{def:arch}. Assume \cref{assumption:differentiable} for $F_{\allparam}$ on $\dataset$. Let $\tilde{F}$ and $\tildeallparam$ be the corresponding copied architecture and copied parameters.  
    Let $\mat{E}^{l, i}_K, \mat{E}^{l, i}_V$ be two random matrices of the same shape as $\tilde{\mat{K}}^{l, i}, \tilde{\mat{V}}^{l, i}$, respectively, whose entries are independently sampled from the uniform distribution over $[-\sqrt{3} \sigma, +\sqrt{3}\sigma]$.
    Let $\tildeparam[l, K]'$ be the perturbed version of $\tildeallparam$ after perturbing $\tildeKparam^l$, where for all $i \in [k]$, $\tilde{\mat{K}}^{l, i}$ is replaced by $\tilde{\mat{K}}^{l, i} + \mat{E}^{l, i}_K$, while all other parameters are kept the same. Define $\tildeparam[l, V]'$ in a similar way except that parameters in $\tilde{\mat{V}}^{l, i}, i \in [k]$ are perturbed. 
    Then we have
    \begin{align}
        &\frac{1}{\sigma^2} \ex{\hat{\loss}(\tilde{F}_{\tildeparam[l, K]'}) - \hat{\loss}(\tilde{F}_{\tildeallparam})}  + \frac{1}{\sigma^2} \ex{\frac{\tilde{F}(y \mid \tildeparam[l, K]', \vec{x}) - \tilde{F}(y \mid \tildeallparam, \vec{x})}{F(y \mid \allparam, \vec{x})}}  - o(1) \cdot n d k\\
        =&   \tildeAS[\tildeKparam^l]
        =   \ex{\norm{\nabla_{\tildeKparam^l} \sampleloss(\tilde{F}_{\tildeallparam}, \vec{s})}_2^2}\\
        =&   \ex{\norm{\vec{x}^l_i}_2^2 \left(\derivatives{\sampleloss}{a^l_{i, j}}\right)^2 \left(d^l_{i, j}\right)^2 \mid d^l_{i, j} > 0} \cdot \ex[\vec{s} \sim \uniform{\dataset}]{\norm{\mat{D}^l}_0},\label{eq:multi_token_full_L0}\\
        &\frac{1}{\sigma^2} \ex{\hat{\loss}(\tilde{F}_{\tildeparam[l, V]'}) - \hat{\loss}(\tilde{F}_{\tildeallparam})}  + \frac{1}{\sigma^2} \ex{\frac{\tilde{F}(y \mid \tildeparam[l, V]', \vec{x}) - \tilde{F}(y \mid \tildeallparam, \vec{x})}{F(y \mid \allparam, \vec{x})}}  - o(1) \cdot n d k\\
        =&   \tildeAS[\tildeVparam^l]
        =   \ex{\norm{\nabla_{\tildeVparam^l} \sampleloss(\tilde{F}_{\tildeallparam}, \vec{s})}_2^2}\\
        =& \sum_{l \in [L]} \ex[\vec{s} \sim \uniform{\dataset}]{\sum_{i \in [k]} \norm{\derivatives{\sampleloss}{\vec{z}^l_i}}_F^2 \norm{\vec{a}^l_i}_F^2}\label{eq:multi_token},
    \end{align}
    where the expectations at \lhs{} are taken over the choice of noises and $\vec{s} \sim \uniform{\dataset}$, independently, and the $o(1)$ term tends to $0$ as $\sigma \to 0$ for any fixed $n, d, k$.

    If further all MLP blocks use $\relu$ activation functions, then \cref{eq:multi_token_full_L0} can be further chained by
    \begin{align}
        \tildeAS[\tildeKparam^l] =&   \ex{\norm{\vec{x}^l_i}_2^2 \left(\derivatives{\sampleloss}{a^l_{i, j}}\right)^2 \mid a^l_{i, j} > 0} \cdot \ex[\vec{s} \sim \uniform{\dataset}]{\norm{\mat{A}^l}_0}.\\
    \end{align}
    If further all inputs to MLP blocks are LayerNorm-ed with affine parameters and $\epsilon_{\textnormal{LayerNorm}}$ turned off, the \cref{eq:multi_token_full_L0} can be further chained by 
    \begin{align}
        \tildeAS[\tildeKparam^l] =&   d \cdot \ex{\left(\derivatives{\sampleloss}{a^l_{i, j}}\right)^2 \mid a^l_{i, j} > 0} \cdot \ex[\vec{s} \sim \uniform{\dataset}]{\norm{\mat{D}^l}_0}.\\
    \end{align}
\end{theorem}

We emphasize that the entire chained equalities in \cref{theorem:multi_token_full} are properties of $F_{\allparam}$, because $\tilde{F}_{\tildeallparam}$ is the result of an operator ``copy'' on $F_{\allparam}$ and reveal properties of $F_{\allparam}$. 
In addition to the relation between the augmented flatness and sparsities, \cref{theorem:multi_token_full} indicates that the augmented flatness is related to the robustness \wrt{} noise parameters. 
The noises are added to matrix copies $\mat{K}^{l, i}$ and $\mat{V}^{l, i}$ in $\tilde{F}_{\tildeallparam}$, and in the language of $F_{\allparam}$, the noises are added to $\mat{K}^l$ and $\mat{V}^l$ independently to every \emph{single} multiplication with \emph{individual} tokens. Since these token-wise parameter noises are propagated to $\mat{A}^l$ or $\mat{Z}^l$ and $\mat{A}^l$ or $\mat{Z}^l$ bottleneck $\mat{K}^l$ and $\mat{V}^l$, the noises can also be thought as hidden feature noises added to $\mat{A}^l$ and $\mat{Z}^l$, where each token in $\mat{A}^l$ and $\mat{Z}^l$ is perturbed independently. 

 \section{Proofs}\label{appendix:proofs}

\subsection{Proof of \cref{lemma:differentiable}}\label{proof:differentiable}

By \cref{assumption:differentiable}, both $\derivatives{\sampleloss(F_{\allparam}, (\vec{x}, y))}{\mat{P}^l}$ and $\derivatives{\sampleloss(F_{\allparam}, (\vec{x}, y))}{\mat{A}^l}$ exist because they are the first-order derivatives \wrt{} hidden features $\mat{P}^l$ and $\mat{A}^l$.

We then prove for any $(i, j) \in [k] \times [n]$, we have 
\begin{align}
    \derivatives{\sampleloss}{p^l_{i, j}} = \derivatives{\sampleloss}{a^l_{i, j}} \cdot d^l_{i, j}. \label{step:differentiable}
\end{align}

If $\sigma'(p^l_{i, j})$ exist, then \cref{step:differentiable} holds by chain rule.
Otherwise since $\sigma'(p^l_{i, j})$ does not exist, $\lim_{\delta \to 0} \frac{\sigma(p^l_{i, j} + \delta) - \sigma(p^l_{i, j})}{\delta}$ does not exist. Since $\sampleloss$ is twice-differentiable \wrt{} $a^l_{i, j}$, we can approximate the loss change \wrt{} $a^l_{i, j}$ by 
\begin{align}
    \sampleloss\mid_{a^l_{i, j} + \delta'} - \sampleloss = \derivatives{\sampleloss}{a^l_{i, j}} \cdot \delta' + o(|\delta'|).
\end{align}
So if $\derivatives{\sampleloss}{a^l_{i, j}} \neq 0$, we have 
\begin{align}
    \lim_{\delta \to 0} \frac{\sampleloss\mid_{p^l_{i, j} + \delta} - \sampleloss}{\delta}
    =&  \lim_{\delta \to 0} \frac{\derivatives{\sampleloss}{a^l_{i, j}} \cdot (\sigma(p^l_{i, j} + \delta) - \sigma(p^l_{i, j})) + o(|\sigma(p^l_{i, j} + \delta) - \sigma(p^l_{i, j})|)}{\delta} \\
    =&  \derivatives{\sampleloss}{a^l_{i, j}} \lim_{\delta \to 0} \frac{(\sigma(p^l_{i, j} + \delta) - \sigma(p^l_{i, j}))}{\delta},
\end{align}
which do not exist, contradicting with the existence of $\derivatives{\sampleloss}{p^l_{i, j}}$.
Therefore, when $\sigma'(p^l_{i, j})$ does not exist, we have $\derivatives{\sampleloss}{a^l_{i, j}} = \derivatives{\sampleloss}{p^l_{i, j}} = 0$ under \cref{assumption:differentiable}. In this case, we still have 
\begin{align}
    \derivatives{\sampleloss}{p^l_{i, j}} = \derivatives{\sampleloss}{a^l_{i, j}} \cdot d^l_{i, j}. 
\end{align} 

\subsection{Proof of \cref{lemma:L0_and_L2}} \label{proof:L0_and_L2}

\newcommand{\randomIndices}{s \sim D, (i, j) \sim \uniform{[h] \times [w]}} 
\begin{align}
    \ex[s \sim D]{\norm{\mat{V}^s \hadamard \mat{U}^s}_F^2}
    =&  \ex[s \sim D]{h w \cdot \ex[(i, j) \sim \uniform{[h] \times [w]}]{\left(v^s_{i, j}\right)^2 \left(u^s_{i, j}\right)^2}}\\
    =&  h w \cdot \ex[\randomIndices]{\left(v^s_{i, j}\right)^2 \left(u^s_{i, j}\right)^2} \label{step:omit_randomness}\\
    =&  h w \cdot \left(
            \prob{u^s_{i, j} = 1} \cdot 1 \cdot \ex{\left(v^s_{i, j}\right)^2 \mid u^s_{i, j} = 1}
            + \prob{u^s_{i, j} = 0} \cdot 0 \cdot \ex{\left(v^s_{i, j}\right)^2 \mid u^s_{i, j} = 0}
        \right)\\
    =&  h w \cdot \ex{\left(v^s_{i, j}\right)^2 \mid u^s_{i, j} = 1} \cdot \prob{u^s_{i, j} = 1}\\
    =&  h w \cdot \ex{\left(v^s_{i, j}\right)^2 \mid u^s_{i, j} = 1} \cdot \ex{u^s_{i, j}}\\
    =&  \ex{\left(v^s_{i, j}\right)^2 \mid u^s_{i, j} = 1} \cdot \ex[s \sim D]{h w \cdot \ex[(i, j) \sim \uniform{[h] \times [w]}]{u^s_{i, j}}}\\
    =& \ex{\left(v^s_{i, j}\right)^2 \mid u^s_{i, j} = 1} \cdot \ex[s \sim D]{\norm{\mat{U}^s}_0},
\end{align}
where the randomness of $\randomIndices$ is omitted after Step (\ref{step:omit_randomness}) to ease presentation.

\subsection{Proof of \cref{lemma:af_and_gradient_norm}}

Expanding definitions and simple rearrangement give
\begin{align}
    \hessian[\param]
    =&  \nabla_{\param}^2 \ex[(\vec{x}, y) \sim \uniform{\dataset}]{\sampleloss(F_{\allparam}, (X, Y))} 
    = -\ex[(\vec{x}, y) \sim \uniform{\dataset}]{\nabla_{\param}^2 \log F(y \mid \allparam, \vec{x})}, \label{step:expanded_hessian_definition}
\end{align}
where expectation and $\nabla^2$ can be swapped because the expectation is over a finite training set so it is a finite summation.

For any $(\vec{x}, y)$, we have
\begin{align}
    \nabla_{\param}^2 \log F(y \mid \allparam, \vec{x})
    =&  J_{\param} \left( \frac{\nabla_{\param} F(y \mid \allparam, \vec{x})}{F(y \mid \allparam, \vec{x})} \right)\\
    =&  \frac{\left(\nabla_{\param}^2 F(y \mid \allparam, \vec{x})\right) F(y \mid \allparam, \vec{x}) - \left(\nabla_{\param} F(y \mid \allparam, \vec{x})\right)\left(\nabla_{\param} F(y \mid \allparam, \vec{x})\right)^\transpose}{F(y \mid \allparam, \vec{x}) F(y \mid \allparam, \vec{x})}\\
    =&  \frac{\nabla_{\param}^2 F(y \mid \allparam, \vec{x})}{F(y \mid \allparam, \vec{x})} - \left(\frac{\nabla_{\param} F(y \mid \allparam, \vec{x})}{F(y \mid \allparam, \vec{x})}\right)\left(\frac{\nabla_{\param} F(y \mid \allparam, \vec{x})}{F(y \mid \allparam, \vec{x})}\right)^\transpose\\
    =&  \frac{\nabla_{\param}^2 F(y \mid \allparam, \vec{x})}{F(y \mid \allparam, \vec{x})} - \left(\nabla_{\param} \log F(y \mid \allparam, \vec{x})\right)\left(\nabla_{\param} \log F(y \mid \allparam, \vec{x})\right)^\transpose.
\end{align}
Plugging this equality back to \cref{step:expanded_hessian_definition} gives
\begin{align}
    \hessian[\param]
    =&  -\ex[(\vec{x}, y) \sim \uniform{\dataset}]{\frac{\nabla_{\param}^2 F(Y \mid \allparam, \vec{x})}{F(Y \mid \allparam, \vec{x})} - \left(\nabla_{\param} \log F(Y \mid \allparam, \vec{x})\right)\left(\nabla_{\param} \log F(Y \mid \allparam, \vec{x})\right)^\transpose}\\
    =&  \ex[(\vec{x}, y) \sim \uniform{\dataset}]{\left(\nabla_{\param} \log F(Y \mid \allparam, \vec{x})\right)\left(\nabla_{\param} \log F(Y \mid \allparam, \vec{x})\right)^\transpose}
            -\ex[(\vec{x}, y) \sim \uniform{\dataset}]{\frac{\nabla_{\param}^2 F(Y \mid \allparam, \vec{x})}{F(Y \mid \allparam, \vec{x})}}.
\end{align}
and we have
\begin{align}
    \hessian[\param] + \ex[(\vec{x}, y) \sim \uniform{\dataset}]{\frac{\nabla_{\param}^2 F(Y \mid \allparam, \vec{x})}{F(Y \mid \allparam, \vec{x})}} 
    =& \ex[(\vec{x}, y) \sim \uniform{\dataset}]{\left(\nabla_{\param} \log F(Y \mid \allparam, \vec{x})\right)\left(\nabla_{\param} \log F(Y \mid \allparam, \vec{x}).\right)^\transpose}\\
    =& \ex[(\vec{x}, y) \sim \uniform{\dataset}]{\left(\nabla_{\param} \sampleloss(F_{\allparam}, (\vec{x}, y))\right)\left(\nabla_{\param} \sampleloss(F_{\allparam}, (\vec{x}, y))\right)^\transpose}.\\
\end{align}
Taking traces of both sides leads to
\begin{align}
       \trace{\hessian[\param] + \ex[(\vec{x}, y) \sim \uniform{\dataset}]{\frac{\nabla^2_{\param} F(y \mid \vec{x}, \allparam)}{F(y \mid \vec{x}, \allparam)}}}
        =&  \trace{\ex[(\vec{x}, y) \sim \uniform{\dataset}]{\left(\nabla_{\param} \sampleloss(F_{\allparam}, (\vec{x}, y))\right)\left(\nabla_{\param} \sampleloss(F_{\allparam}, (\vec{x}, y))\right)^\transpose}}\\
        =& \ex[(\vec{x}, y) \sim \uniform{\dataset}]{\trace{\left(\nabla_{\param} \sampleloss(F_{\allparam}, (\vec{x}, y))\right)\left(\nabla_{\param} \sampleloss(F_{\allparam}, (\vec{x}, y))\right)^\transpose}}\\
        =& \ex[(\vec{x}, y) \sim \uniform{\dataset}]{\trace{\left(\nabla_{\param} \sampleloss(F_{\allparam}, (\vec{x}, y))\right)^\transpose \left(\nabla_{\param} \sampleloss(F_{\allparam}, (\vec{x}, y))\right)}}\\
       =& \ex[\vec{s} \sim \uniform{\dataset}]{\norm{\nabla_{\param} \sampleloss(F_{\allparam}, \vec{s})}_2^2}.
\end{align}

\subsection{Proof of \cref{lemma:main_full}}\label{proof:main_full}

Instantiating \cref{lemma:af_and_gradient_norm} on $\Kparam^l$ and $\Vparam^l$, we have
\begin{align}
       \trace{\hessian[\Kparam^l] + \ex[(\vec{x}, y) \sim \uniform{\dataset}]{\frac{\nabla^2_{\Kparam^l} F(y \mid \vec{x}, \allparam)}{F(y \mid \vec{x}, \allparam)}}}
       =&   \ex[\vec{s} \sim \uniform{\dataset}]{\norm{\nabla_{\Kparam^l} \ell(F_{\allparam}, \vec{s})}_F^2}.\\
       \trace{\hessian[\Vparam^l] + \ex[(\vec{x}, y) \sim \uniform{\dataset}]{\frac{\nabla^2_{\Vparam^l} F(y \mid \vec{x}, \allparam)}{F(y \mid \vec{x}, \allparam)}}}
       =&   \ex[\vec{s} \sim \uniform{\dataset}]{\norm{\nabla_{\Vparam^l} \ell(F_{\allparam}, \vec{s})}_F^2}.
\end{align}

According to Eq.(\ref{eq:mlp_start})-(\ref{eq:mlp_end}) and the (sub-)differentiability from \cref{assumption:differentiable} and \cref{lemma:differentiable}, the gradient \wrt{} matrix $\mat{K}^l$ is given by
\begin{align}
    \derivatives{\sampleloss(F_{\allparam}, \vec{s})}{\mat{K}^l}
    =&  \left(\derivatives{\ell(F_{\allparam}, \vec{s})}{\mat{P}^l}\right)^\transpose \mat{X}^l 
    =   \left(\derivatives{\ell(F_{\allparam}, \vec{s})}{\mat{A}^l} \hadamard \mat{D}^l\right)^\transpose \mat{X}^l.
\end{align}
Similarly, the gradient \wrt{} $\mat{V}^l$ is given by
\begin{align}
    \derivatives{\sampleloss(F_{\allparam}, \vec{s})}{\mat{V}^l} = \left(\derivatives{\ell(F_{\allparam}, \vec{s})}{\mat{Z}^l}\right)^\transpose \mat{A}^l.
\end{align}
The theorem follows after putting everything together.

\subsection{Proof of \cref{lemma:single_token_full}}\label{proof:single_token_full}

For single-token networks, \cref{eq:main_K} and \cref{eq:main_V} from \cref{lemma:main} can be re-written into
\begin{align}
    \AS[\Kparam^l]
    =& \ex[\vec{s} \sim \uniform{\dataset}]{\norm{\vec{x}^l \times \left(\derivatives{\sampleloss(F_{\allparam}, \vec{s})}{\vec{a}^l} \hadamard \vec{d}^l \right)^\transpose}_F^2},\\
    \AS[\Vparam^l]
    =& \ex[\vec{s} \sim \uniform{\dataset}]{\norm{\vec{a}^l \times \left(\derivatives{\sampleloss(F_{\allparam}, \vec{s})}{\vec{z}^l}\right)^\transpose}_F^2},
\end{align}
where column vectors $\vec{a}^l, \vec{x}^l, \vec{d}^l$ and $\vec{z}^l$ are the transpositions of the only rows in $\mat{A}^l, \mat{X}^l, \mat{D}^l$ and $\mat{Z}^l$, respectively.
Using the fact that for any vector $\vec{u} \in \reals^{d_u}$ and any vector $\vec{v} \in \reals^{d_v}$
\begin{align}
    \norm{\vec{u} \vec{v}^\transpose}_F^2
    =&  \trace{\left(\vec{u} \vec{v}^\transpose\right)^\transpose \left(\vec{u} \vec{v}^\transpose\right)}
    =   \trace{\vec{v}^\transpose \vec{v} \vec{u}^\transpose \vec{u}} = \norm{\vec{u}}_2^2 \norm{\vec{v}}_2^2,
\end{align}
we obtain
\begin{align}
    \AS[\Kparam^l]
    =& \ex[\vec{s} \sim \uniform{\dataset}]{\norm{\derivatives{\sampleloss(F_{\allparam}, \vec{s})}{\vec{a}^l} \hadamard \vec{d}^l}_2^2 \norm{\vec{x}^l}_2^2},\\
    \AS[\Vparam^l]
    =& \ex[\vec{s} \sim \uniform{\dataset}]{\norm{\derivatives{\sampleloss(F_{\allparam}, \vec{s})}{\vec{z}^l}}_2^2 \norm{\vec{a}^l}_2^2}.
\end{align}

Replacing $\mat{A}^l, \mat{X}^l, \mat{D}^l$ and $\mat{Z}^l$ back leads to the lemma.

\subsection{Proof of \cref{theorem:single_token_full}}\label{proof:single_token_theorem_full}

Based on \cref{eq:single_token_K} and \cref{eq:single_token_V}, we have
\begin{align}
    \AS[\Kparam^l]
    =& \ex[\vec{s} \sim \uniform{\dataset}]{\norm{\derivatives{\sampleloss(F_{\allparam}, \vec{s})}{\mat{A}^l} \hadamard \mat{D}^l}_F^2 \cdot \norm{\mat{X}^l}_F^2}
    = \ex[\vec{s} \sim \uniform{\dataset}]{\norm{\derivatives{\sampleloss(F_{\allparam}, \vec{s})}{\mat{A}^l} \hadamard \mat{D}^l \hadamard \mat{D}^l_{>0}}_F^2 \cdot \norm{\mat{X}^l}_F^2},\\
    \AS[\Vparam^l] 
    =& \ex[\vec{s} \sim \uniform{\dataset}]{\norm{\derivatives{\sampleloss(F_{\allparam}, \vec{s})}{\mat{Z}^l}}_F^2 \norm{\mat{A}^l}_F^2},
\end{align}
where $\mat{D}^l_{>0} = [\indic{d^l_{i, j}} > 0]_{i, j}$ is $0$-$1$ valued.
Applying \cref{lemma:L0_and_L2}, we obtain
\begin{align}
    \AS[\Kparam^l]
    =& \ex{\norm{\mat{X}^l}_F^2 \left(\derivatives{\sampleloss(F_{\allparam}, \vec{s})}{d^l_{i, j}}\right)^2 \left(d^{l}_{i, j}\right)^2 \mid d^l_{i, j} > 0} \cdot \ex[\vec{s} \sim \uniform{\dataset}]{\norm{\mat{D}^l}_0},
\end{align}
where the first expectation is taken over $\vec{s} \sim \uniform{\dataset}, (i, j) \in \uniform{[d] \times [k]}$.

Since for $\relu$ network, $d^l_{i, j} = 1 \iff p^l_{i, j} > 0 \iff a^l_{i, j} > 0$, one has $\norm{\mat{D}^l}_0 = \norm{\mat{A}^l}_0$. Plugging this coincidence obtains the first chaining of the equalities.

Furthermore, when $\mat{X}^l = \begin{bmatrix} \vec{x}^\transpose \end{bmatrix}$ is LayerNorm-ed, let $\mat{U}^l = \begin{bmatrix} \vec{u}^\transpose \end{bmatrix}$ be the input of that application of LayerNorm, then 
\begin{align}
    \vec{x} = \vec{\gamma} \hadamard \frac{\vec{u} - \bar{u} \vec{1}_d}{\sqrt{\frac{1}{d}\norm{\vec{u} - \bar{u} \vec{1}_d}_2^2 + \epsilon_{\textnormal{LayerNorm}}}} + \vec{\beta} \label{def:layernorm},
\end{align}
where $\bar{u} = \frac{1}{d} \sum_{i=1}^d u_i$, $\vec{\gamma}$ and $\vec{\beta}$ are affine parameters.
By assumption that affine parameters and $\epsilon_{\textnormal{LayerNorm}}$ are turned off, \ie{} $\vec{\gamma} = \vec{1}, \vec{\beta} = \vec{0}, \epsilon_{\textnormal{LayerNorm}} = 0$, we have
\begin{align}
    \vec{x} = \frac{\vec{u} - \bar{u} \vec{1}_d}{\sqrt{\frac{1}{d}\norm{\vec{u} - \bar{u} \vec{1}_d}_2^2}},
\end{align}
and with $\vec{r} \defeq \vec{u} - \bar{u} \vec{1}_d$
\begin{align}
    \vec{x} = \frac{\vec{r}}{\sqrt{\frac{1}{d}\norm{\vec{r}}_2^2}},
\end{align}
whose squared $L_2$ norm is
\begin{align}
    \norm{\mat{X}}_F^2 = \norm{\vec{x}}_2^2 = d.
\end{align}
The equality under LayerNorm follows.

\subsection{Proof of \cref{theorem:multi_token_full}}\label{proof:multi_token_full}

\newcommand{\derivativewrttilde}[2]{\derivatives{\sampleloss(\tilde{F}_{\tildeallparam}, (\vec{x}, y))}{\tilde{#1}^{#2}}}
\newcommand{\derivativewrt}[2]{\derivatives{\sampleloss(F_{\allparam}, (\vec{x}, y))}{#1^{#2}}}

On any sample $(\vec{x}, y) \in \reals^{d_{\textnormal{input}}} \times \mathcal{Y}$, one can easily verify by induction that all of the following claims are true
\begin{itemize}
    \item $(\tilde{\mat{X}}^l, \tilde{\mat{Z}}^l, \tilde{\mat{S}}^l) = (\mat{X}^l, \mat{Z}^l, \mat{S}^l), \tilde{\mat{S}}^{l, i} = \mat{S}^{l-1}$;
    \item $\tilde{\vec{x}}^{l, i} = \vec{x}^l_i, \tilde{\vec{z}}^{l, i} = \vec{z}^l_i$;
    \item $\tilde{\mat{A}}^{l, i} = [\left(\vec{a}^{l}_i\right)^\transpose], \tilde{\mat{P}}^{l, i} = [\left(\vec{p}^{l}_i\right)^\transpose]$;
\end{itemize}
and during backward propagation on differentiable samples for $F_{\theta}$,
\begin{itemize}
    \item $\left(\derivativewrttilde{\mat{X}}{l}, \derivativewrttilde{\mat{Z}}{l}, \derivativewrttilde{\mat{S}}{l}\right) = \left(\derivativewrt{\mat{X}}{l}, \derivativewrt{\mat{Z}}{l}, \derivativewrt{\mat{S}}{l}\right)$;
    \item $\derivativewrttilde{\vec{z}}{l, i} = \derivativewrt{\vec{z}_i}{l}, \derivativewrttilde{\vec{a}}{l, i} = \derivativewrt{\vec{a}_i}{l}$;
    \item $\vec{d}^l_i = \tilde{\vec{d}}^{l, i}$.
\end{itemize}

We first relate the noise-related terms to augmented flatness.
By \cref{assumption:differentiable} for $F_{\allparam}$ at all training samples, it also holds for $\tilde{F}_{\tildeallparam}$ at all training samples.
Since $\sampleloss(\tilde{F}_{\tildeallparam}, \vec{s})$ is twice continuously differentiable \wrt{} $\tildeallparam$ for any training sample $\vec{s} \in \dataset$, we can approximate the loss after the perturbation two the second order by Taylor expansion
\begin{align}
    \sampleloss(\tilde{F}_{\tildeparam[l, K]'}, \vec{s}) - \sampleloss(\tilde{F}_{\tildeallparam}, \vec{s})
    =& \left(\nabla_{\tildeKparam^l} \sampleloss(\tilde{F}_{\tildeallparam}, \vec{s})\right)^\transpose \vec{\delta} + \vec{\delta}^\transpose \hessian[\tildeKparam^l](\vec{s}) \vec{\delta} + r_{\vec{s}}(\vec{\delta}),
\end{align}
where $\vec{\delta}$ corresponds to the perturbation by adding noises $\mat{E} = (\mat{E}^{l, 1}_K, \cdots, \mat{E}^{l, k}_K) \in (\reals^{n \times d})^k$, entries in which independently follow the uniform distribution on $[-\sqrt{3} \sigma, +\sqrt{3} \sigma]$, and $r_{\vec{s}}(\vec{\delta}) = o(\norm{\vec{\delta}}_2^2)$.
For a fixed sample, taking expectation over the noises, we have 
\begin{align}
    \ex[\mat{E}]{\sampleloss(\tilde{F}_{\tildeparam[l, K]'}, \vec{s}) - \sampleloss(\tilde{F}_{\tildeallparam}, \vec{s})}
    =& \ex[\mat{E}]{\left(\nabla_{\tildeKparam} \sampleloss(\tilde{F}_{\tildeallparam}, \vec{s})\right)^\transpose \vec{\delta} + \vec{\delta}^\transpose \hessian[\tildeKparam](\vec{s}) \vec{\delta} + r_{\vec{s}}(\vec{\delta})}\\
    =& \ex[\mat{E}]{\left(\nabla_{\tildeKparam} \sampleloss(\tilde{F}_{\tildeallparam}, \vec{s})\right)^\transpose \vec{\delta}} + \ex[\mat{E}]{\vec{\delta}^\transpose \hessian[\tildeKparam](\vec{s}) \vec{\delta}} + \ex[\mat{E}]{r_{\vec{s}}(\vec{\delta})}.
\end{align}
The first term is $0$ since $\vec{\delta}$ is symmetric. The second term can be simplified to 
\begin{align}
    \ex[\mat{E}]{\vec{\delta}^\transpose \hessian[\tildeKparam](\vec{s}) \vec{\delta}} 
    =& \trace{\ex[\mat{E}]{\vec{\delta}^\transpose \hessian[\tildeKparam](\vec{s}) \vec{\delta}}}
    =  \ex[\mat{E}]{\trace{\vec{\delta}^\transpose \hessian[\tildeKparam](\vec{s}) \vec{\delta}}}
    =  \ex[\mat{E}]{\trace{\hessian[\tildeKparam](\vec{s}) \vec{\delta} \vec{\delta}^\transpose}}\\
    =& \trace{\ex[\mat{E}]{\hessian[\tildeKparam](\vec{s}) \vec{\delta} \vec{\delta}^\transpose}}
    =  \trace{\hessian[\tildeKparam](\vec{s}) \times \ex[\mat{E}]{\vec{\delta} \vec{\delta}^\transpose}}
    =  \trace{\hessian[\tildeKparam](\vec{s}) \times \sigma^2 \mat{I}} \\
    =& \sigma^2 \trace{\hessian[\tildeKparam](\vec{s})}.
\end{align}
Let $r^{E}_{\vec{s}}(\sigma)$ to denote the third expectation, then 
\begin{align}
    \abs{r^{E}_{\vec{s}}(\sigma)}
    \le&    \sup_{\vec{\delta} \in [-\sqrt{3}\sigma, +\sqrt{3}\sigma]^{2 n d k l}} \abs{r_{\vec{s}}(\vec{\delta})} = o(\sigma^2 n k d).
\end{align}
So we have 
\begin{align}
    \ex[\mat{E}]{\sampleloss(\tilde{F}_{\tildeparam[l, K]'}, \vec{s}) - \sampleloss(\tilde{F}_{\tildeallparam}, \vec{s})}
    = \sigma^2 \trace{\hessian[\tildeKVparam](\vec{s})} + r^{E}_{\vec{s}}(\sigma).
\end{align}

Taking expectation over samples leads to 
\begin{align}
    \ex{\sampleloss(\tilde{F}_{\tildeparam[l, K]'}, \vec{s}) - \sampleloss(\tilde{F}_{\tildeallparam}, \vec{s})}
    =&  \sigma^2 \trace{\hessian[\tildeKparam]} + \ex[\vec{s}]{r^E_{\vec{s}}(\sigma)}\\
    =&  \sigma^2 \trace{\hessian[\tildeKparam]} + o(\sigma^2 n d k l),
\end{align}
where the equality between the last terms is because there are only a finite number of training samples so we still have $\sup_{\vec{s} \in \dataset} |r^E_{\vec{s}}(\sigma)| = o(\sigma^2 n d k l)$.
To sum up, there is
\begin{align}
    \frac{1}{\sigma^2} \ex{\sampleloss(\tilde{F}_{\tildeparam[l, K]'}, \vec{s}) - \sampleloss(\tilde{F}_{\tildeallparam}, \vec{s})} = \trace{\hessian[\tildeKparam]} + o(1) \cdot n d k. \label{step:noise_and_hessian}
\end{align}
In a similar way one can prove
\begin{align}
    \frac{1}{\sigma^2} \ex{\frac{\tilde{F}(y \mid \tildeparam[l, K]', \vec{x}) - \tilde{F}(y \mid \tildeallparam, \vec{x})}{F(y \mid \allparam, \vec{x})}} = \trace{\ex{\frac{\nabla^2_{\tildeKparam} \tilde{F}(y \mid \tildeallparam, \vec{x})}{F(y \mid \allparam, \vec{x})}}} + o(1) \cdot n d k, \label{step:noise_and_hessian_augmentation}
\end{align}
by approximating $\tilde{F}_{\tildeallparam}$ to the second order. Combining \cref{step:noise_and_hessian} and \cref{step:noise_and_hessian_augmentation}, we have 
\begin{align}
        \frac{1}{\sigma^2} \ex{\hat{\loss}(\tilde{F}_{\tildeparam[l, K]'}) - \hat{\loss}(\tilde{F}_{\tildeallparam})}  + \frac{1}{\sigma^2} \ex{\frac{\tilde{F}(y \mid \tildeparam[l, K]', \vec{x}) - \tilde{F}(y \mid \tildeallparam, \vec{x})}{F(y \mid \allparam, \vec{x})}}  - o(1) \cdot n d k
        =   \tildeAS[\tildeKparam^l].
\end{align}
Similar results can be derived for value matrices:
\begin{align}
        \frac{1}{\sigma^2} \ex{\hat{\loss}(\tilde{F}_{\tildeparam[l, V]'}) - \hat{\loss}(\tilde{F}_{\tildeallparam})}  + \frac{1}{\sigma^2} \ex{\frac{\tilde{F}(y \mid \tildeparam[l, V]', \vec{x}) - \tilde{F}(y \mid \tildeallparam, \vec{x})}{F(y \mid \allparam, \vec{x})}}  - o(1) \cdot n d k
        =  \tildeAS[\tildeVparam^l].
\end{align}

Then we repeat arguments used for single-token cases on the copied model and relate the copied model back to the multi-token model.
By how $\tilde{F}_{\tildeallparam}$ is constructed from $F_{\allparam}$, $\tilde{F}_{\tildeallparam}$ is twice continuously differentiable \wrt{} its all parameters and hidden features if $F_{\theta}$ is. So under \cref{assumption:differentiable} for $F_{\theta}$ and the training set, \cref{assumption:differentiable} also holds for $\tilde{F}_{\tildeallparam}$ on the training set. Since $\tilde{F}_{\tildeallparam}$ is a single-token network, by \cref{theorem:single_token_full} we have
\begin{align}
        \tildeAS[\tildeKparam^l]
        =&  \ex{\norm{\nabla_{\tildeKparam^l} \sampleloss(\tilde{F}_{\tildeallparam}, \vec{s})}_2^2}
        =   \sum_{i \in [k]} \ex{\norm{\nabla_{\tildeKparam^{l, i}} \sampleloss(\tilde{F}_{\tildeallparam}, \vec{s})}_2^2}\\
        =&   \sum_{i \in [k]} \ex[\vec{s}, j]{\norm{\tilde{\vec{x}}^{l, i}}_2^2\left(\derivatives{\sampleloss(\tilde{F}_{\tildeallparam}, \vec{s})}{\tilde{a}^{l, i}_{j}}\right)^2 \left(\tilde{d}^{l, i}_{j}\right)^2 \mid \tilde{d}^{l, i}_{j} > 0} \cdot \ex[\vec{s}]{\norm{\tilde{\vec{d}}^{l, i}}_0},\\
        \tildeAS[\tildeVparam^l]
        =&  \ex{\norm{\nabla_{\tildeVparam} \sampleloss(\tilde{F}_{\tildeallparam}, \vec{s})}_2^2}
        =   \sum_{i \in [k]}  \ex{\norm{\nabla_{\tildeVparam^{l, i}} \sampleloss(\tilde{F}_{\tildeallparam}, \vec{s})}_2^2}\\
        =&    \sum_{i \in [k]} \ex[\vec{s} \sim \uniform{\dataset}]{\norm{\derivatives{\sampleloss(\tilde{F}_{\tildeallparam}, \vec{s})}{\tilde{\vec{z}}^{l, i}}}_F^2 \norm{\tilde{\vec{a}}^{l, i}}_F^2},
\end{align}
where the first expectations at \rhs{} of both equations are over $\vec{s} \sim \uniform{\dataset}$ and $j \sim \uniform{[n]}$.
By the correspondences during forward and backward propagation between $F_{\allparam}$ and $\tilde{F}_{\tildeallparam}$, we have 
\begin{align}
        \tildeAS[\tildeKparam^l]
        =&  \ex{\norm{\nabla_{\tildeKparam^l} \sampleloss(\tilde{F}_{\tildeallparam}, \vec{s})}_2^2}
        =   \sum_{i \in [k]} \ex[\vec{s}, j]{\norm{\vec{x}^l_i}_2^2 \left(\derivatives{\sampleloss(F_{\allparam}, \vec{s})}{a^{l}_{i, j}}\right)^2 \left(d^{l}_{i, j}\right)^2 \mid d^{l}_{i,j} > 0} \cdot \ex[\vec{s}]{\norm{\vec{d}^{l}_i}_0} ,\\
        \tildeAS[\tildeVparam^l]
        =&  \ex{\norm{\nabla_{\tildeVparam^l} \sampleloss(\tilde{F}_{\tildeallparam}, \vec{s})}_2^2}
        =   \sum_{i \in [k]} \ex[\vec{s} \sim \uniform{\dataset}]{\norm{\derivatives{\sampleloss(F_{\allparam}, \vec{s})}{\vec{z}^{l}_i}}_F^2 \norm{\vec{a}^l_{i}}_F^2},\\
\end{align}
Rearrange the first term:
\begin{align}
    &\sum_{i \in [k]} \ex[\vec{s}, j]{\norm{\vec{x}^l_i}_2^2 \left(\derivatives{\sampleloss(F_{\allparam}, \vec{s})}{a^{l}_{i, j}}\right)^2  \left(d^{l}_{i, j}\right)^2 \mid d^{l}_{i,j} > 0} \cdot \ex[\vec{s}]{\norm{\vec{d}^{l}_i}_0}\\ 
    =&  \sum_{i \in [k]} \ex[\vec{s}, j]{\norm{\vec{x}^l_i}_2^2 \left(\derivatives{\sampleloss(F_{\allparam}, \vec{s})}{a^{l}_{i, j}}\right)^2 \left(d^{l}_{i, j}\right)^2 \mid d^{l}_{i,j} > 0} \cdot n \cdot \prob[][\vec{s}, j]{d^l_{i, j} > 0}\\
    =&  n \sum_{i \in [k]} \ex[\vec{s}, j]{\norm{\vec{x}^l_i}_2^2 \left(\derivatives{\sampleloss(F_{\allparam}, \vec{s})}{a^{l}_{i, j}}\right)^2 \left(d^{l}_{i, j}\right)^2 \cdot \indic{d^l_{i, j} > 0}}\\
    =&  n k \ex[\vec{s}, (i, j)]{\norm{\vec{x}^l_i}_2^2 \left(\derivatives{\sampleloss(F_{\allparam}, \vec{s})}{a^{l}_{i, j}}\right)^2 \left(d^{l}_{i, j}\right)^2 \cdot \indic{d^l_{i, j} > 0}}\\
    =&  n k \ex[\vec{s}, (i, j)]{\norm{\vec{x}^l_i}_2^2 \left(\derivatives{\sampleloss(F_{\allparam}, \vec{s})}{a^{l}_{i, j}}\right)^2 \left(d^{l}_{i, j}\right)^2 \mid d^l_{i, j} > 0} \cdot \prob[][\vec{s}, (i, j)]{d^{l}_{i, j} > 0}\\
    =&  \ex[\vec{s}, (i, j)]{\norm{\vec{x}^l_i}_2^2 \left(\derivatives{\sampleloss(F_{\allparam}, \vec{s})}{a^{l}_{i, j}}\right)^2 \left(d^{l}_{i, j}\right)^2 \mid d^l_{i, j} > 0} \cdot \ex[\vec{s}]{\norm{\mat{D}^l}_0}.\\
\end{align}

Putting everything together leads to the theorem for general cases. 
When $\relu$ is used, plugging $\norm{\mat{A}^l}_0 = \norm{\mat{D}^l}_0$ leads to the part for $\relu$ cases.
When LayerNorm is used, plugging $\norm{\vec{x}^l_i}_2^2 = d$ leads to the part for LayerNorm cases. \section{More Results}\label{appendix:jrelu_results}

We extend \cref{theorem:multi_token} for a more general class of activation functions $\sigma$ that features in \cref{def:jumping}.
\begin{definition}\label{def:jumping}
    $\sigma: \reals \to \reals$ is called jumping if the following are true:
    \begin{itemize}
        \item $\sigma(x) = 0$ for $x < 0$, $\sigma(0) = 0$, $\sigma(x) > 0$ for $x > 0$; and
        \item $\sigma$ is continuous and twice continuously differentiable in $(-\infty, 0)$ and $(0, +\infty)$; and
\item $\abs{\sigma'(x)} \ge 1$ when $x > 0$ and $\sigma'(x) = 0$ when $x < 0$.
    \end{itemize}
\end{definition}
$\jrelu$ fits the definition so it is covered by \cref{corollary:jumping}.

\begin{corollary}\label{corollary:jumping}
    Under the same condition of \cref{theorem:multi_token_full} except that $\sigma$ is jumping, then for general multi-token $\sigma$-activated $F_{\theta}$ we have 
    \begin{align}
        \norm{\mat{A}^l}_0 = \norm{\mat{D}^l}_0,
    \end{align}
    and
    \begin{align}
        &\frac{1}{\sigma^2} \ex{\hat{\loss}(\tilde{F}_{\tildeparam[l, K]'}) - \hat{\loss}(\tilde{F}_{\tildeallparam})}  + \frac{1}{\sigma^2} \ex{\frac{\tilde{F}(y \mid \tildeparam[l, K]', \vec{x}) - \tilde{F}(y \mid \tildeallparam, \vec{x})}{F(y \mid \allparam, \vec{x})}}  - o(1) \cdot n d k
        =   \tildeAS[\tildeKparam^l]
        =   \ex{\norm{\nabla_{\tildeKVparam} \sampleloss(\tilde{F}_{\tildeallparam}, \vec{s})}_2^2}\\
        =& \ex{\norm{\vec{x}^l_i}_2^2 \left(\derivatives{\sampleloss}{a^l_{i, j}}\right)^2 \left(d^l_{i, j}\right)^2 \mid a^l_{i, j} > 0} \cdot \ex[\vec{s} \sim \uniform{\dataset}]{\norm{\mat{D}^l}_0}
        \ge \ex{\norm{\vec{x}^l_i}_2^2 \left(\derivatives{\sampleloss}{a^l_{i, j}}\right)^2 \mid a^l_{i, j} > 0} \cdot \ex[\vec{s} \sim \uniform{\dataset}]{\norm{\mat{D}^l}_0}.
    \end{align}
    If further all inputs to MLP blocks are LayerNorm-ed with affine parameters and $\epsilon_{\textnormal{LayerNorm}}$ are turned off, the above inequalities can be further interleaved by 
    \begin{align}
        \tildeAS[\Kparam^l]
        =&   d \sum_{l \in [L]} \ex{\left(\derivatives{\sampleloss}{a^l_{i, j}}\right)^2 \left(d^{l}_{i, j}\right)^2 \mid a^l_{i, j} > 0} \cdot \ex[\vec{s} \sim \uniform{\dataset}]{\norm{\mat{D}^l}_0}\\
        \ge& d \sum_{l \in [L]} \ex{\left(\derivatives{\sampleloss}{a^l_{i, j}}\right)^2 \mid a^l_{i, j} > 0} \cdot \ex[\vec{s} \sim \uniform{\dataset}]{\norm{\mat{D}^l}_0}
    \end{align}
\end{corollary}

\cref{corollary:jumping} indicates that activation sparsities together still lower-bound augmented flatness.

\section{Experimental Details}\label{appendix:details}

\begin{table*}
    \centering
    \caption{\minorrevision{Hyperparameters for pretraining in vision tasks.}}
    \minorrevision{
    \begin{tabular}{lccccccccc}
         \toprule
                 Data & Arch. & \makecell{Model\\size}&    \makecell{Learning \\rate (max)}    &    \makecell{Weight \\decay}    &    \makecell{Batch \\size}    &    \makecell{Grad. \\clip}    &    \makecell{Label smoothing\\ or CLIP Temperature}    &    Augmentation    & Epochs\\
         \midrule
         ImageNet-1K    &    ViT    &    Base/16    &    $3.0 \times 10^{-3}$    &    $0.300$    &    $2048$    &    $1.0$    &    $0.11$    &    \makecell{mixup: $\alpha=0.2$\\ cutmix: $\alpha=1.0$\\auto aug: \texttt{ra}\\rand erase: $p=0.25$}    &     300\\
         \midrule
         ImageNet-1K    &    \makecell{Swin\\Transformer}    &    Base/16    &    $1.0 \times 10^{-3}$    &    $0.050$    &    $1024$    &    $5.0$    &    $0.10$    &    \makecell{mixup: $\alpha=0.8$\\ cutmix: $\alpha=1.0$\\auto aug: \texttt{ta\_wide}\\rand erase: $p=0.25$}    &    300\\
         \midrule
         Places365-Std    &    ViT    &    Large/16    &    $1.0 \times 10^{-3}$    &    $0.005$    &    $2048$    &    $1.0$    &    $0.10$    &    \makecell{mixup: $\alpha=0.2$\\ cutmix: $\alpha=1.0$\\auto aug: \texttt{ta\_wide}\\rand erase: $p=0.10$}    &    300\\
         \midrule
         LAION-40M    &    ViT    &    Base/16    &    $3.0 \times 10^{-3}$    &  $0.300$    &    $24 \times 360$    &    $1.0$    &    $T_{\textnormal{init}} = 2.66$    &     auto aug: \texttt{ra}    &    16\\
         \bottomrule
    \end{tabular}
    }
    \label{table:vision_hyperparameters}
\end{table*}

\minorrevision{
For vision tasks, we use ImageNet-1K, Places365 and an 1/10 subset of LAION-400M of different scales and diversity.
We use a subset LAION-40M of LAION-400M containing approximately 43M images, due to limited computational resources and partially inaccessible online resources.
LAION datasets provide text captions of images instead of explicit labels. Therefore, we train ViT with the CLIP loss \cite{clip}. We use the pretrained Base/16-sized text encoder (\texttt{huggingface.co/openai/clip-vit-base-patch16}) to generate text embeddings of each image caption. 
Note that the CLIP loss is a contrastive loss that samples negative samples in the batch, and the batch size affects the performance of CLIP training. However, due to limited GPU, we use a smaller batch size $24 \times 360 = 8640$, where $24$ is the number of gradient accumulation and $360$ is the sub-batch size of each gradient accumulation step. We pre-compute all text embeddings, and so that we can select text negative samples in the whole $8640$-sized batch with little costs. To further avoid gradient checkpointing and complicated inter-GPU gradient transferring, we sample visual negative samples in the $360$ images of one gradient accumulation step within one GPU.
We parameterize temperature with $[\textrm{logit}] = \cos(f_{\textnormal{visual}}, f_{\textnormal{text}}) \cdot e^{T}$ as in \cite{clip}, with temperature initialization $T_{\textnormal{init}} = 2.66$ and temperature bound $T \in [-100, +100]$.

All ViTs or SwinTransformers uses images of $224 \times 224$ with patch size 16.
The training recipe is adapted from \cite{pytorch_recipe}. Specifically, the optimizer is AdamW with $(\beta_1, \beta_2)=(0.9, 0.999)$, dropout $0.0$, gradient clipping, linear warmup followed by cosine decay learning rate scheduling. Detailed hyperparameters are listed in \cref{table:vision_hyperparameters}. Model EMA is \emph{removed} to ease experiment tracking. AMP with \texttt{fp16} is used to accelerate training. Images are resized to $256 \times 256$ to save storage. Logging happens every 100 steps.
During finetuning, the differences include modifications for sparsity described in \cref{sec:algorithm} and \cref{sec:p_experiments}, LoRA of rank $192$, and epoch number reduced to $15$, among which LayerNorm uplifting takes $5$ epochs.
Among vision tasks, the learning on Places365 exhibits strong overffiting at the end of training. Since focus on sparsity instead of performance or overfitting, we report the testing performance and testing sparsity at the step when the testing Acc@1 is the highest.  
}

During training T5-Base from scratch on C4, the recipe is adapted from \cite{t5_recipe} and \cite{observation}. Specifically, the optimizer is AdamW with learning rate $0.01$, $(\beta_1, \beta_2)=(0.9, 0.999)$, weight decay $0.001$, dropout $0.0$ and gradient clipping of global norm $1.0$. The learning rate is scheduled with inverse square root scheduling with a linear warmup of $10,000$ steps, while the full training lasts for $100,000$ steps. Batch size is $256$. To avoid downloading the entire data set, we use the streaming functionality of Huggingface \texttt{datasets} to download the first $25,600,000$ training samples without shuffling and the full validation split. Recipe from \cite{t5_recipe} includes data preparation that concatenates all raw sentences and re-splits them into samples. This step consumes too much storage so we confine the concatenation within a batch and empties are filled with padding tokens that are ignored during computing the loss. This compromise reduces the number of usable samples during training, but there are still approximately at least $50$ non-empty samples for every 64 samples. The generation of spans is left unchanged, where 15\% of the input is corrupted without any mixture. The encoder receives $512$ tokens while $114$ tokens are used on the decoder side. Logging happens every 25 steps. When training modified T5, the rented GPUs expired at step $95,000$. Nevertheless, the effectiveness of modification is already obvious at that step.
During finetuning, LoRA of rank $192$ considering achieving better sparsity may involve largely modifying weights, and modifications described in \cref{sec:algorithm} and \cref{sec:p_experiments} are equipped. The training step reduces to $10,000$, among which learning rate warmup takes $1,000$ steps and LayerNorm uplifting and activation function mixing last for $3,000$ steps. Logging happens every $50$ step during finetuning T5.

Since a lot of statistics are computed which slows the training, for ViT the logging (including training losses, sparsity, etc.) happens after per-epoch evaluation or every 100 steps in training. The encoder-decoder architecture and larger token number of T5 cost more memory use, so half of a batch is used to compute the statistics. As compensation logging happens every 25 steps for T5.

When we trained ViT on ImageNet-1K in \cref{sec:coefficients}, the checkpoint frequency was too coarse to capture the evolution during initial epochs, where the ratio decreases drastically. As a result, \cref{fig:initial_epochs_imagenet1k} is obtained by retraining for 10 epochs. 
Since there are a lot of checkpoints to process, the estimation is conducted on a random subset of ImageNet-1K's training set with $2048 \times 128$ samples (approximately 1/5 of the training set) while the whole training set is used for CIFAR-10.

In experiments of \cref{sec:stability}, ViT-Base is trained on CIFAR-10 using Adam. Hyperparameters include a learning rate of $10^{-4}$, a batch size of $64$, no dropout, weight decay, learning rate decay, or gradient clipping. Random cropping and random horizontal flipping are used as data augmentation. The training lasts for 100 epochs with the first 5 epochs of linear warmup. Automatic mixed precision (AMP) is not used. Shifted $\weird$ is used as activation functions, replacing all $\relu$ functions. $\dbmlp$s are used, but the experiments of \cref{sec:stability} were conducted before we realize the contradiction between zeroth biases and LayerNorm, so zeroth bias modules are added but \cref{algo:zeroth_bias} is not executed. LayerNorm affine parameters are \emph{not} turned off and \emph{not} lower-bounded. 
The ViT implementation is based on that of Torch Vision. Logging happens every 10 steps to plot \cref{fig:exp_stability}.

In experiments of \cref{sec:ablation}, in contrast to those in \cref{sec:p_experiments} we used grad clipping norm up to $10$. We also recorded other statistics using torch hooks in different layers. The last layer contains a lot of extreme values so the hook in the last layer was turned off and sparsity in it was not available. \section{Full Experimental Results}\label{appendix:more_experiments}

\minorrevision{
Figs. \ref{figure:productive_vit_full}-\ref{figure:productive_t5_full} display the detailed evolution of sparsity during pretraining.
Figs. \ref{fig:detail_start}-\ref{fig:detail_end} display the detailed evolution of augmented flatness and coefficients of activation norms during ViT-Base pretraining on CIFAR-10 and ImageNet-1K.
}

\begin{figure*}
    \centering
    \resetHeight{}
    \subfloat[Training, modified. \label{figure:productive_vit_sparsified_training}]{
        \myincludegraphics[width=0.24\linewidth]{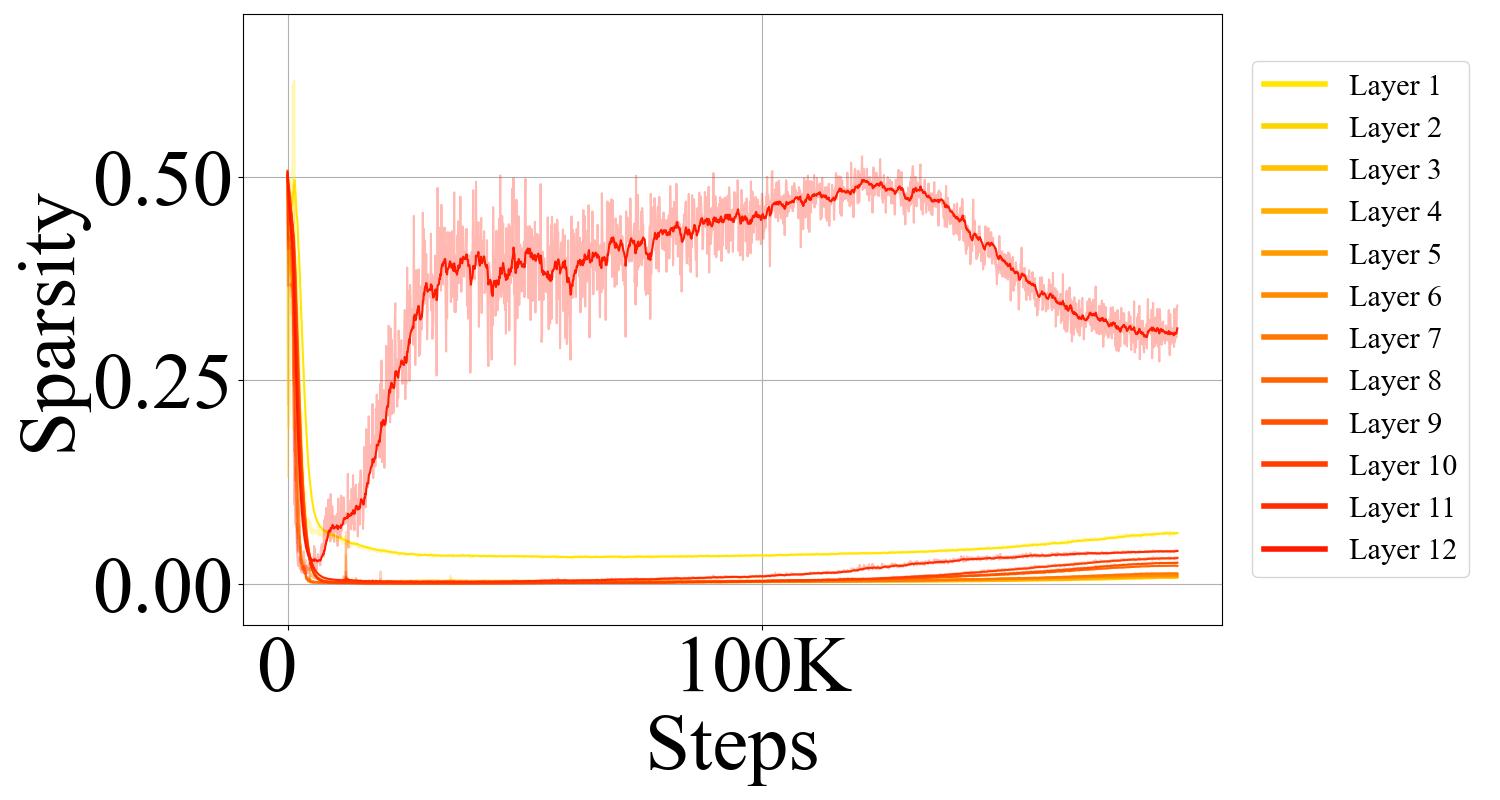}
}
    \subfloat[Training, vanilla. \label{figure:productive_vit_vanilla_training}]{
        \myincludegraphics[width=0.24\linewidth]{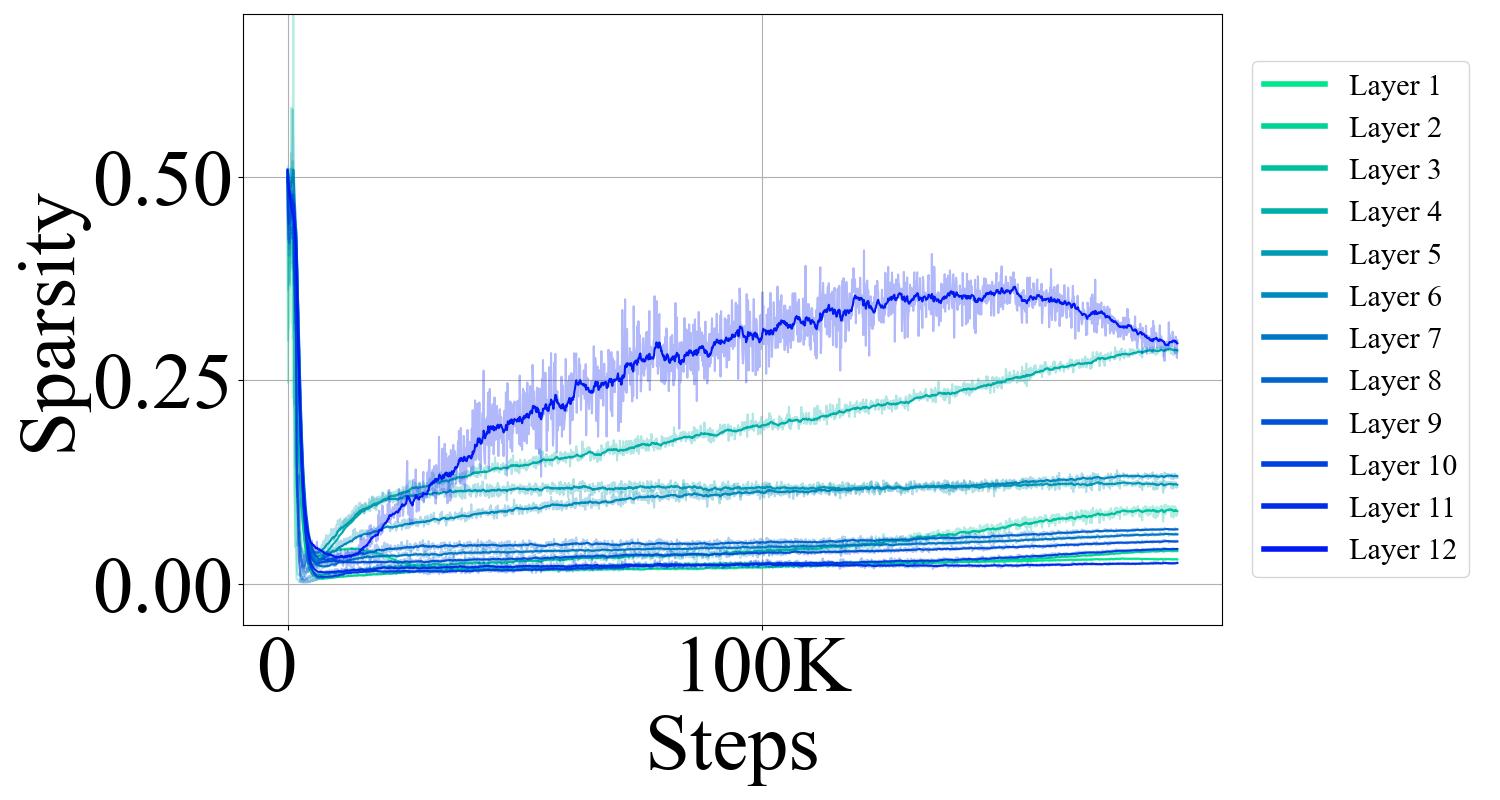}
}
    \subfloat[Testing, modified. \label{figure:productive_vit_sparsified_testing}]{
        \myincludegraphics[width=0.24\linewidth]{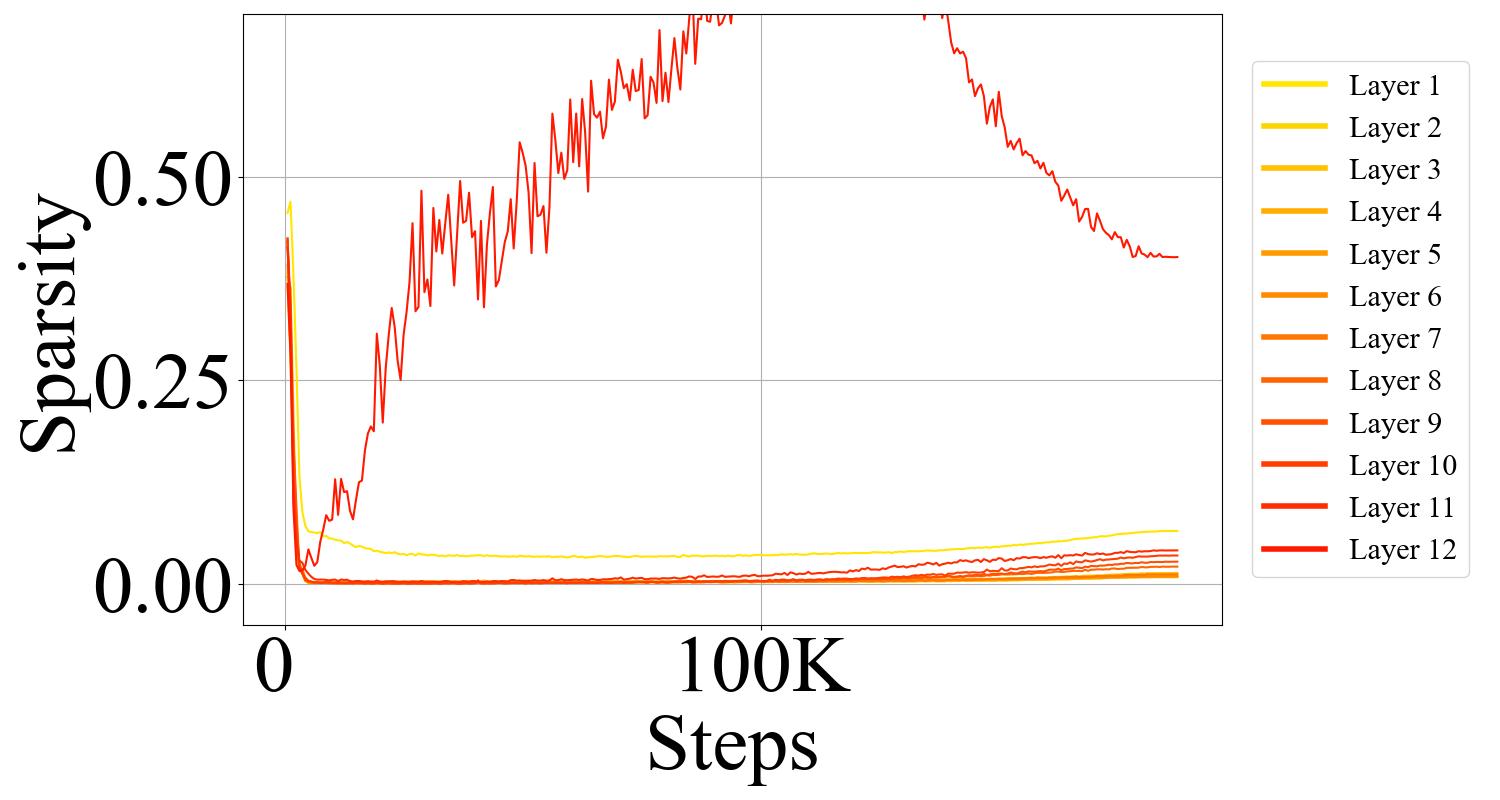}
}
    \subfloat[Testing, vanilla. \label{figure:productive_vit_vanilla_testing}]{
        \myincludegraphics[width=0.24\linewidth]{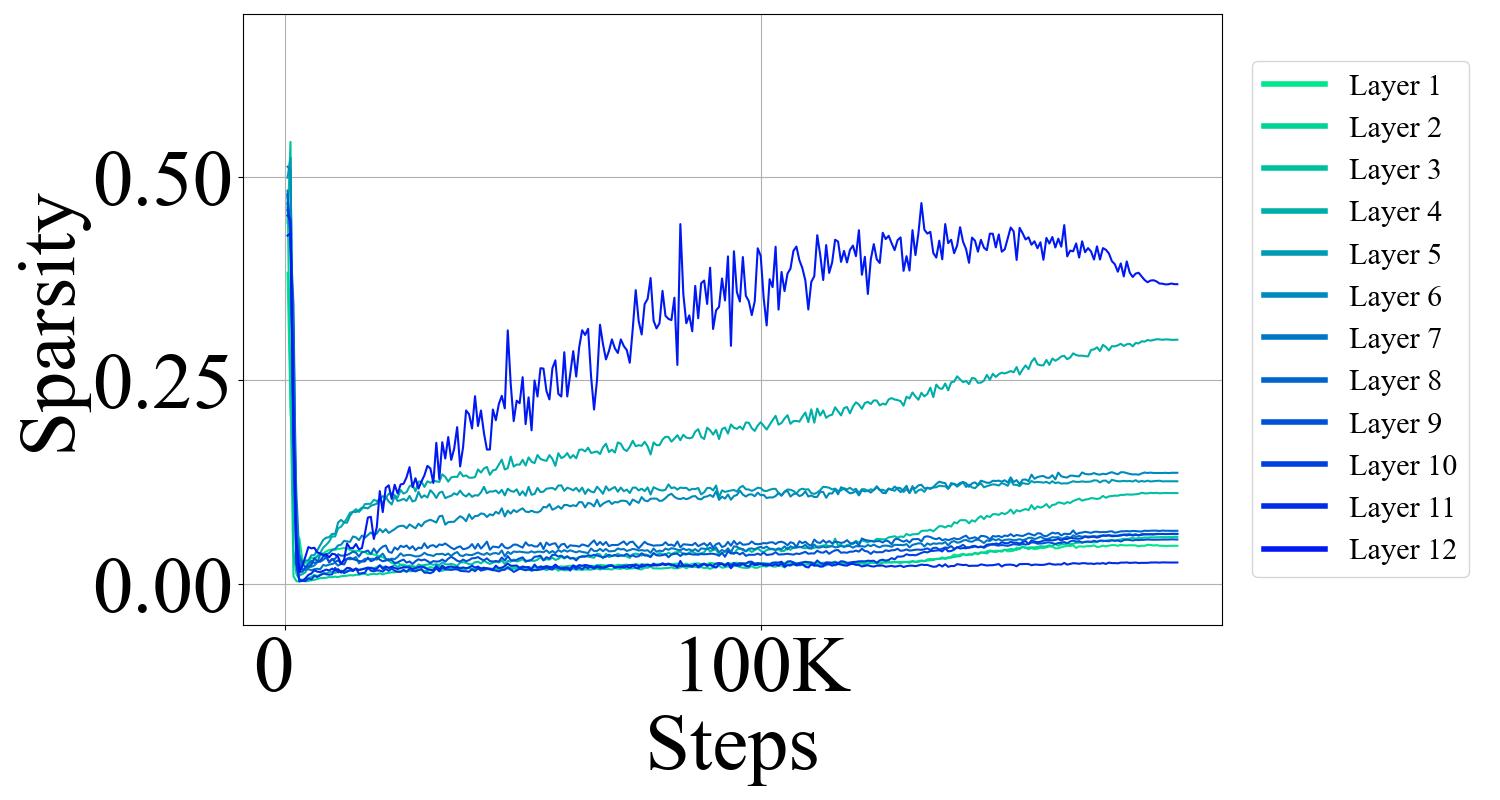}
}
    \caption{Training and testing sparsity during training of ViT-Base on ImageNet-1K across layers and steps. Red is used for modified ViT while blue indicates vanilla ViT.
    }\label{figure:productive_vit_full}
\end{figure*}

\begin{figure*}
    \centering
    \resetHeight{}
    \subfloat[Training, stepwise. \label{figure:productive_swin_average_training}]{
        \myincludegraphics[width=0.24\linewidth]{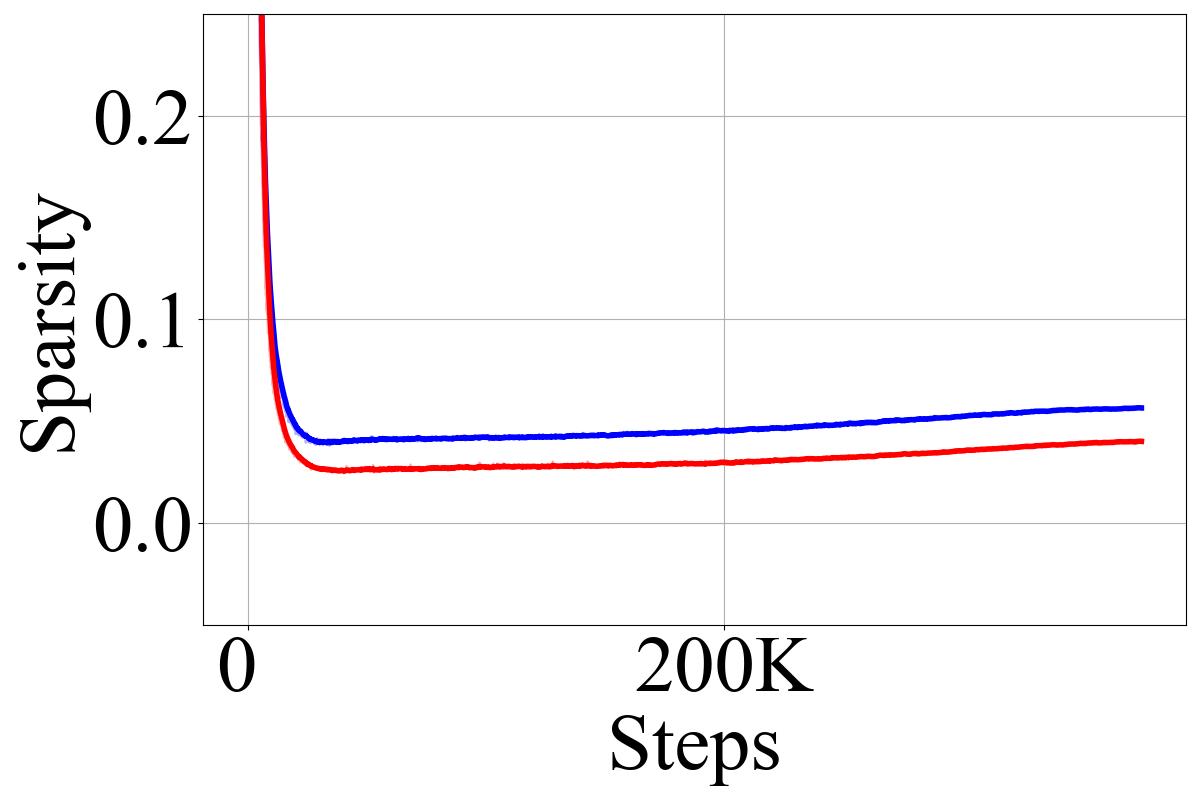}
    }
    \subfloat[Training, layerwise. \label{figure:productive_swin_end_training}]{
        \myincludegraphics[width=0.24\linewidth]{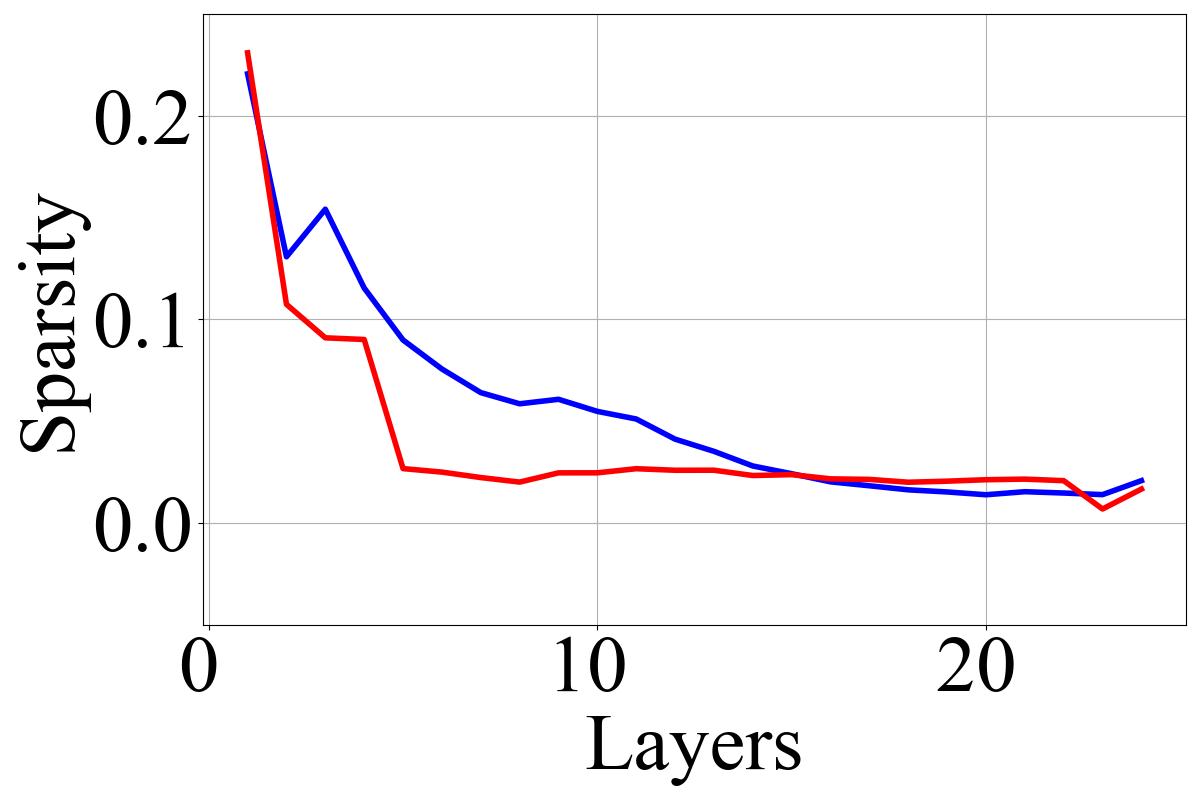}
    }
    \subfloat[Testing, stepwise.]{
        \myincludegraphics[width=0.24\linewidth]{pic/results/dumps/swin_transformer/comparison/testing.jpg}
    }
    \subfloat[Testing, layerwise.]{
        \myincludegraphics[width=0.24\linewidth]{pic/results/dumps/swin_transformer/comparison/test_end.jpg}
    }\\
    \resetHeight{}
    \subfloat[Training, modified. \label{figure:productive_swin_sparsified_training}]{
        \myincludegraphics[width=0.24\linewidth]{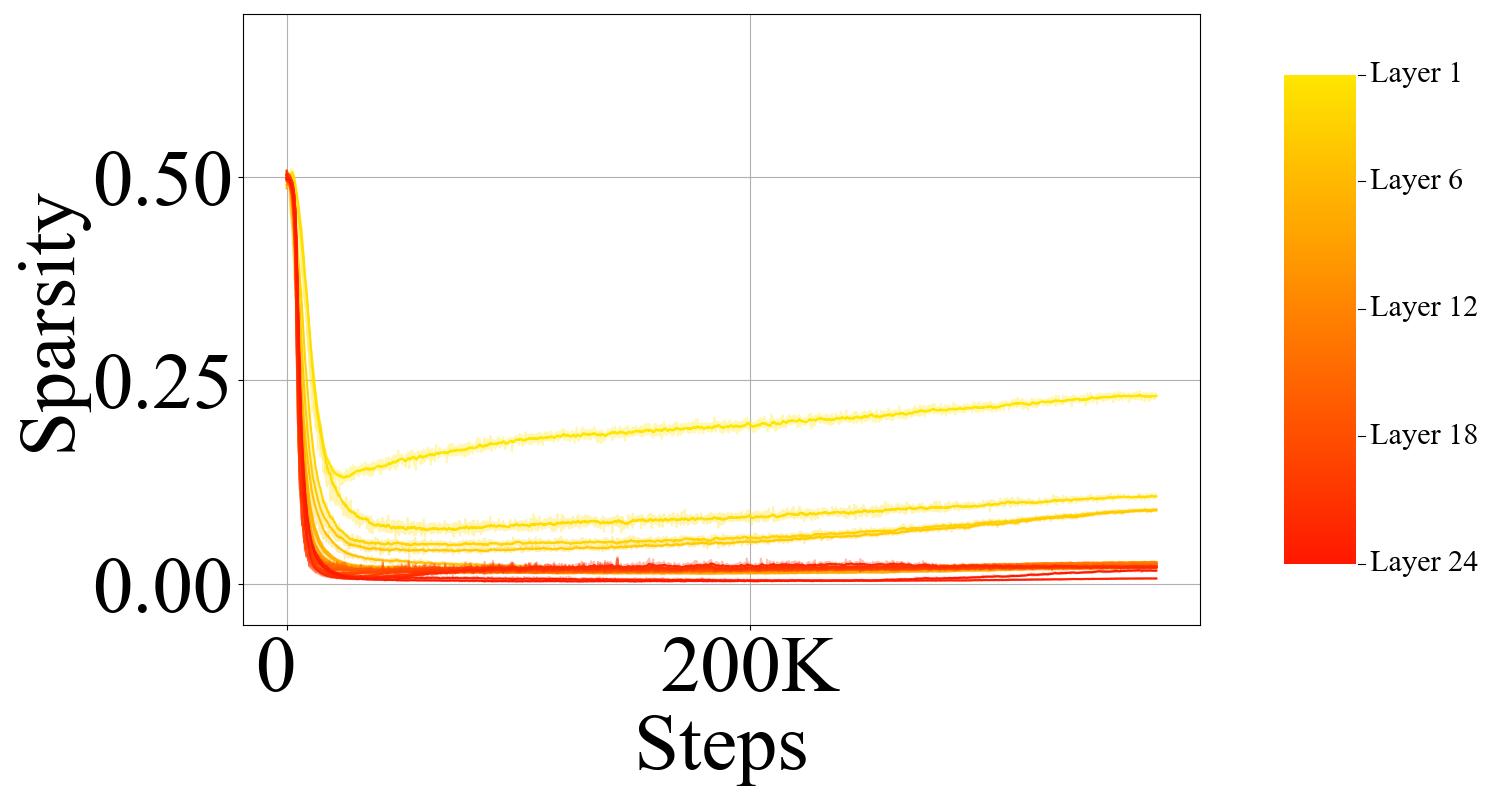}
}
    \subfloat[Training, vanilla. \label{figure:productive_swin_vanilla_training}]{
        \myincludegraphics[width=0.24\linewidth]{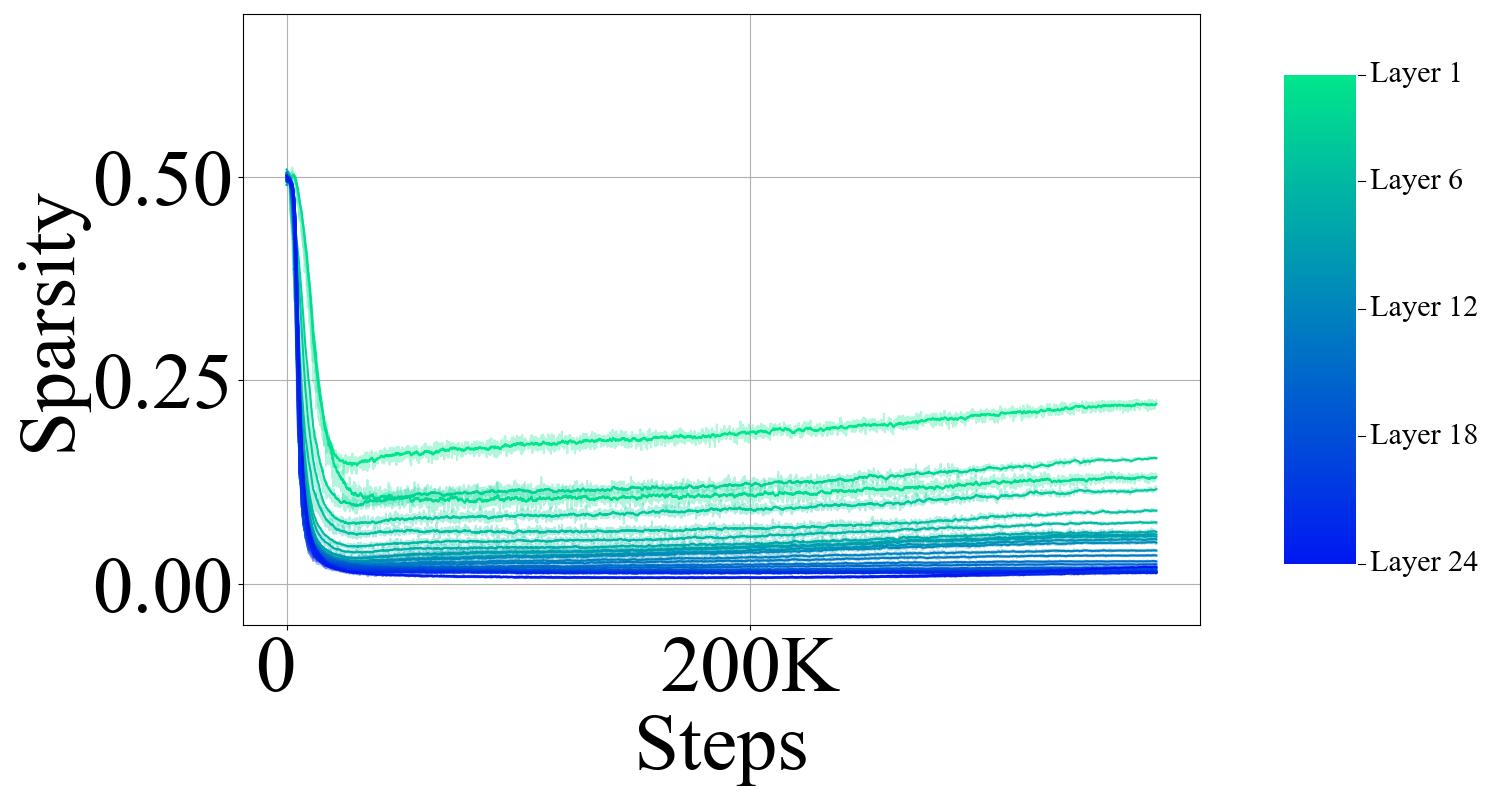}
}
    \subfloat[Testing, modified. \label{figure:productive_swin_sparsified_testing}]{
        \myincludegraphics[width=0.24\linewidth]{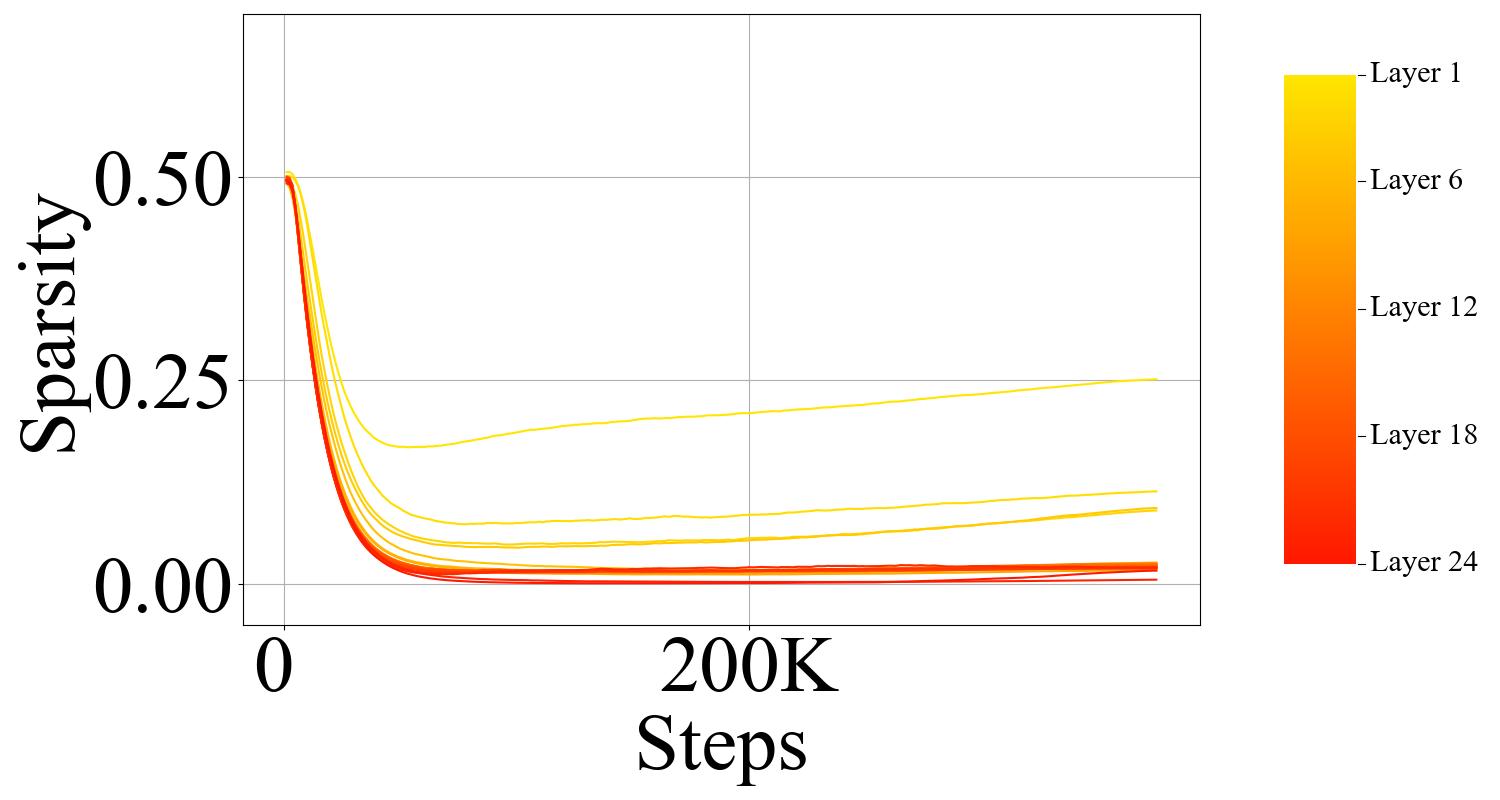}
}
    \subfloat[Testing, vanilla. \label{figure:productive_swin_vanilla_testing}]{
        \myincludegraphics[width=0.24\linewidth]{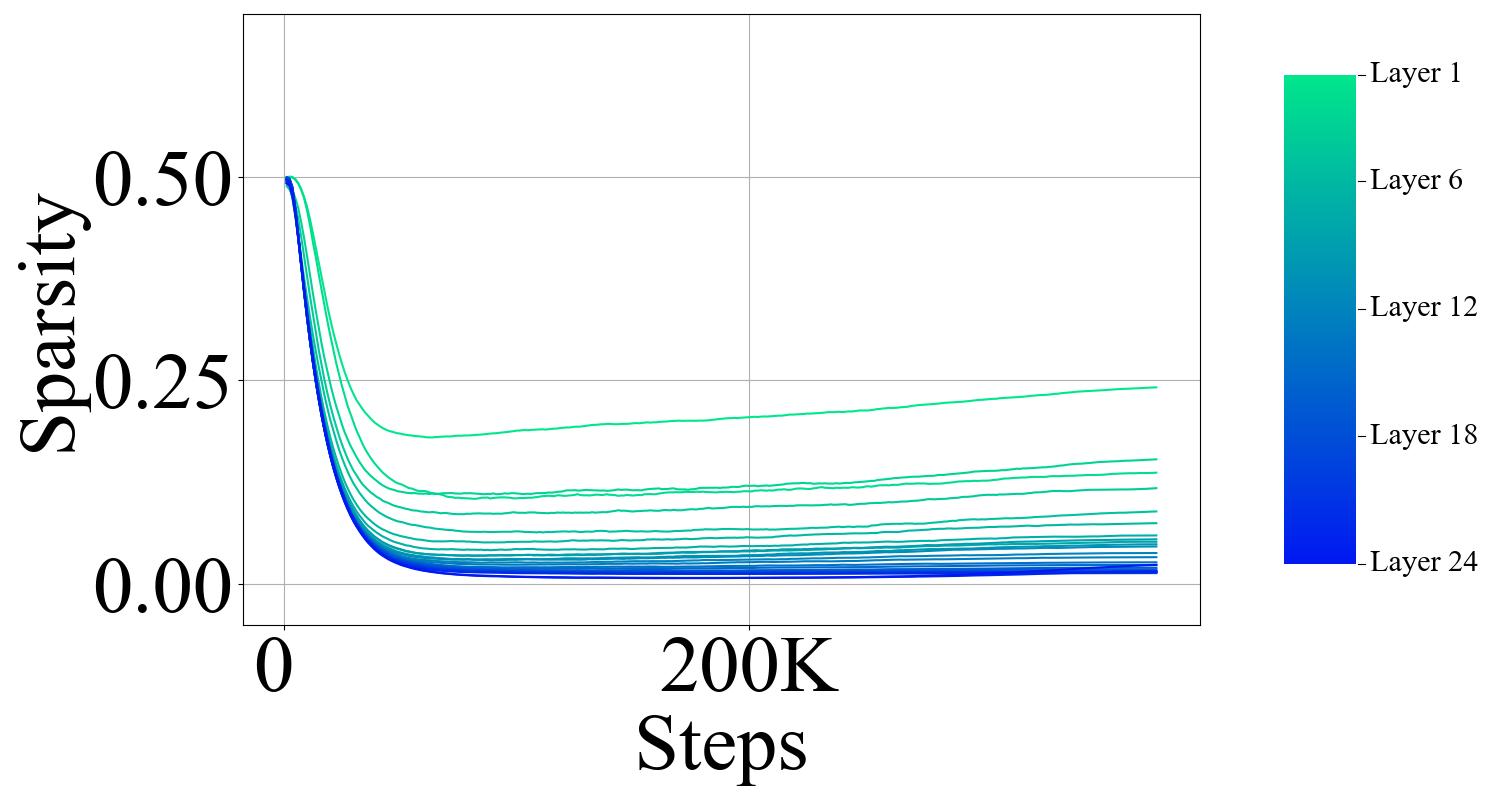}
}
    \caption{\minorrevision{Training and testing sparsity during training of SwinTransformer-Base on ImageNet-1K across layers and steps. Red is used for modified SwinTransformer while blue indicates vanilla SwinTransformer.
    }}\label{figure:productive_swin_full}
\end{figure*}

\begin{figure*}
    \centering
    \resetHeight{}
    \subfloat[Training, stepwise. \label{figure:productive_vit_large_average_training}]{
        \centering
        \myincludegraphics[width=0.24\linewidth]{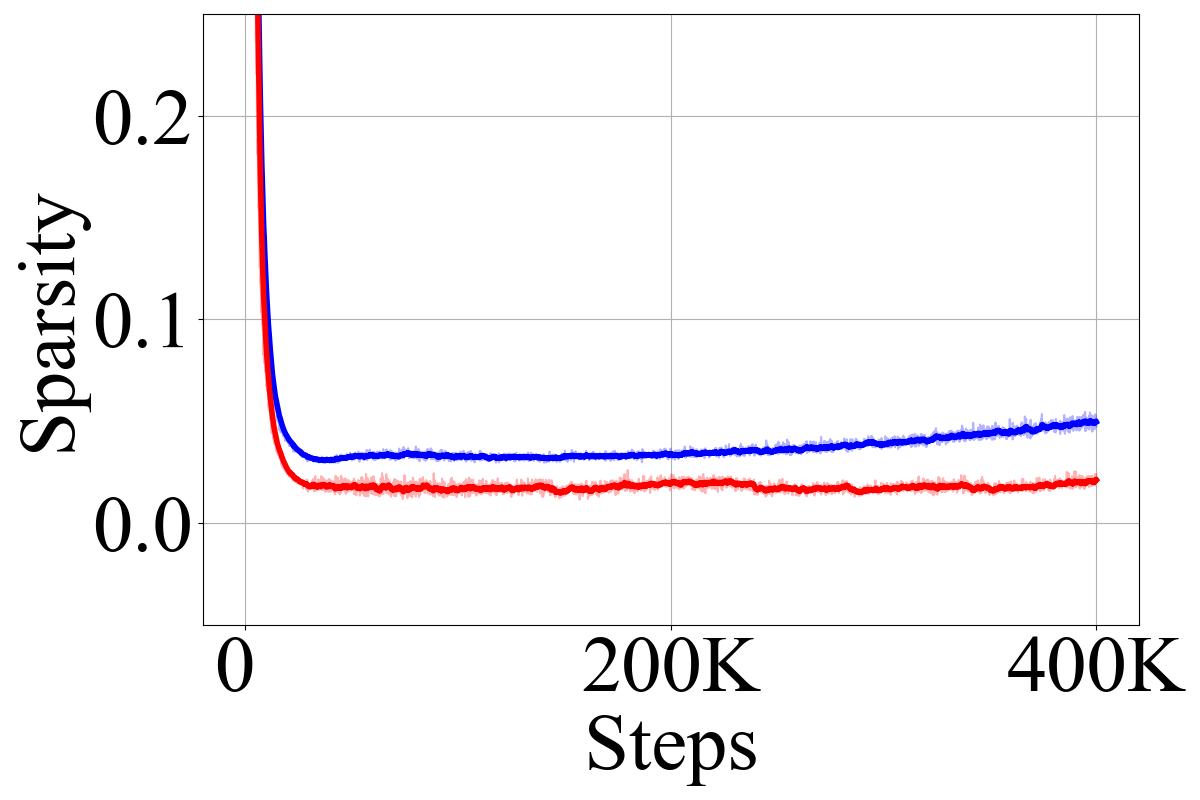}
    }
    \subfloat[Training, layerwise. \label{figure:productive_vit_large_end_training}]{
        \centering
        \myincludegraphics[width=0.24\linewidth]{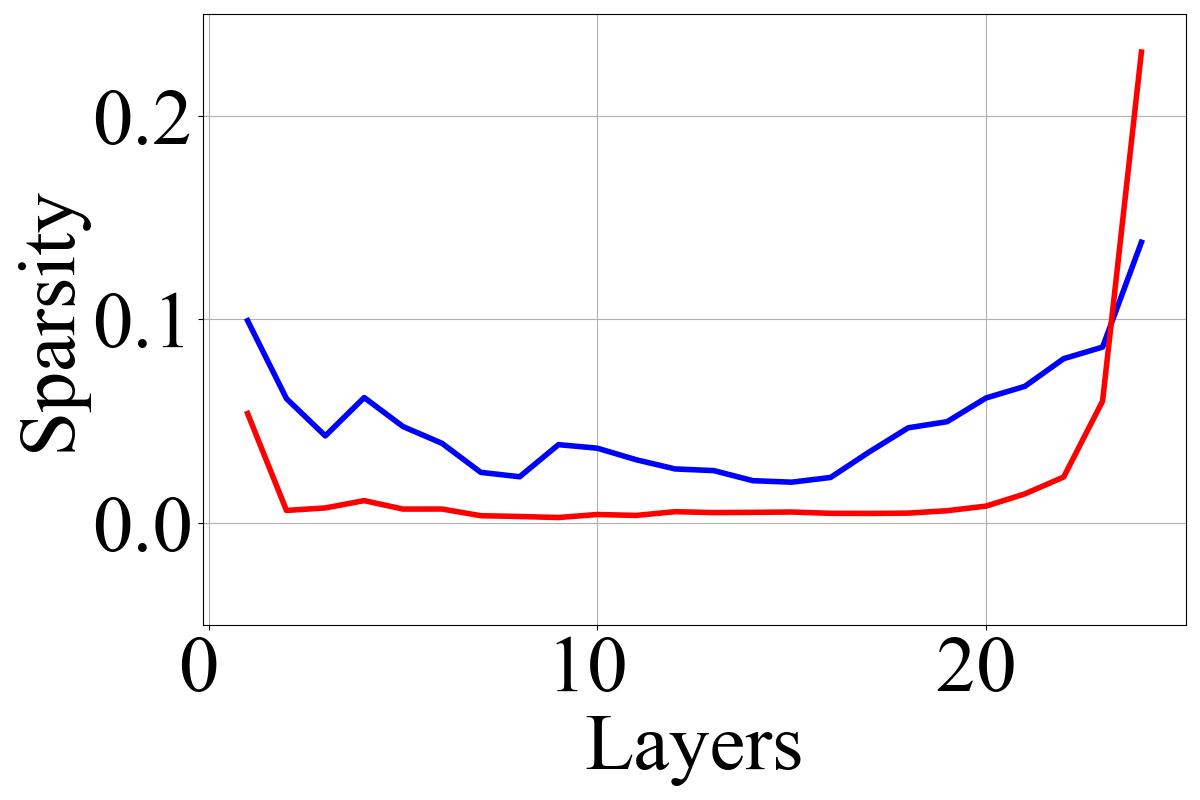}
    }
    \subfloat[Testing, stepwise.]{
        \centering
        \myincludegraphics[width=0.24\linewidth]{pic/results/dumps/vit_large/comparison/testing.jpg}
    }
    \subfloat[Testing, layerwise.]{
        \centering
        \myincludegraphics[width=0.24\linewidth]{pic/results/dumps/vit_large/comparison/test_end.jpg}
    }\\
    \resetHeight{}
    \subfloat[Training, modified. \label{figure:productive_vit_large_sparsified_training}]{
        \centering
        \myincludegraphics[width=0.24\linewidth]{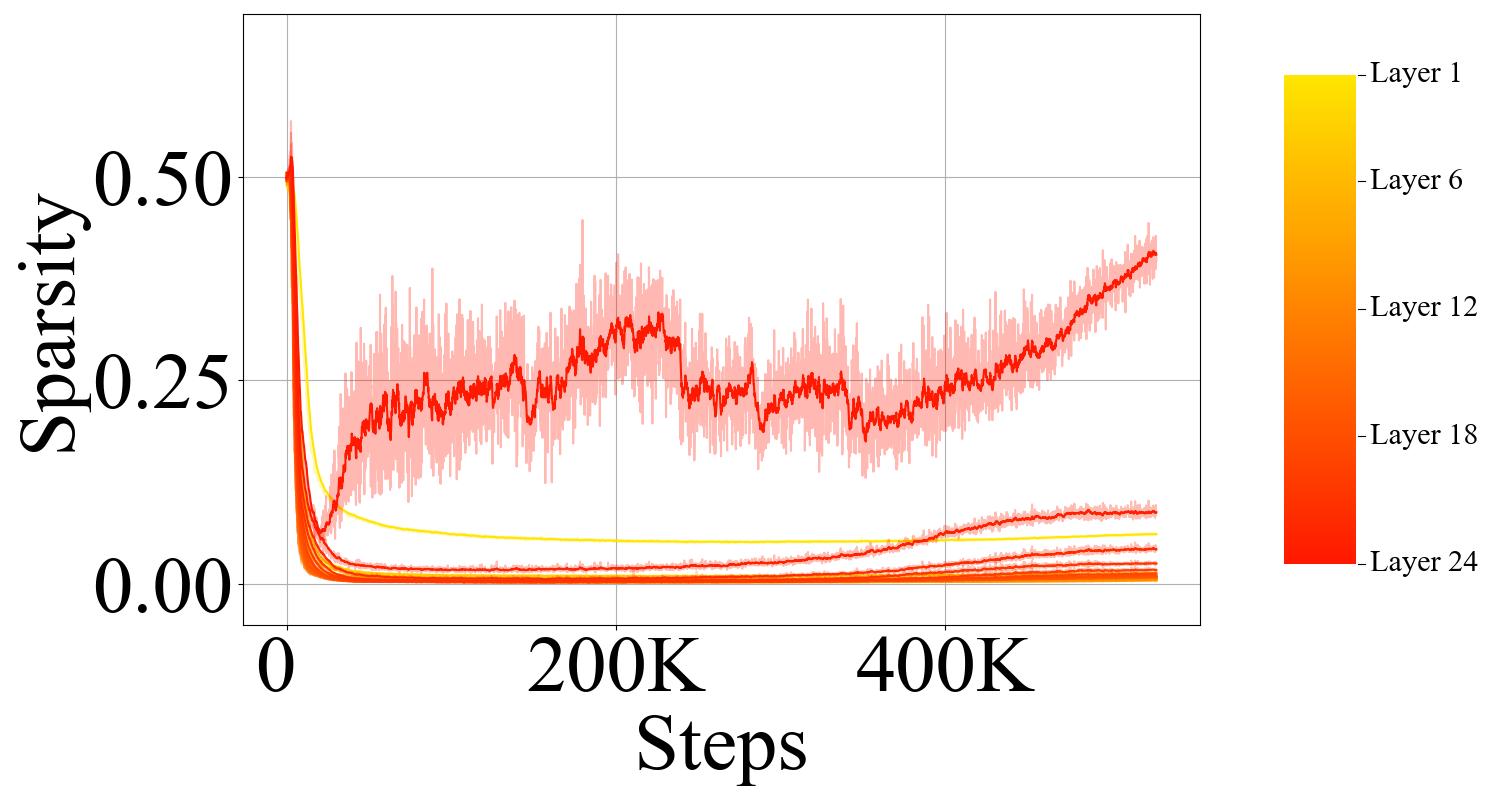}
}
    \subfloat[Training, vanilla. \label{figure:productive_vit_large_vanilla_training}]{
        \centering
        \myincludegraphics[width=0.24\linewidth]{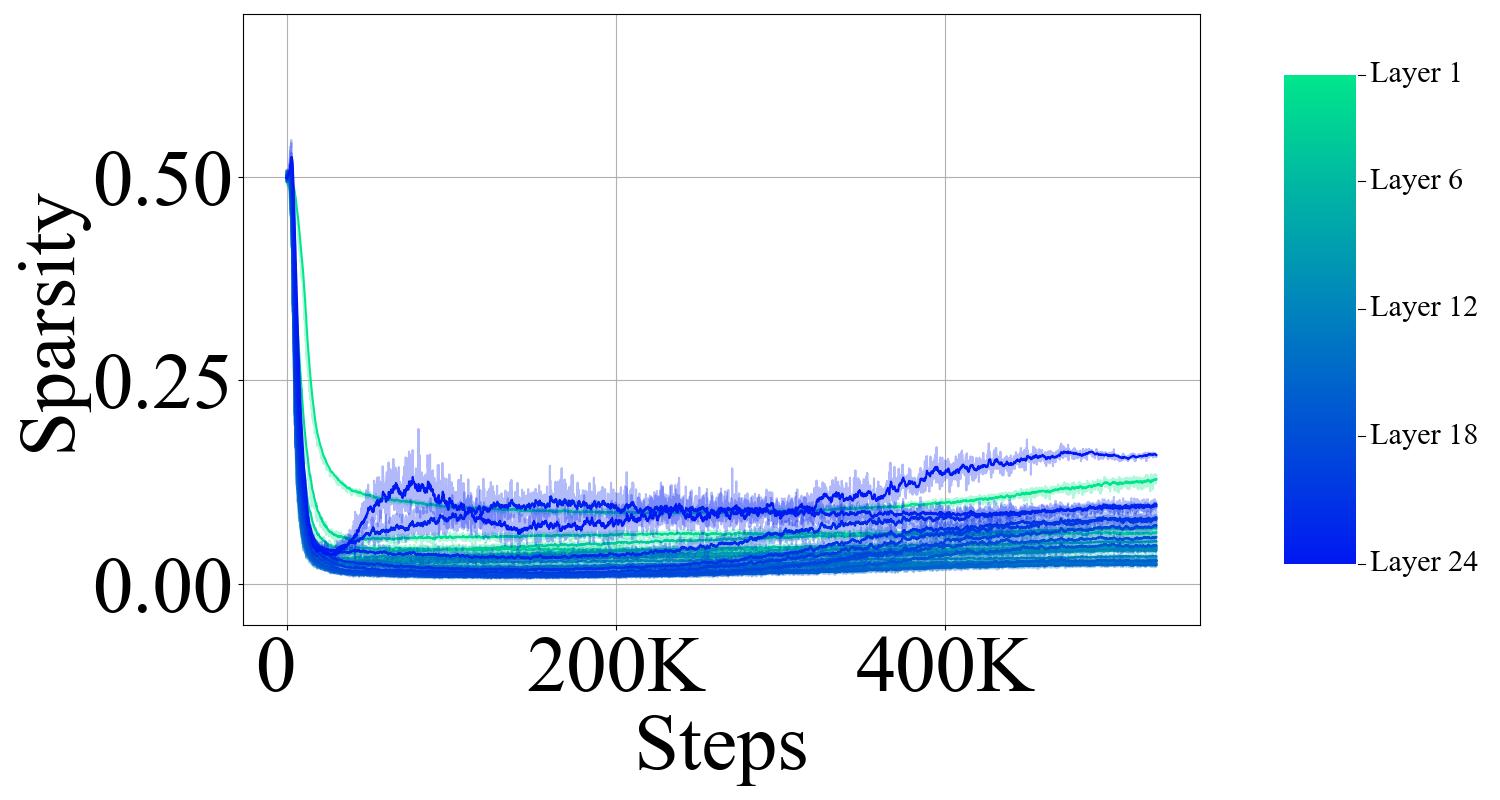}
}
    \subfloat[Testing, modified. \label{figure:productive_vit_large_sparsified_testing}]{
        \centering
        \myincludegraphics[width=0.24\linewidth]{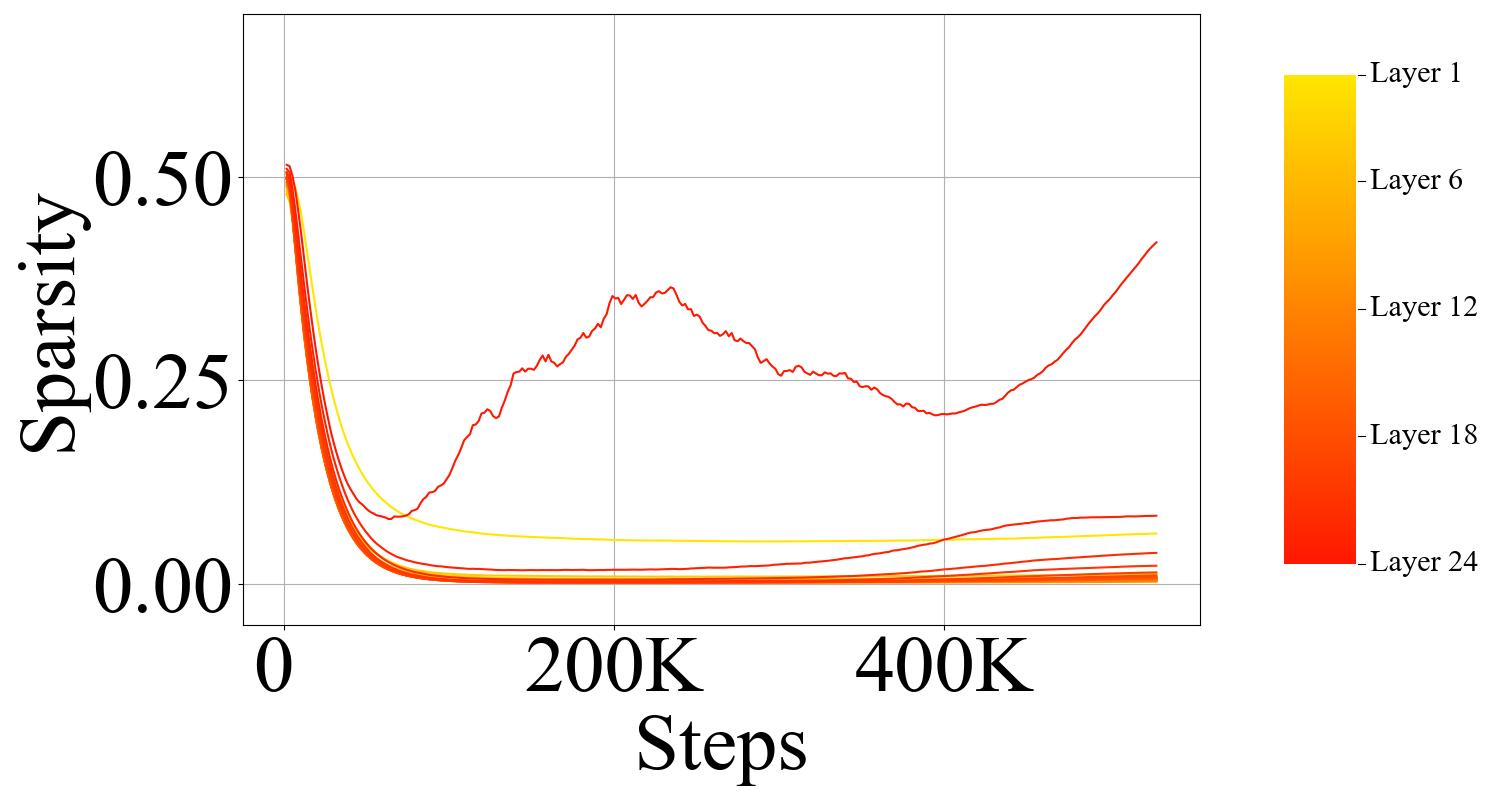}
}
    \subfloat[Testing, vanilla. \label{figure:productive_vit_large_vanilla_testing}]{
        \centering
        \myincludegraphics[width=0.24\linewidth]{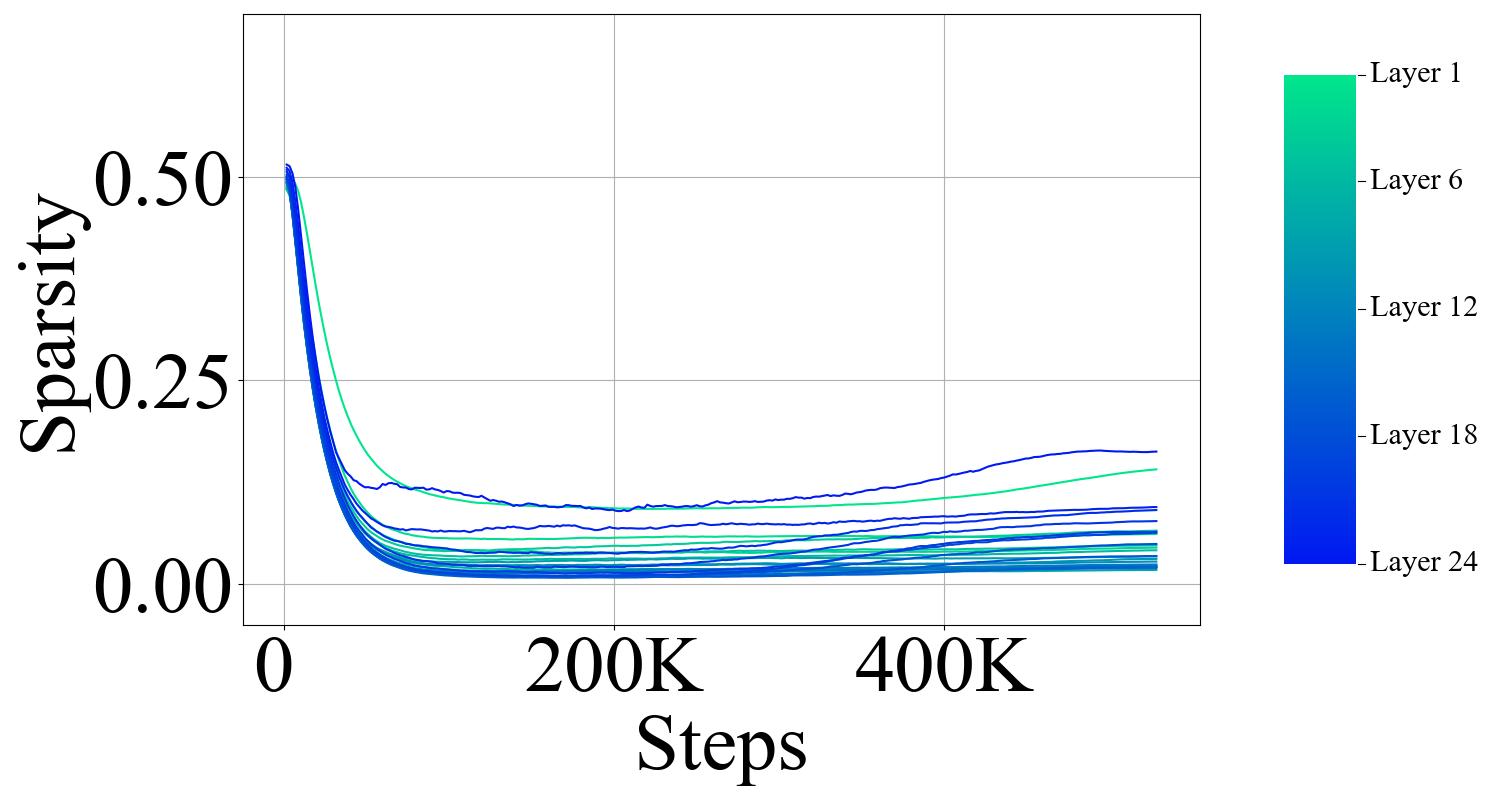}
}
    \caption{\minorrevision{Training and testing sparsity during training of ViT-Large on Places365-Standard across layers and steps. Red is used for modified ViT-Large while blue indicates vanilla ViT-Large.
    }}\label{figure:productive_vit_large_full}
\end{figure*}

\begin{figure*}
    \centering
    \resetHeight{}
    \subfloat[Training, stepwise. \label{figure:productive_laion_average_training}]{
        \centering
        \myincludegraphics[width=0.24\linewidth]{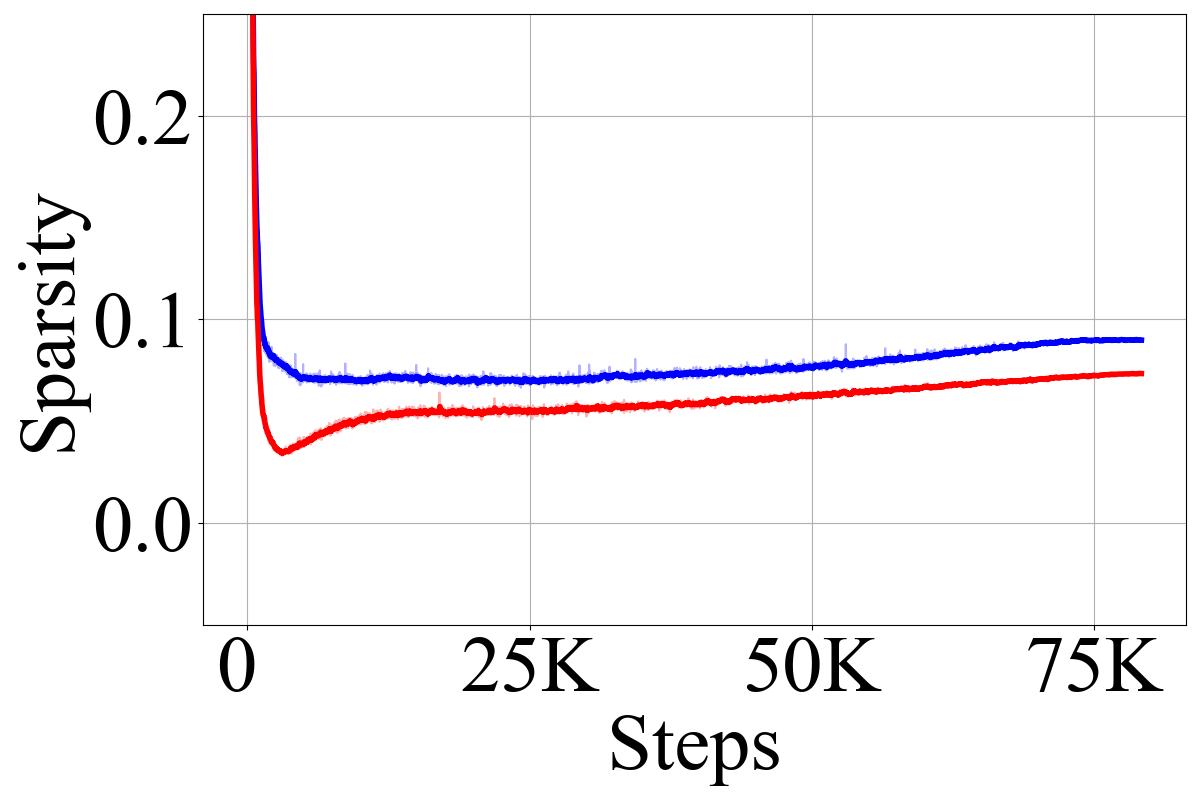}
    }
    \subfloat[Training, layerwise. \label{figure:productive_laion_end_training}]{
        \centering
        \myincludegraphics[width=0.24\linewidth]{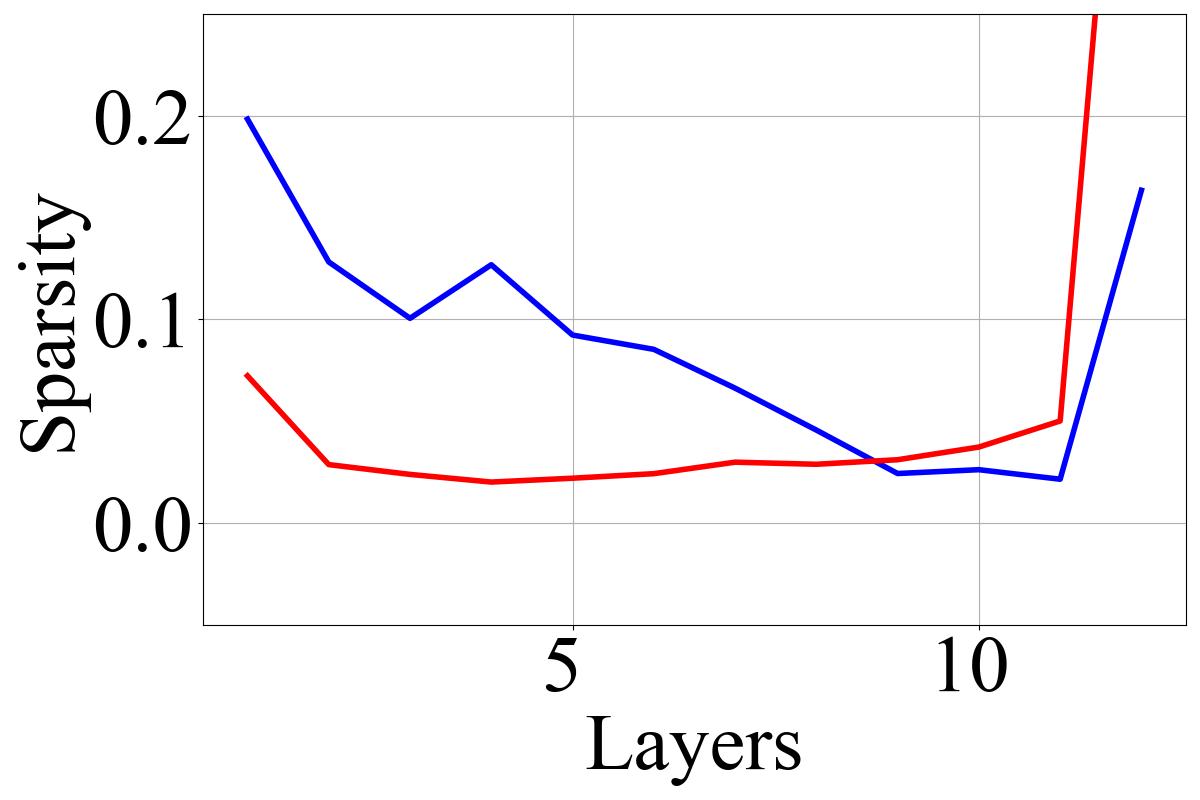}
    }
    \subfloat[Testing, stepwise.]{
        \centering
        \myincludegraphics[width=0.24\linewidth]{pic/results/dumps/laion/comparison/testing.jpg}
    }
    \subfloat[Testing, layerwise.]{
        \centering
        \myincludegraphics[width=0.24\linewidth]{pic/results/dumps/laion/comparison/test_end.jpg}
    }\\
    \resetHeight{}
    \subfloat[Training, modified. \label{figure:productive_laion_sparsified_training}]{
        \centering
        \myincludegraphics[width=0.24\linewidth]{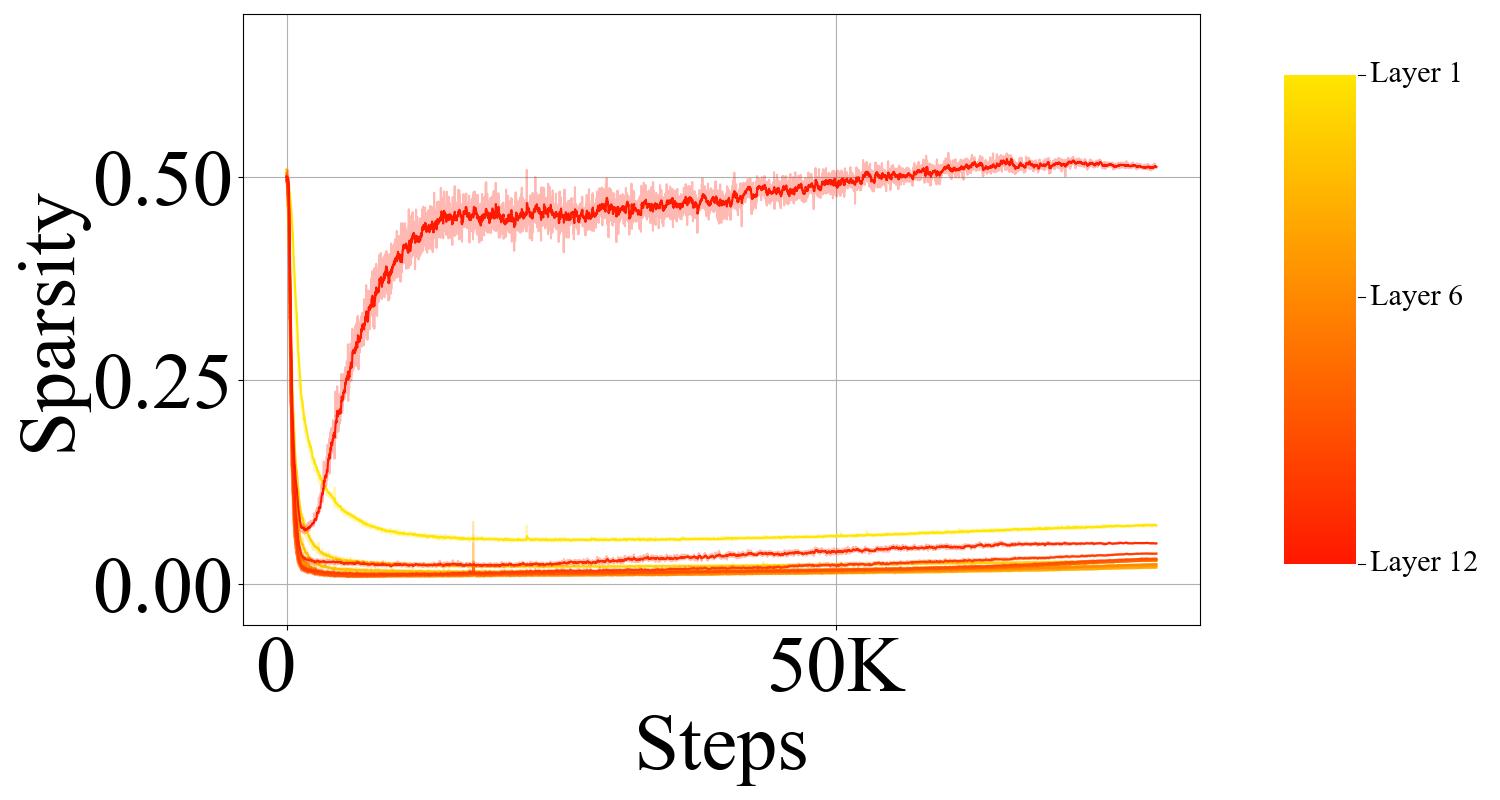}
}
    \subfloat[Training, vanilla. \label{figure:productive_laion_vanilla_training}]{
        \centering
        \myincludegraphics[width=0.24\linewidth]{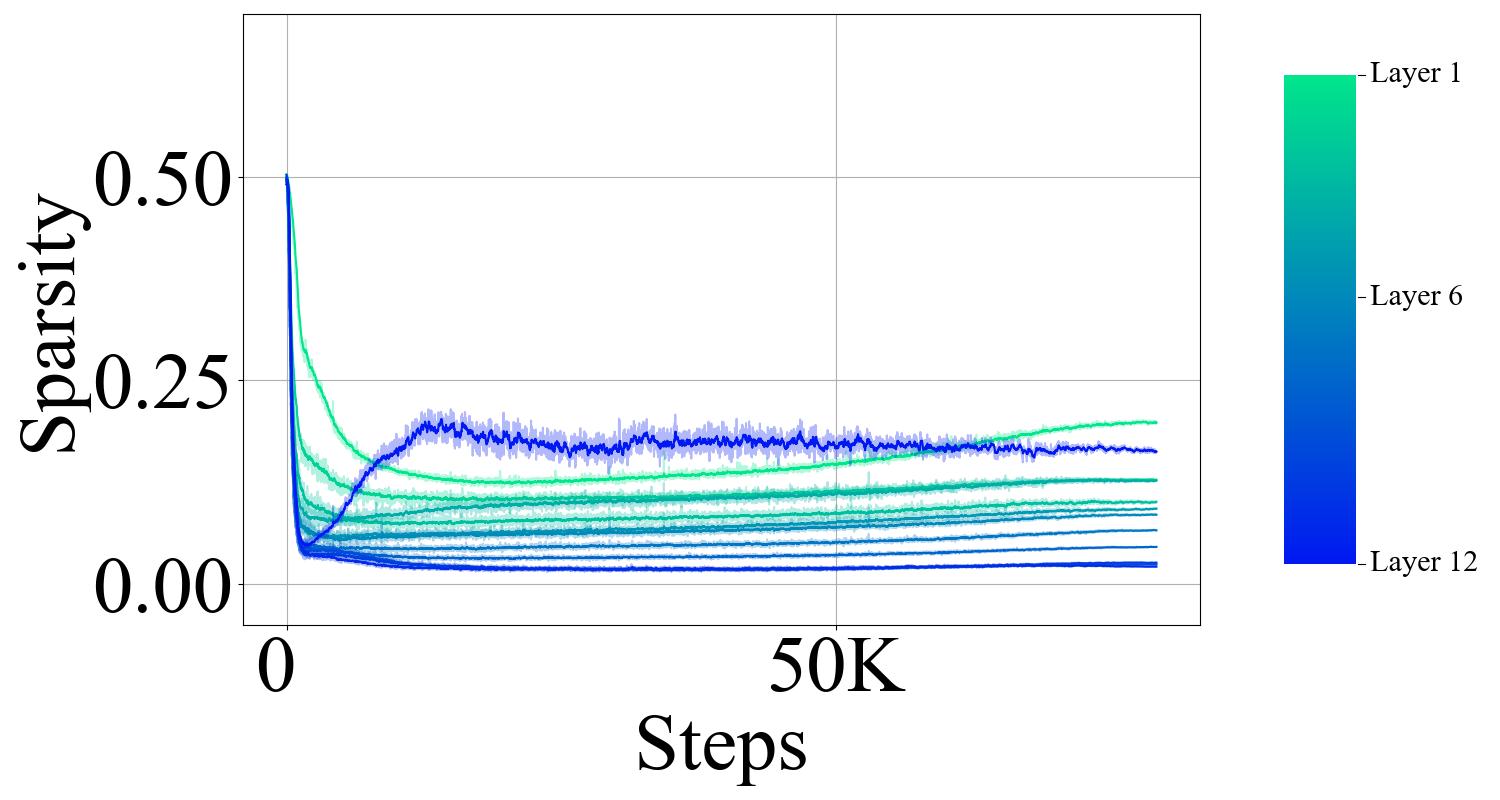}
}
    \subfloat[Testing, modified. \label{figure:productive_laion_sparsified_testing}]{
        \centering
        \myincludegraphics[width=0.24\linewidth]{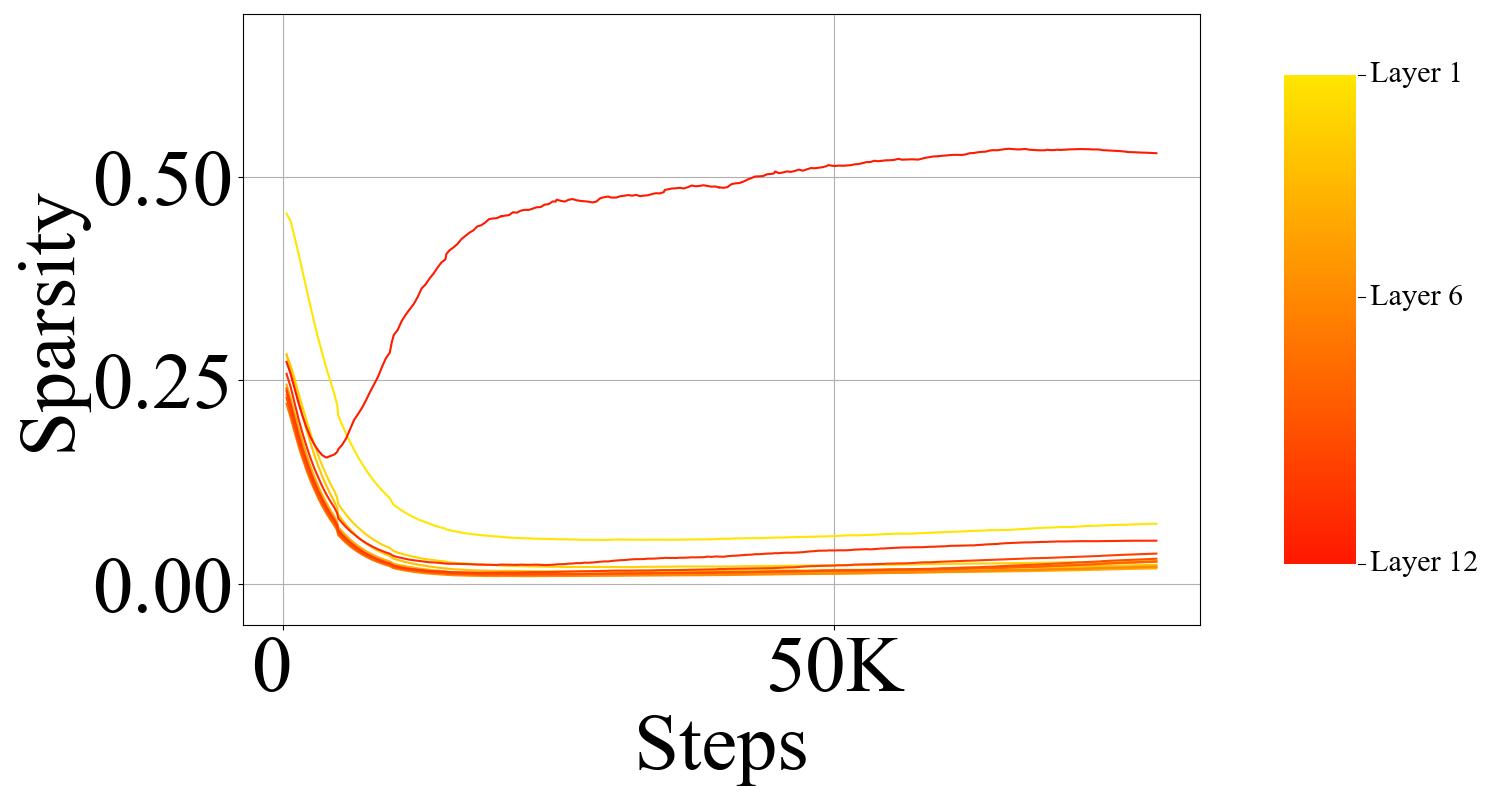}
}
    \subfloat[Testing, vanilla. \label{figure:productive_laion_vanilla_testing}]{
        \centering
        \myincludegraphics[width=0.24\linewidth]{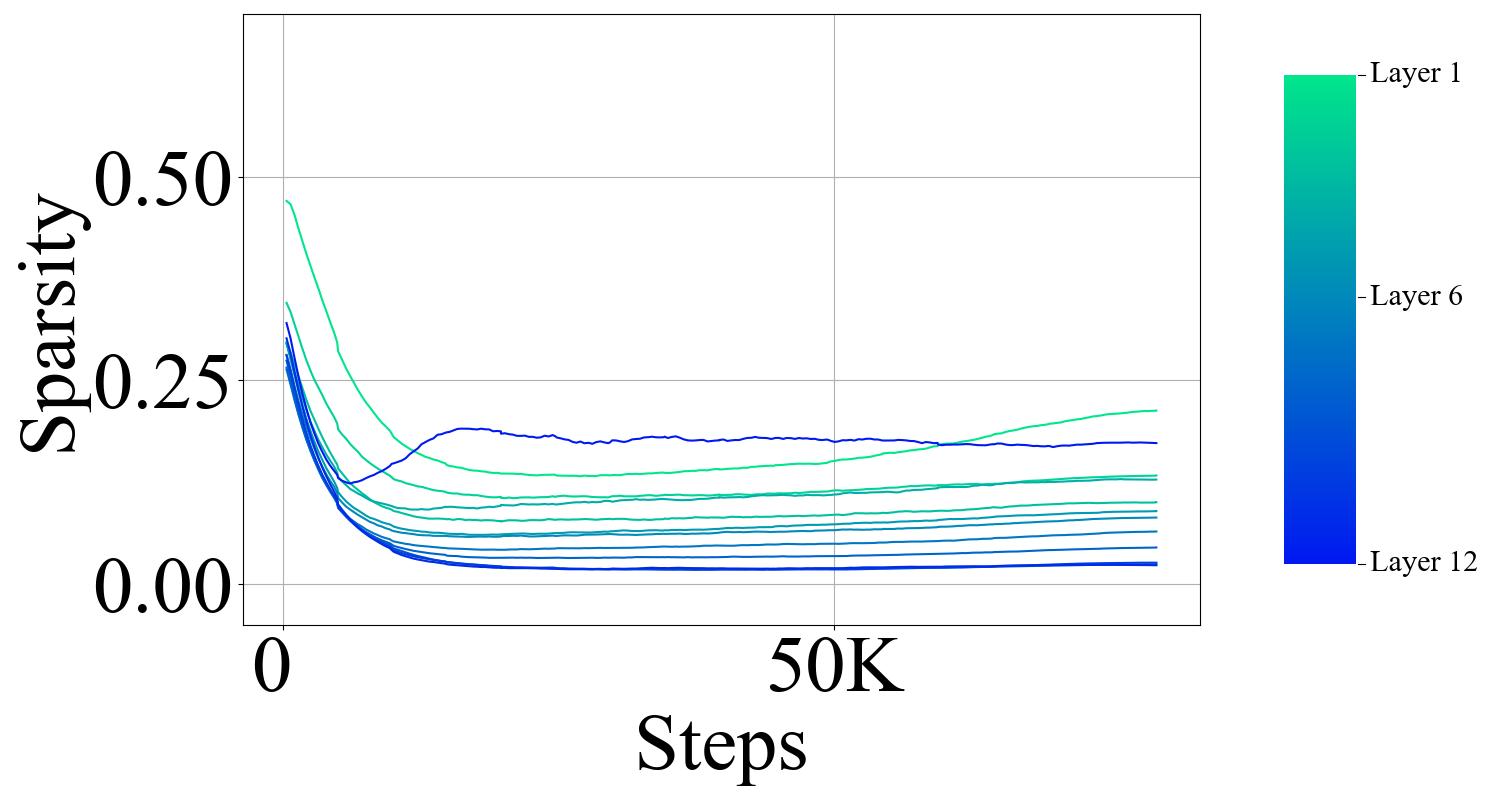}
}
    \caption{\minorrevision{Training and testing sparsity during training of ViT-Base on LAION-40M across layers and steps. Red is used for modified ViT-Base while blue indicates vanilla ViT-Base.
    }}\label{figure:productive_laion_full}
\end{figure*}

\begin{figure*}
    \centering
    \resetHeight{}
    \subfloat[Training, modified enc.\label{figure:productive_t5_sparsified_training_encoder}]{
        \myincludegraphics[width=0.24\linewidth]{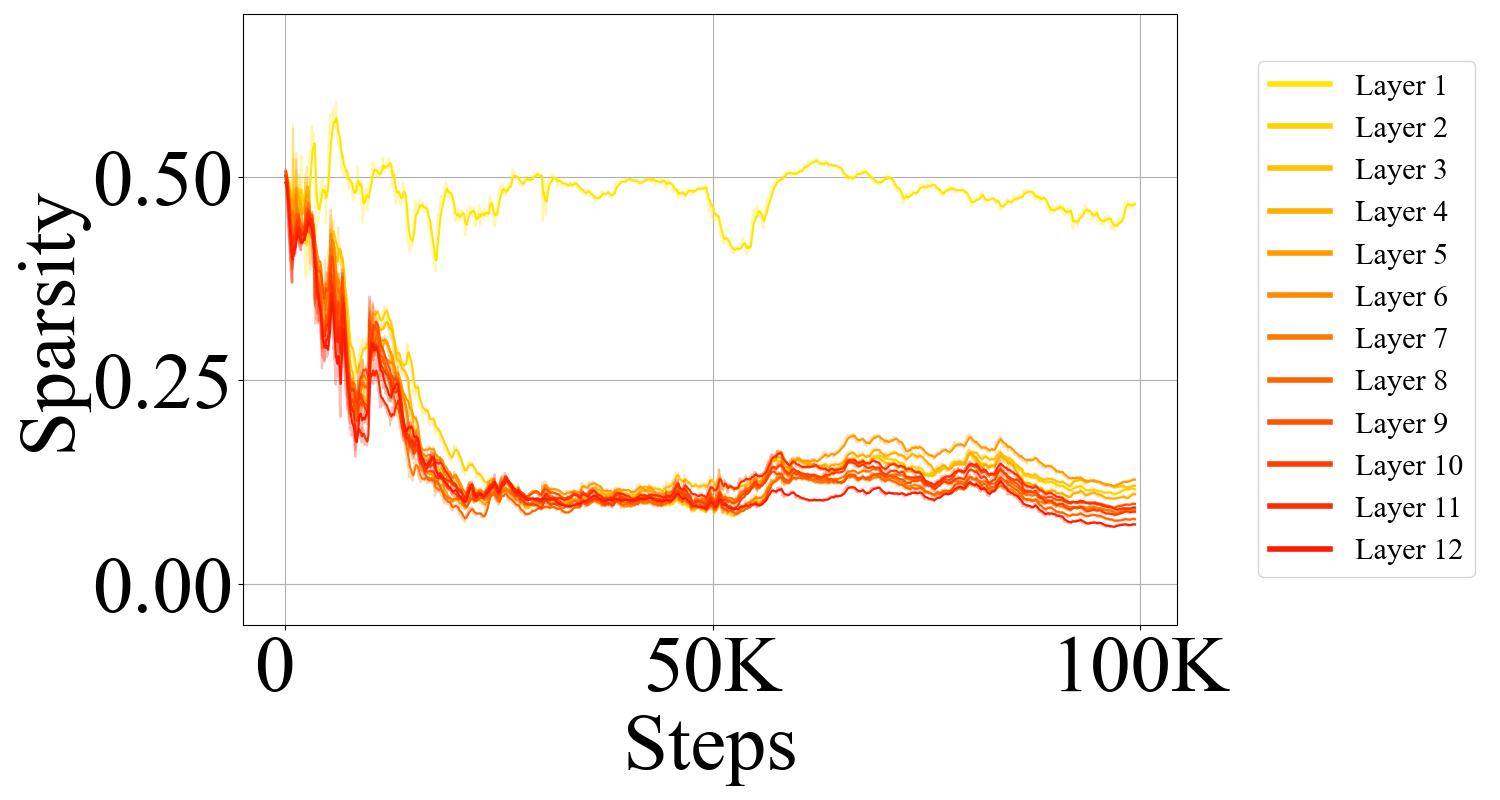}
}
    \subfloat[Training, vanilla enc.\label{figure:productive_t5_vanilla_training_encoder}]{
        \myincludegraphics[width=0.24\linewidth]{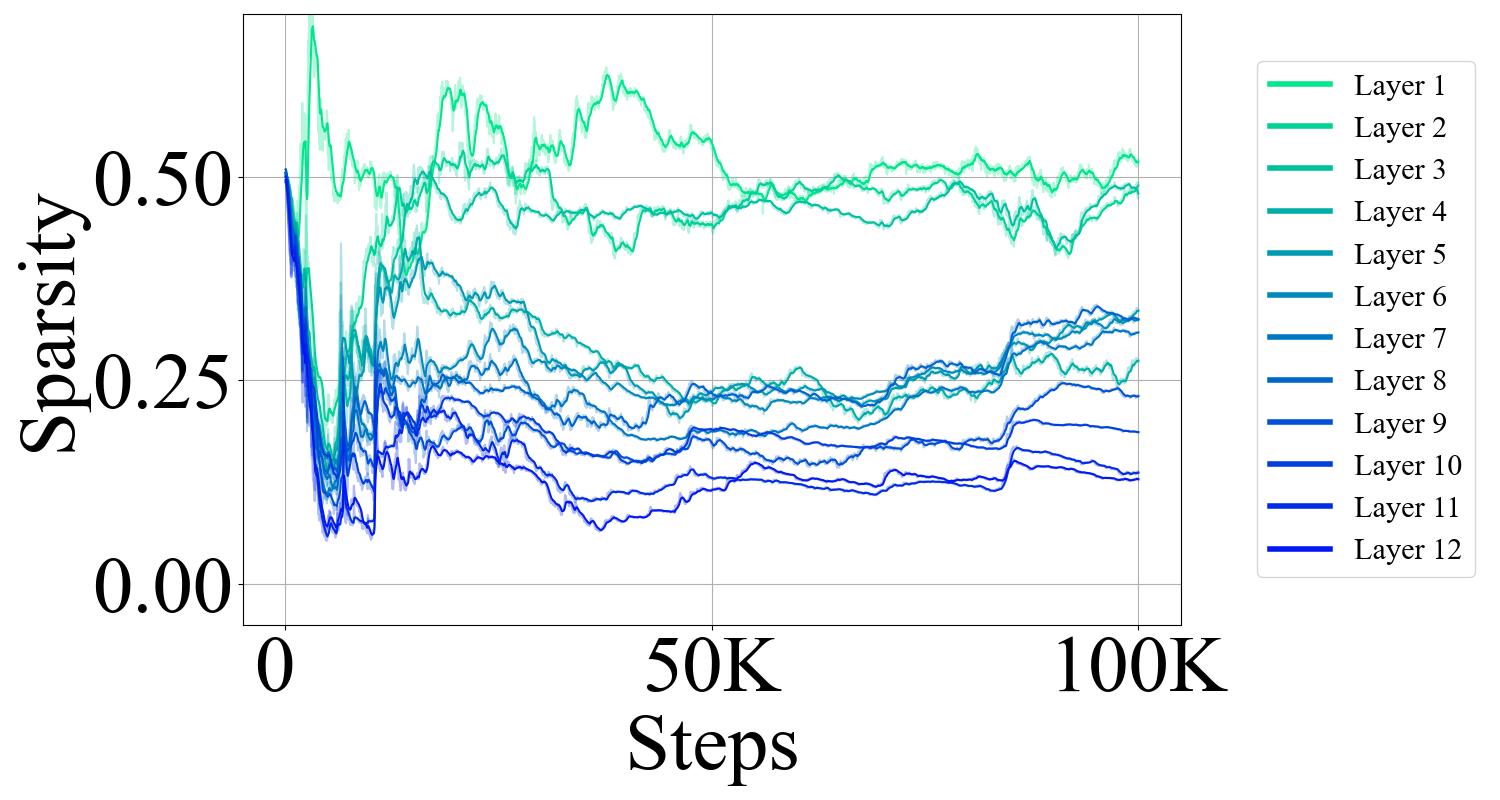}
}
    \subfloat[Training, modified dec.\label{figure:productive_t5_sparsified_training_decoder}]{
        \myincludegraphics[width=0.24\linewidth]{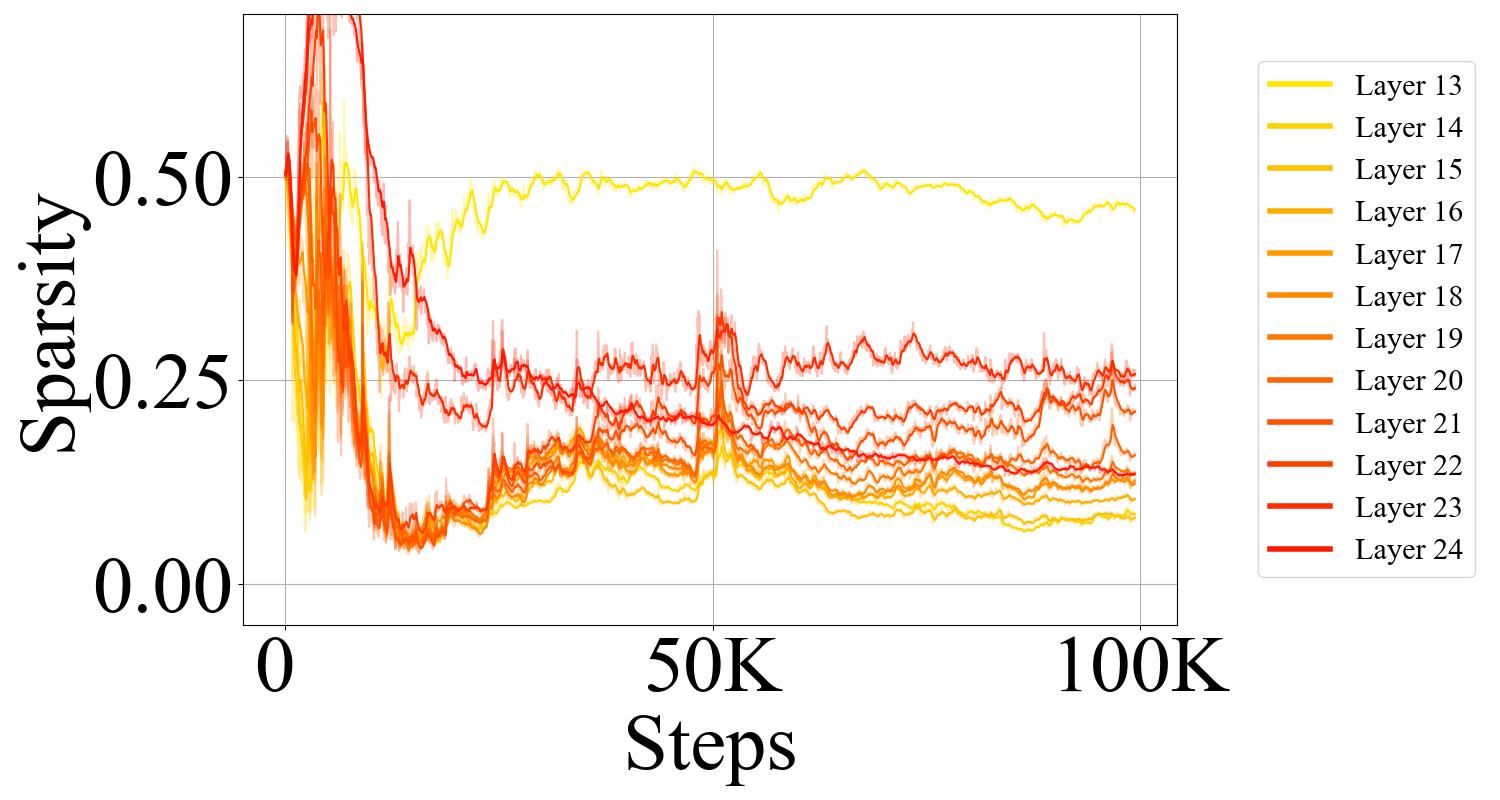}
}
    \subfloat[Training, vanilla dec.\label{figure:productive_t5_vanilla_training_decoder}]{
        \myincludegraphics[width=0.24\linewidth]{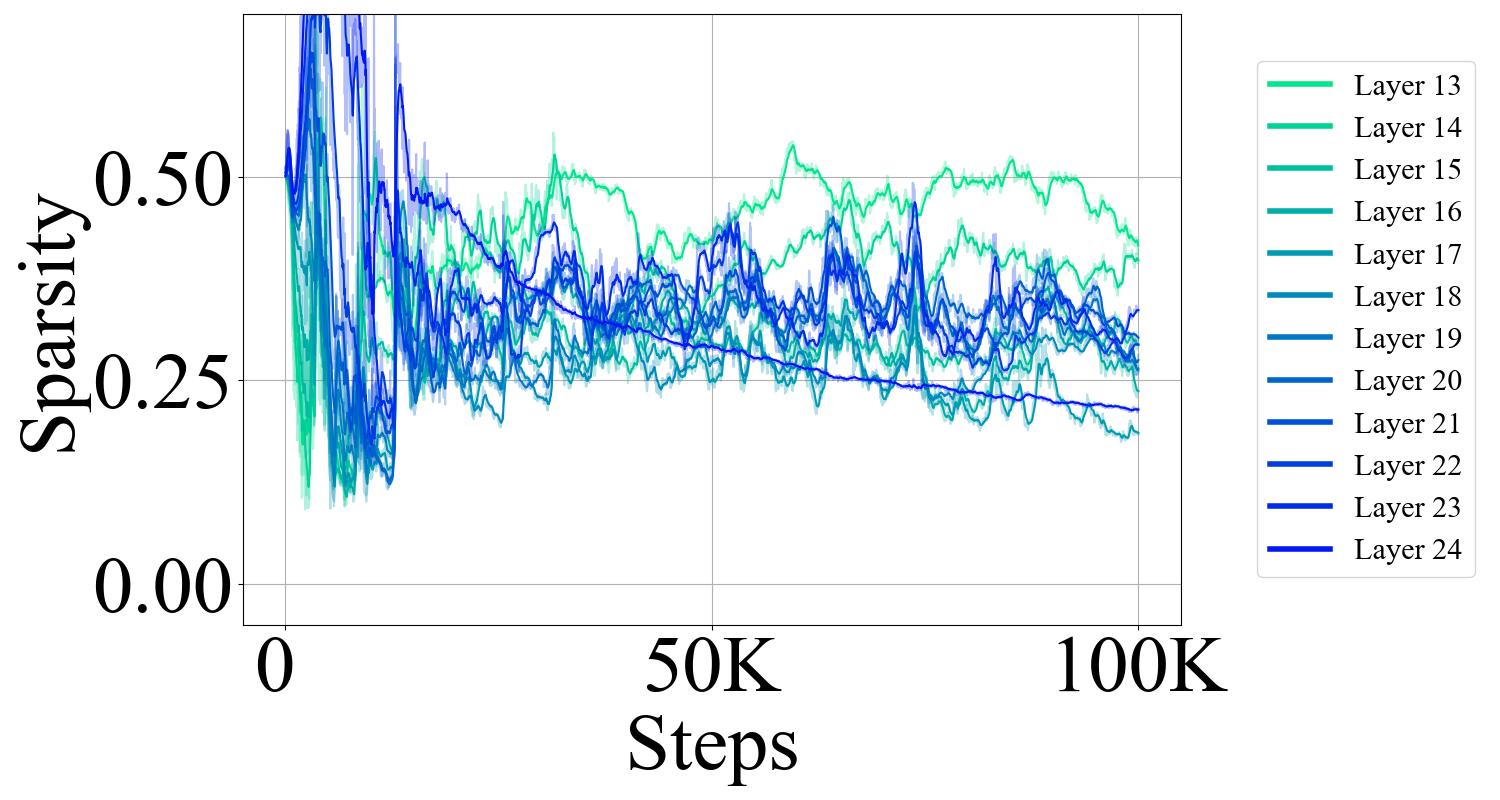}
} \\
    \subfloat[Testing, modified enc.\label{figure:productive_t5_sparsified_testing_encoder}]{
        \myincludegraphics[width=0.24\linewidth]{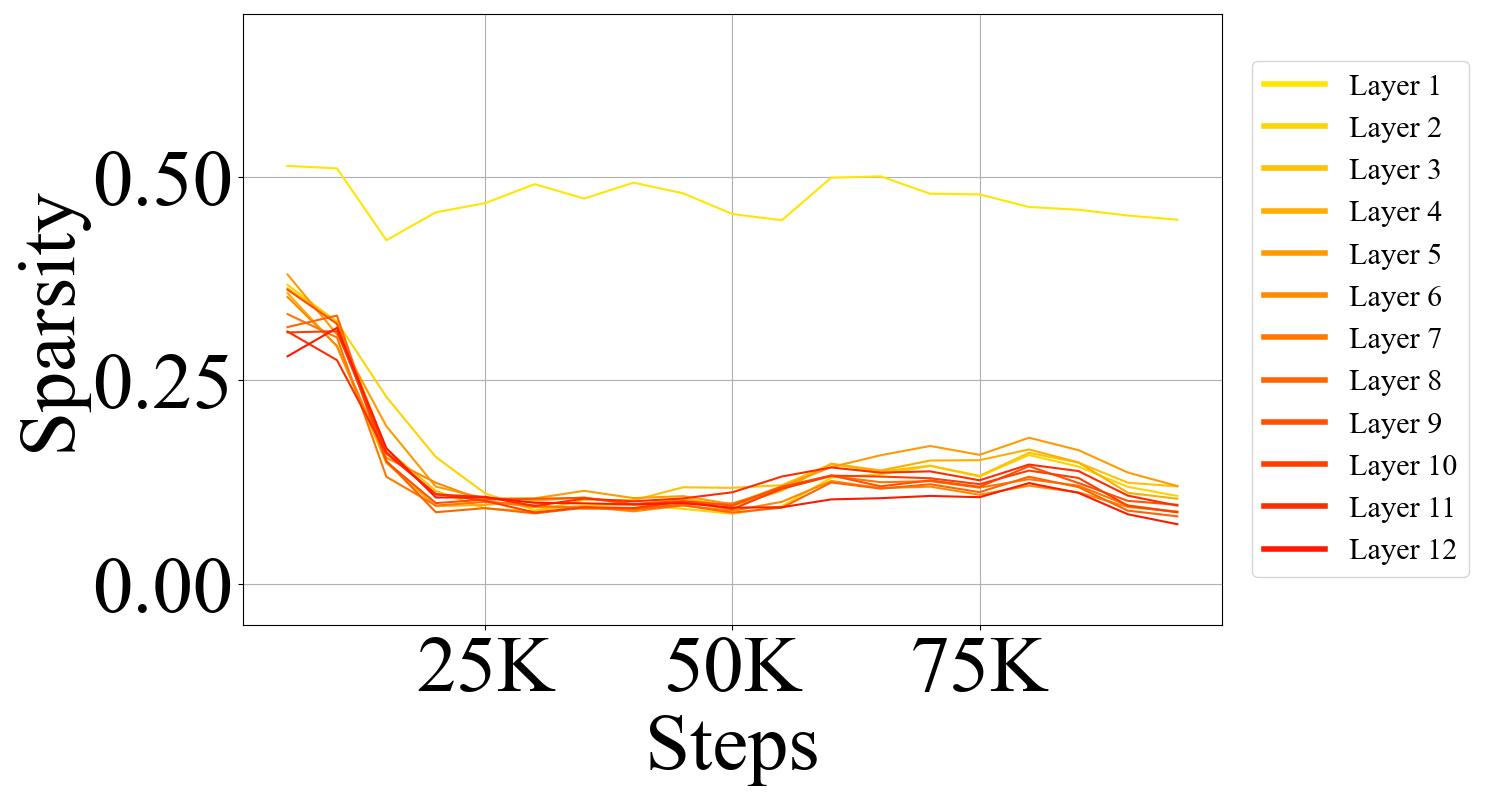}
}
    \subfloat[Testing, vanilla enc.\label{figure:productive_t5_vanilla_testing_encoder}]{
        \myincludegraphics[width=0.24\linewidth]{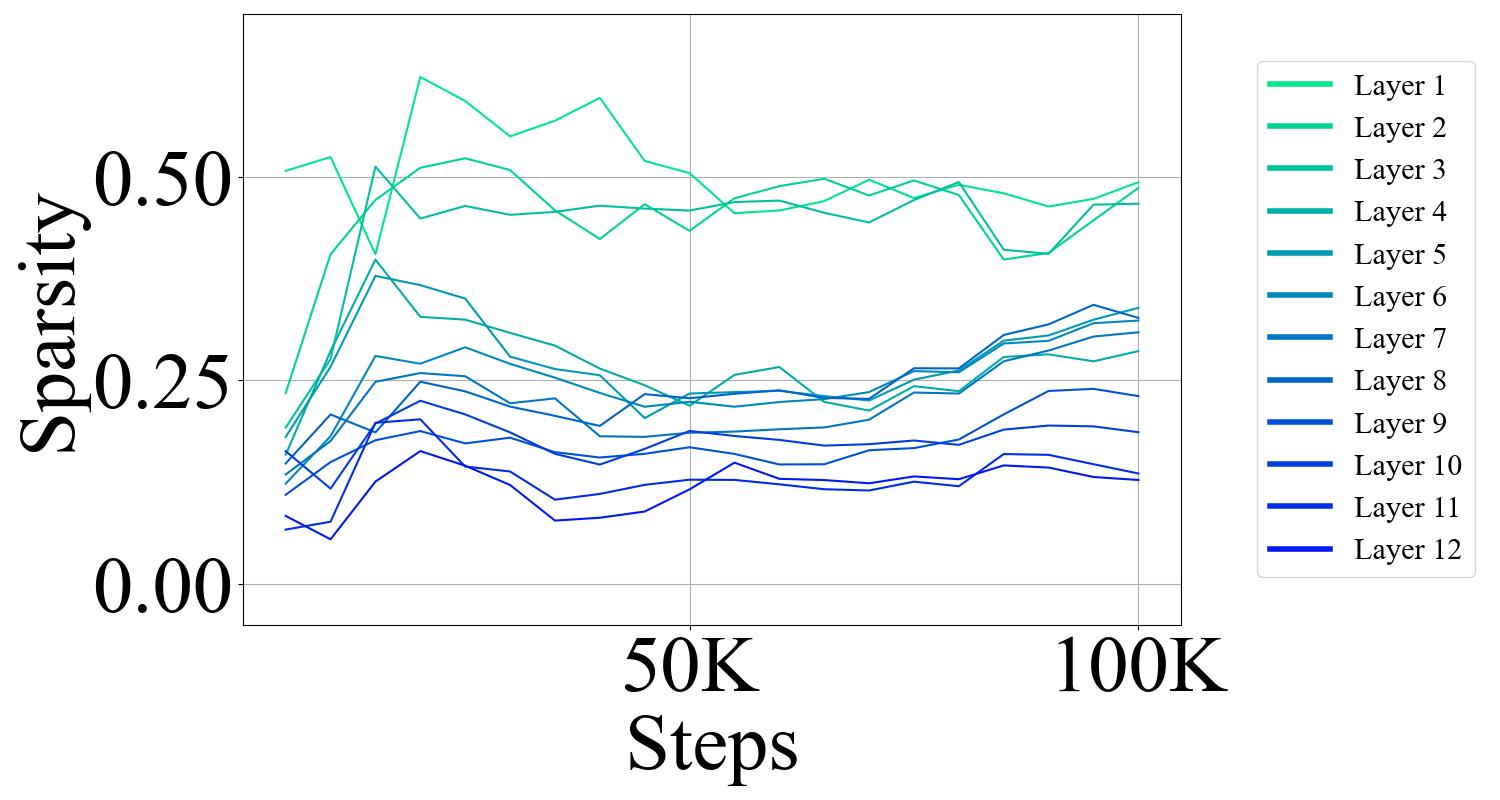}
}
    \subfloat[Testing, modified dec.\label{figure:productive_t5_sparsified_testing_decoder}]{
        \myincludegraphics[width=0.24\linewidth]{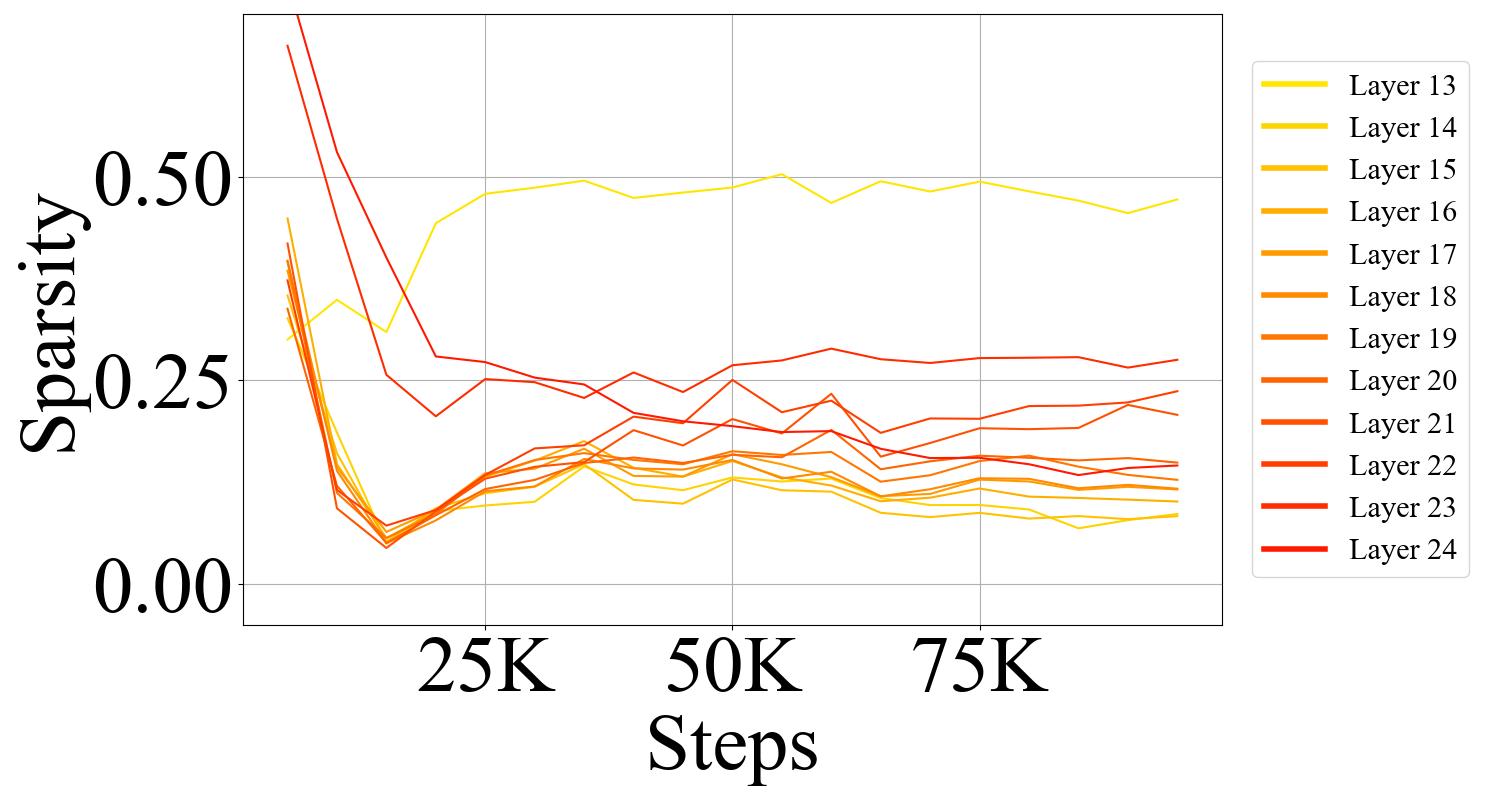}
}
    \subfloat[Testing, vanilla dec.\label{figure:productive_t5_vanilla_testing_decoder}]{
        \myincludegraphics[width=0.24\linewidth]{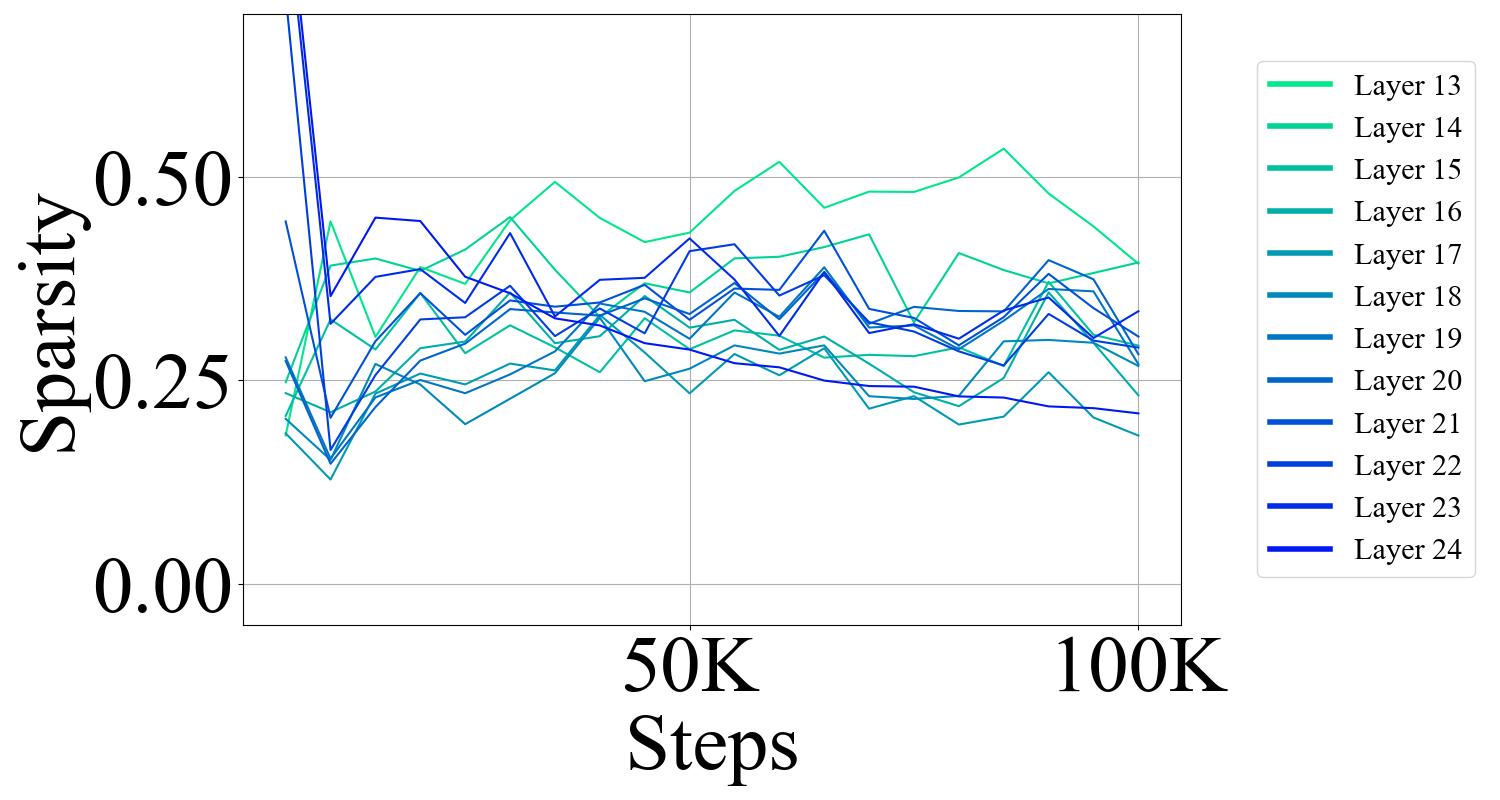}
}
    \caption{Training and testing sparsity during training of T5-Base on C4. Red is used for modified models while blue indicates vanilla ones. 
    }\label{figure:productive_t5_full}
\end{figure*}

\begin{figure*}[!htb]
    \centering
    \includegraphics[width=0.4\linewidth]{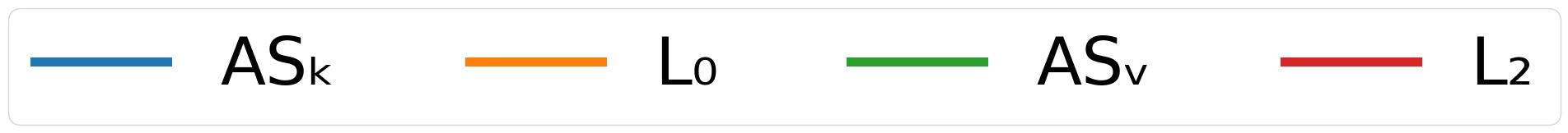}\\
    \subfloat[Sparsity.]{
        \includegraphics[width=0.17\linewidth]{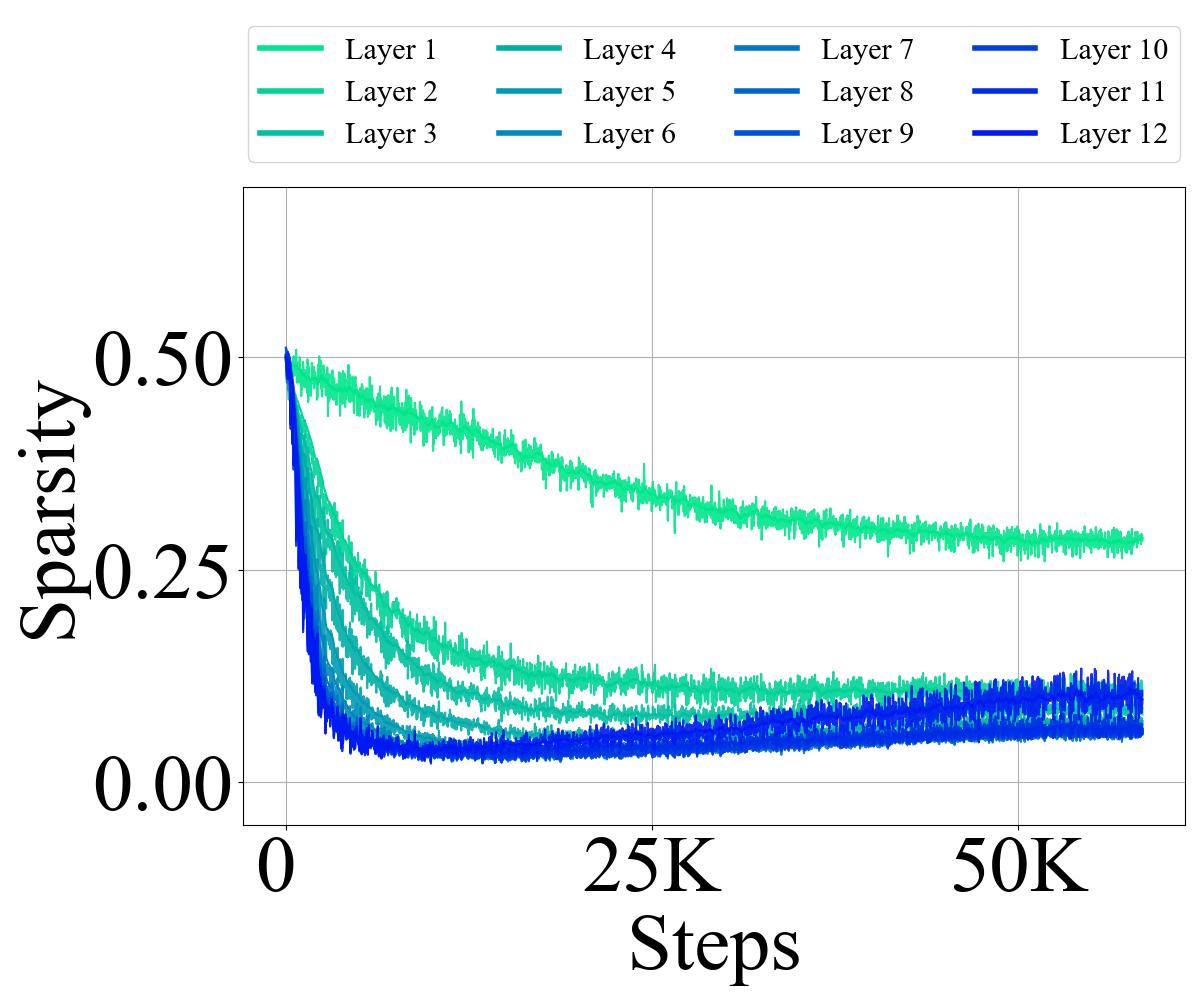}
    }
    \foreach \i in {0,...,11}{
        \pgfmathtruncatemacro{\layerid}{1+\i} 
        \subfloat[\small Layer \layerid]{
            \includegraphics[width=0.17\linewidth]{pic/af/augmented_flatness_cifar10_allepochs/vanilla/output/\i.jpg}
        }
    }
    \caption{Layerwise sparsity and $\tildeAS[\tildeKparam], c^l_{L_0}, \tildeAS[\tildeVparam], c^l_{L_2}$ of ViT-Base trained on CIFAR-10.}\label{fig:detail_start}
\end{figure*}

\begin{figure*}[!htb]
    \centering
    \includegraphics[width=0.4\linewidth]{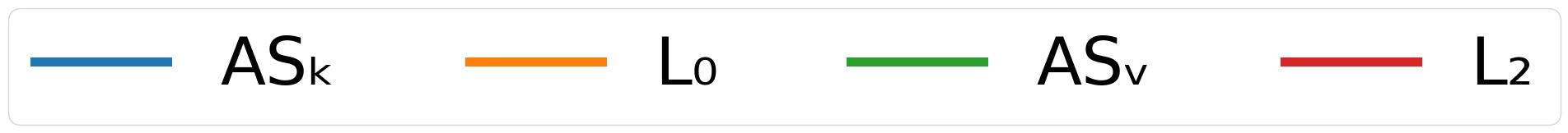}\\
    \foreach \i in {0,...,11}{
        \pgfmathtruncatemacro{\layerid}{1+\i} 
        \subfloat[\small Layer \layerid]{
            \includegraphics[width=0.17\linewidth]{pic/af/augmented_flatness_128/vanilla/output/\i.jpg}
        }
    }
    \caption{Layerwise $\tildeAS[\tildeKparam], c^l_{L_0}, \tildeAS[\tildeVparam], c^l_{L_2}$ of vanilla ViT-Base trained on ImageNet-1K during the first 10 epochs.}
\end{figure*}

\begin{figure*}[!htb]
    \centering
    \includegraphics[width=0.4\linewidth]{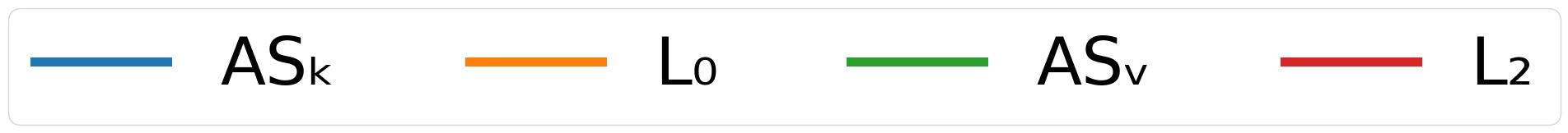}\\
    \foreach \i in {0,...,11}{
        \pgfmathtruncatemacro{\layerid}{1+\i} 
        \subfloat[\small Layer \layerid]{
            \includegraphics[width=0.17\linewidth]{pic/af/augmented_flatness_128_allepochs/vanilla/output/\i.jpg}
        }
    }
    \caption{Layerwise $\tildeAS[\tildeKparam], c^l_{L_0}, \tildeAS[\tildeVparam], c^l_{L_2}$ of vanilla ViT-Base trained on ImageNet-1K.}
\end{figure*}

\begin{figure*}[!htb]
    \centering
    \includegraphics[width=0.4\linewidth]{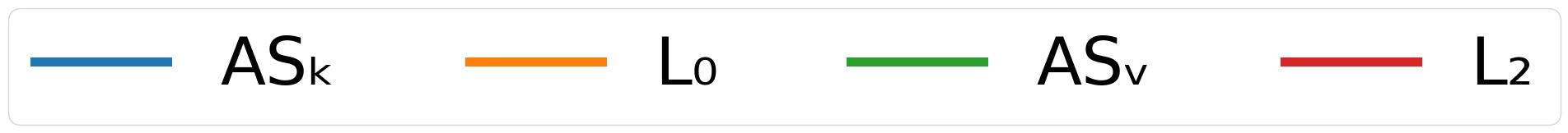}\\
    \foreach \i in {0,...,11}{
        \pgfmathtruncatemacro{\layerid}{1+\i} 
        \subfloat[\small Layer \layerid]{
            \includegraphics[width=0.17\linewidth]{pic/af/augmented_flatness_128_allepochs_sparsified/sparsified/output/\i.jpg}
        }
    }
    \caption{Layerwise $\tildeAS[\tildeKparam], c^l_{L_0}, \tildeAS[\tildeVparam], c^l_{L_2}$ of sparsified ViT-Base trained on ImageNet-1K.}\label{fig:detail_end}
\end{figure*} 
\minorrevision{
\section{Generalization of Sparsity}\label{appendix:sparsity_overfitting}

To show that the sparsity overfitting discovered in \cref{sec:analyses} cannot be arbitrarily severe, and to provide preliminary guidance for future algorithms that alleviate this overfitting, we derive an information-theoretic generalization bound for sparsity.

Information-theoretic generalization theory is increasingly viewed as promising due to its data- and algorithm-dependent nature, tightness, and, importantly, mild assumptions \cite{wang_tighter_2023,peng2025leveraging}, which allows us to apply it to sparsity.
Specifically, the theory shows that if a loss function is sub-Gaussian (a random variable $X$ is $R$-sub-Gaussian if $\ex{e^{\lambda (X - \ex{X})}} \le e^{\lambda^2 R^2/2}$ for any $\lambda \in \reals$), then overfitting measured by this loss is upper-bounded by the mutual information between training data and parameters, which reflects how much training-set-specific information is memorized in the parameters \cite{xu_information-theoretic_2017}:
\begin{lemma}\label{lemma:MI_bound}
    Let $\mathcal{P}_{\vec{x}, y}$ be the underlying data distribution of the task of interest.
    Let $\ell(\allparam, (\vec{x}, y)) \ge 0$ be a non-negative loss function.
    Assume $\ell(\allparam, \cdot)$ is $R$-sub-Gaussian on the sample distribution $\mathcal{P}_{\vec{x}, y}$ for any $\allparam$.
    Suppose the model is trained by first sampling $n$ samples to form training set $\dataset$ from $\mathcal{P}_{\vec{x}, y}$ in an independent and identically distributed (I.I.D.) manner, and then training the model using any training algorithm. Let $\allparam$ be the resulting parameter, which is a function of random $\dataset$ and is itself a random variable correlated with $\dataset$. 
    The generalization error \wrt{} $\ell$ is bounded by
    \begin{align}
        \ex[\dataset]{\underbrace{\ex[(\vec{x}', y') \sim \mathcal{P}_{\vec{x}, y}]{\ell(\allparam, (\vec{x}', y'))}}_{\text{testing loss}} - \underbrace{\ex[(\vec{x}, y) \sim U[\dataset]]{\ell(\allparam, (\vec{x}, y))}}_{\text{training loss}}}
         \le \sqrt{\frac{2 R^2}{n} I(\dataset; \allparam)},
    \end{align}
    where $I(\dataset; \allparam)$ is the mutual information between the training set and the resulting parameter.
\end{lemma}

Note that in \cref{lemma:MI_bound}, the training algorithm can use any loss other than $\ell$ and they are decoupled. Therefore, we can treat sparsity percentage (percentage of non-zero activations)
\begin{align}
    S^l(\allparam, (\vec{x}, y)) \defeq& \norm{\mat{D}^l}_0 / k n,\\
    S^l_{\textnormal{train}}(\allparam, \dataset) \defeq& \ex[(\vec{x}, y) \sim U[\dataset]]{S^l(\allparam, (\vec{x}, y))},\\
    S^l_{\textnormal{test}}(\allparam) \defeq& \ex[(\vec{x}', y') \sim \mathcal{P}_{\vec{x}, y}]{S^l(\allparam, (\vec{x}', y'))}
\end{align}
as a special loss.

The only non-trivial assumption required by \cref{lemma:MI_bound} is that the loss is sub-Gaussian.
Note that boundedness, which sparsity percentage naturally satisfies, naturally implies sub-Gaussianity. Therefore, \cref{lemma:MI_bound} can be applied to sparsity overfitting, as formalized by the following theorem:
\begin{theorem}\label{theorem:sparsity_overfitting}
    Assume the training set is sampled in an I.I.D. manner. Then the difference between testing and training sparsity is upper-bounded by
    \begin{align}
        \ex[\dataset]{S^l_{\textnormal{test}}(\allparam) - S^l_{\textnormal{train}}(\allparam, \dataset)}
         \le \sqrt{\frac{I(\dataset; \allparam)}{2|\dataset|}}
    \end{align}
\end{theorem}
\begin{proof}
    Since $\mat{D}^l \in \reals^{k \times n}$, we have $\norm{\mat{D}^l}_0 \in [0, k n]$ and $S^l(\allparam, (\vec{x}, y)) \defeq \norm{\mat{D}^l}_0 / k n \in [0, 1]$.
    Since a random variable in $[a, b]$ is naturally sub-Gaussian with $R = (b - a) / 2$, $S^l(\allparam, \cdot)$ is sub-Gaussian with $R = 1/2$.
    Then applying \cref{lemma:MI_bound} \wrt{} $S^l$ with $R=1/2$ and $n=|\dataset|$ proves the desired result.
\end{proof}
}

\end{document}